\documentclass[twoside,11pt]{article}

\usepackage{blindtext}

\usepackage[abbrvbib,preprint]{jmlr2e}

\usepackage{amsfonts, amsmath, amssymb,mathtools}
\usepackage{algorithm, algorithmicx, algpseudocode}
\usepackage{booktabs}

\newcommand{\R}{\mathbb{R}}
\newcommand{\Sym}{\operatorname{Sym}}

\renewcommand{\O}{\mathcal{O}}
\renewcommand{\S}{\mathcal{S}}
\newcommand{\V}{\mathcal{V}}
\newcommand{\Sf}{\mathcal{S}_{1:d}}
\newcommand{\Uf}{U_{1:d}}
\newcommand{\qf}{q_{1:d}}

\newcommand{\diag}[1]{\operatorname{diag}\left(#1\right)}
\newcommand{\tr}[1]{\operatorname{tr}\left(#1\right)}
\newcommand{\norm}[1]{\left\lVert#1\right\rVert}
\newcommand{\lrp}[1]{\left(#1\right)}
\newcommand{\lrb}[1]{\left[#1\right]}

\newcommand{\twodots}{\mathinner {\ldotp \ldotp}}
\newcommand{\St}{\operatorname{St}}
\newcommand{\Gr}{\operatorname{Gr}}
\newcommand{\Fl}{\operatorname{Fl}}
\newcommand{\T}{^\top}

\DeclareMathOperator*{\argmin}{arg\,min}
\DeclareMathOperator*{\argmax}{arg\,max}
\newcommand{\N}[1]{\mathcal{N}\left(#1\right)}
\newcommand{\bmat}[1]{\begin{bmatrix}#1\end{bmatrix}}

\usepackage{lastpage}
\jmlrheading{}{2025}{1-\pageref{LastPage}}{}{}{}{Tom Szwagier and Xavier Pennec}

\ShortHeadings{Nested Subspace Learning with Flags}{Szwagier and Pennec}
\firstpageno{1}

\begin{document}

\title{Nested Subspace Learning with Flags}

\author{\name Tom Szwagier\email tom.szwagier@inria.fr \\
       \addr Université Côte d'Azur and Inria\\
       Sophia Antipolis, France
       \AND
       \name Xavier Pennec \email xavier.pennec@inria.fr \\
       \addr Université Côte d'Azur and Inria\\
       Sophia Antipolis, France}

\editor{}

\maketitle

\begin{abstract}
Many machine learning methods look for low-dimensional representations of the data. The underlying subspace can be estimated by {first} choosing a dimension $q$ and {then} optimizing a certain objective function over the space of $q$-dimensional subspaces---the Grassmannian. Trying different $q$ yields in general non-nested subspaces, which raises an important issue of consistency between the data representations. In this paper, we propose a simple and easily implementable principle to enforce nestedness in subspace learning methods. It consists in lifting Grassmannian optimization criteria to flag manifolds---the space of nested subspaces of increasing dimension---via nested projectors. We apply the flag trick to several classical machine learning methods and show that it successfully addresses the nestedness issue.
\end{abstract}

\begin{keywords}
  subspace learning, Grassmann manifolds, flag manifolds, nested subspaces, dimensionality reduction
\end{keywords}

\section{Introduction}\label{sec:intro}

Finding low-dimensional representations of datasets is a quite common objective in machine learning, notably in virtue of the \textit{curse of dimensionality}~\citep{bellman_dynamic_1984}.
A classical way of building such a representation is by searching for a low-dimensional subspace that well represents the data, where ``well'' is defined by an application-dependent criterion~\citep{cunningham_linear_2015}. Here is a non-exhaustive list of examples of subspace learning problems.
\textit{Principal component analysis} (PCA)~\citep{pearson_lines_1901,hotelling_analysis_1933,jolliffe_principal_2002} searches for a low-dimensional subspace that minimizes the average squared Euclidean distance to the data.
\textit{Robust subspace recovery} (RSR)~\citep{lerman_overview_2018} minimizes the average (absolute) Euclidean distance to the data, which is less sensitive to outliers.
Many matrix decomposition methods, like \textit{robust PCA}~\citep{candes_robust_2011} and \textit{matrix completion}~\citep{keshavan_matrix_2010,candes_exact_2012} look for low-rank approximations of the data matrix, which can be decomposed into a product of subspace and coordinate matrices.
\textit{Trace ratio} (TR)~\citep{ngo_trace_2012} refers to a wide class of problems that look for subspaces making a tradeoff between desired yet antagonist properties of low-dimensional embeddings. An example is Fisher's \textit{linear discriminant analysis (LDA)}~\citep{fisher_use_1936} which seeks to maximize the between-class variance while minimizing the within-class variance.
\textit{Domain adaptation} methods learn some domain-invariant subspaces~\citep{baktashmotlagh_unsupervised_2013} by minimizing the projected maximum mean discrepancy (MMD)~\citep{gretton_kernel_2012} between the source and target distributions.
\textit{Subspace tracking} methods incrementally minimize a distance between the current subspace and the available data~\citep{balzano_online_2010}.
Subspaces can also be estimated not from the data but from their adjacency matrix or graph Laplacian, as done in the celebrated \textit{Laplacian eigenmaps}~\citep{belkin_laplacian_2003} and \textit{spectral clustering}~\citep{ng_spectral_2001}.
Subspaces can also be estimated beyond Euclidean data, for instance on symmetric positive-definite matrix datasets~\citep{harandi_dimensionality_2018}.
In all the previously cited examples, the search space---the space of all linear subspaces of dimension $q$ embedded in an ambient space of larger dimension $p$---is called the \textit{Grassmannian}, or the \textit{Grassmann manifold}~\citep{bendokat_grassmann_2024} and is denoted $\Gr(p, q)$. 
Consequently, many machine learning methods can be recast as an optimization problem on Grassmann manifolds.
\begin{remark}[Sequential Methods]\label{rem:sequential}
    Some alternative methods to subspace learning via Grassmannian optimization exist. Notably, a large family of methods---like \textit{canonical correlation analysis} (CCA)~\citep{hotelling_relations_1936,hardoon_canonical_2004}, \textit{independent component analysis} (ICA)~\citep{hyvarinen_independent_2000}, \textit{partial least squares} (PLS)~\citep{geladi_partial_1986} or \textit{projection pursuit}~\citep{huber_projection_1985,tsakiris_dual_2018}---builds a low-dimensional subspace out of the data \textit{sequentially}. They first look for the best 1D approximation and then recursively add dimensions one-by-one via deflation schemes or orthogonality constraints~\citep[Section~II.C]{lerman_overview_2018}. Those sequential methods however suffer from many issues---from a subspace estimation point of view---due to their greedy nature. Those limitations are well described in \citet{cunningham_linear_2015} and \citet{lerman_overview_2018}. In light of the fundamental differences between those methods and the previously described subspace methods, we decide not to address sequential methods in this paper.
\end{remark}

\noindent Learning a low-dimensional subspace from the data through Grassmannian optimization requires to choose the dimensionality $q$ \textit{a priori}. 
This prerequisite has many important limitations, not only since a wrong choice of dimension might remove all the theoretical guarantees of the method~\citep{lerman_overview_2018}, but also because the observed dataset might simply {not} have a favored dimension, above which the added information is worthless.\footnote{Many generative models assume that the data lies on a lower dimensional subspace, up to isotropic Gaussian noise~\citep{tipping_probabilistic_1999,lerman_robust_2015}. However, anisotropy is much more realistic~\citep{maunu_well-tempered_2019}, which nuances the theoretical guarantees and removes hopes for a clear frontier between signal and noise at the intrinsic dimension.}
For these reasons, one might be tempted to ``try'' several dimensions---i.e. run the subspace optimization algorithm with different dimensions and choose the best one \textit{a posteriori}, with cross-validation or statistical model selection techniques for instance.
In addition to being potentially costly, such a heuristic raises an important issue that we later refer to as the \textit{nestedness issue}: the optimal subspaces at different dimensions are \textit{not nested} in general.
This notably means that the data embeddings at two different dimensions might drastically differ, which is a pitfall for data analysis.
\begin{remark}[Importance of Nestedness]
    The importance of nestedness has been largely documented in the statistical literature: it yields a hierarchy of representations, improving interpretability, dimension selection, computational efficiency and stability. In geometric statistics, nestedness is considered as a key feature for multilevel analysis \citep{damon_backwards_2014}. Many works consequently directly impose the nestedness of the approximation subspaces by construction (e.g., via forward or backward sequential methods), in contrast to subspace-based methods which cannot guarantee nestedness. One can notably cite the geodesic versions of PCA~\citep{fletcher_principal_2004,sommer_optimization_2014,huckemann_intrinsic_2010}, the \textit{principal nested spheres} of~\citet{jung_analysis_2012}, the \textit{barycentric subspace analysis} of~\citet{pennec_barycentric_2018} and subsequent contributions \citep{eltzner_torus_2018,huckemann_backward_2018,dryden_principal_2019,jaquier_high-dimensional_2020,yang_nested_2021,fan_nested_2022,mankovich_chordal_2023,rabenoro_geometric_2024,mankovich_flag_2025,su_principal_2025}. All these works show the interest of nestedness in statistical inference, with theoretical advantages (e.g., sharper confidence intervals) and meaningful applications. It is also worth noticing that the review paper from~\citet{lerman_overview_2018} emphasizes that the lack of nestedness is a fundamental issue in robust subspace recovery, which is left open for future research.
\end{remark}
We illustrate in Figure~\ref{fig:motivations} the nestedness issue on toy datasets related to three important machine learning problems: robust subspace recovery, linear discriminant analysis and sparse spectral clustering~\citep{lu_convex_2016,wang_grassmannian_2017}.
\begin{figure}
\centering
\includegraphics[width=\linewidth]{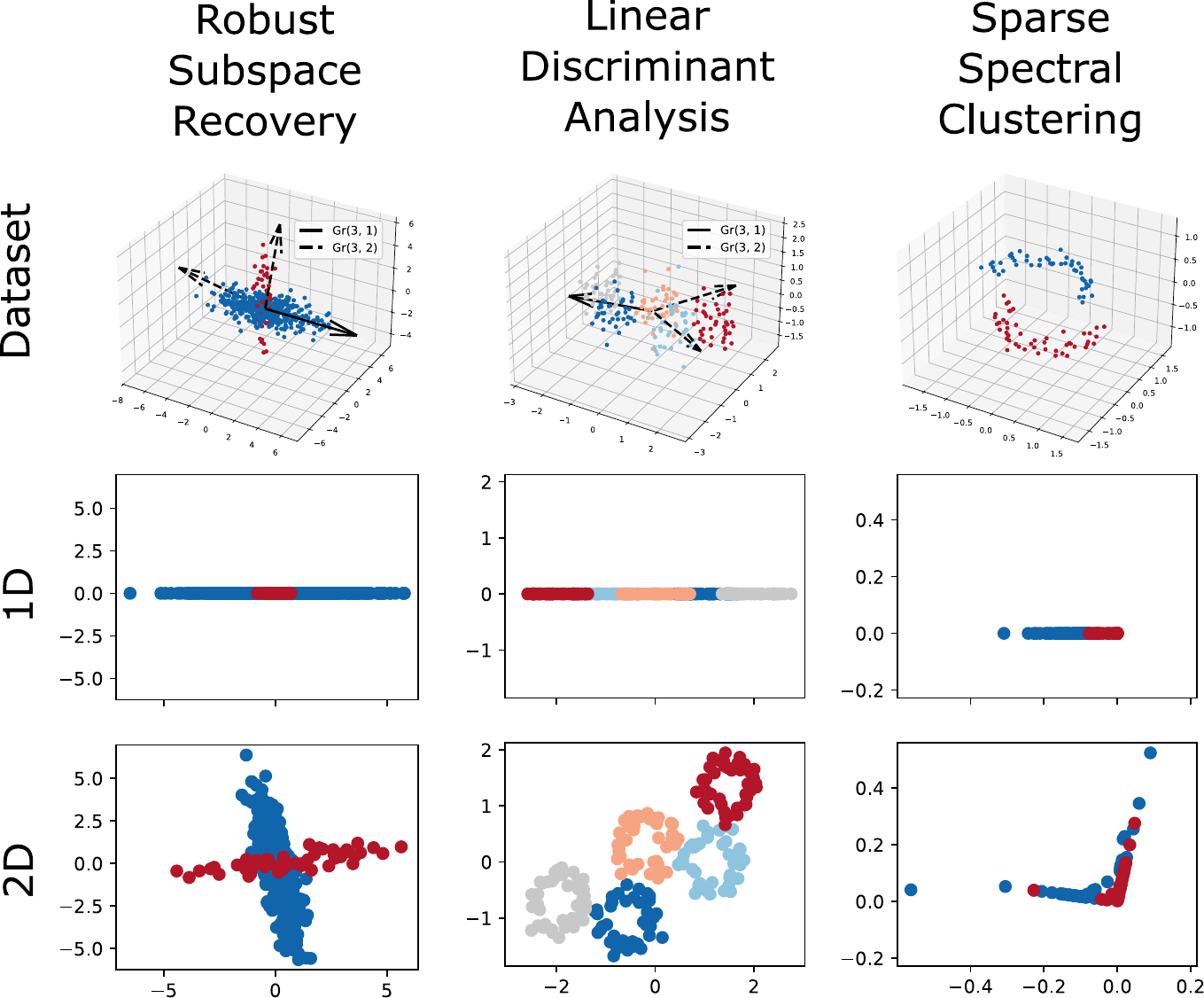}
\caption{Illustration of the subspace nestedness issue in three important machine learning problems: robust subspace recovery (left), linear discriminant analysis (middle) and sparse spectral clustering (right). For each dataset (top), we plot its projection onto the optimal 1D subspace (middle) and 2D subspace (bottom) obtained by solving the associated Grassmannian optimization problem. The 1D and 2D representations are \textit{inconsistent}---in the sense that the 1D plot is not the projection of the 2D plot onto the horizontal axis---which is a pitfall for data analysis.}
\label{fig:motivations}
\end{figure}
We can see that the scatter plots of the projected data for $q=1$ and $q=2$ are \textit{inconsistent}---in that the 1D representation is not the projection of the 2D representation onto the horizontal axis.

In this paper, we propose a natural solution to the nestedness issue with a geometrical object that has not been much considered yet in machine learning: \textit{flags}---i.e. sequences of nested subspaces of increasing dimension.
The idea is to first select a sequence of candidate dimensions---the \textit{signature} of the flag---then formulate a multilevel subspace learning criterion which extends the original one by integrating multiple dimensions simultaneously---with the \textit{flag trick}---and finally optimize over the space of flags of the chosen signature, which has a manifold structure similar to the one of Grassmannians.
The output subspaces are naturally nested, which solves the nestedness issue of subspace learning methods and provides the subspaces with a \textit{hierarchy} that is good for interpretability~\citep[Section~7]{huber_projection_1985}.
Moreover, the nested representations of the data can be fit to general machine learning algorithms and combined via \textit{ensembling} methods. The ensembling weights can then be interpreted as a measure of importance for the different dimensions, echoing the perspectives of the celebrated paper on the automatic choice of dimensionality in PCA~\citep{minka_automatic_2000}.
Beyond the change of paradigm (from subspaces to flags of subspaces), the main contribution of the paper is the \textit{flag trick}: a generic method to convert a \textit{fixed-dimension} subspace criterion into a \textit{multilevel} flag criterion, without nestedness issue.

The paper is organized as follows.
In Section~\ref{sec:flags}, we provide some prerequisites on flag manifolds and describe a steepest descent optimization algorithm that we will use throughout this work.
In Section~\ref{sec:flag_trick}, we introduce the flag trick and propose an ensemble learning algorithm to leverage the hierarchical information of the flag into richer machine learning models.
In Section~\ref{sec:examples}, we show the interest of the flag trick to several subspace learning methods such as robust subspace recovery, trace ratio and spectral clustering.
In Section~\ref{sec:discussion}, we conclude and discuss the limits and perspectives of such a framework.
In Appendix~\ref{app:RSR}, \ref{app:TR}, \ref{app:SSC}, we show how the flag trick can be used to develop two \textit{advanced} optimization algorithms on flag manifolds (namely an iteratively reweighted least squares (IRLS) method for robust subspace recovery and a Newton-Lanczos method for trace ratio problems) that go beyond the limitations of simple algorithms like the steepest descent, and we provide the proofs that are not in the main body. In Appendix~\ref{app:time}, we report the empirical running times for the flag trick in robust subspace recovery, trace ratio and spectral clustering.

\section{Basics of Flag Manifolds}\label{sec:flags}

A flag is a sequence of nested linear subspaces of increasing dimension.
This section introduces flag manifolds and provides the minimal tools to perform optimization on those spaces. More details and properties can be obtained in dedicated papers~\citep{ye_optimization_2022}.

\subsection{Flags in the Scientific Literature}
By providing a natural parametrization for the eigenspaces of symmetric matrices, flags have long been geometrical objects of interest in the scientific community, with traditional applications in physics \citep{arnold_modes_1972} and numerical analysis~\citep{duistermaat_functions_1983,ammar_geometry_1986,watkins_chasing_1991,helmke_isospectral_1991,helmke_optimization_1994}. The seminal work of \citet{edelman_geometry_1998} on the geometry of algorithms with orthogonality constraints paved the way for Riemannian optimization algorithms on Stiefel, Grassmann and related matrix manifolds. The first work that explicitly proposed an optimization algorithm on flag manifolds dates back to 2006. \citet{nishimori_riemannian_2006} reformulated \textit{independent subspace analysis}---an important variant of independent component analysis which maximizes the independence between the norms of projected samples into mutually-orthogonal subspaces~\citep{cardoso_multidimensional_1998, hyvarinen_emergence_2000}---as an optimization problem on flag manifolds, and solved it via a Riemannian gradient descent. 
This led to many papers developing or performing optimization on flag manifolds~\citep{nishimori_flag_2007, nishimori_natural_2008, fiori_extended_2011, fiori_extended_2012, ye_optimization_2022, zhu_practical_2024} and to implementations in open-source packages for differential geometry such as \textit{Manifolds.jl}~\citep{axen_manifoldsjl_2023}.

Recently, flag manifolds played a central role in the investigation of a \textit{curse of isotropy} in principal component analysis. By means of a probabilistic covariance model parameterized with flags, \citet{szwagier_curse_2025} show that \textit{close eigenvalues}---whose relative distance is lower than a certain threshold---should be assumed \textit{equal}. The resulting \textit{principal subspace analysis} has led to several follow-up works, notably an efficient algorithm for covariance estimation in high dimensions~\citep{szwagier_eigengap_2025} and an extension to parsimonious Gaussian mixture models~\citep{szwagier_parsimonious_2025}.

It is important to notice that flags---as mutually-orthogonal subspaces---also implicitly appear in many works, with different keywords. For instance, the mutually-orthogonal class-subspaces of~\citet{watanabe_subspace_1973}, the adaptive-subspace self-organizing maps of~\citet{kohonen_emergence_1996} and follow-up works~\citep{tae-kyun_kim_-line_2010,giguere_manifold_2017} implicitly involve flags under the name of ``mutually-orthogonal subspaces''. Flags also appear in many spectral methods under keywords such as ``isospectral manifolds''~\citep{deift_ordinary_1983,watkins_isospectral_1984,brockett_dynamical_1991,helmke_isospectral_1991,lim_simple_2024}; indeed, at fixed eigenvalues, the remaining degrees of freedom in a symmetric matrix are the associated eigenspaces, which exactly correspond to flags.

Flags as sequences of nested subspaces were first introduced in the machine learning literature as robust prototypes to represent collections of subspaces of different dimensions~\citep{draper_flag_2014,santamaria_order_2016,mankovich_flag_2022}.
They were obtained via a sequential construction (where one dimension is added at a time), which can be problematic for greediness reasons~\citep{huber_projection_1985,lerman_overview_2018}.
\citet{pennec_barycentric_2018} was the first to show that PCA can be reformulated as an optimization problem on flag manifolds by summing the unexplained variance at different dimensions, raising interesting perspectives for multilevel data analysis on manifolds. This principle was recently applied to several variants of PCA under the name of \textit{flagification} in~\citet{mankovich_fun_2024}, showing an important robustness to outliers. Some works compute distance-related quantities on flag-valued datasets, with notable applications in computer vision and hyperspectral imaging~\citep{ma_flag_2021, ma_self-organizing_2022, nguyen_closed-form_2022, mankovich_chordal_2023, szwagier_rethinking_2023}. Other applications include shape spaces~\citep{ciuclea_shape_2023}, quantum physics and chemistry~\citep{vidal_geometric_2024}.
Finally, flags were recently used to learn nested representations from data with a prespecified hierarchy~\citep{mankovich_flag_2025}.

\subsection{Definition and Representation of Flag Manifolds}
Let $p \geq 2$ and $q_{1:d} \coloneqq (q_1, q_2 ,\dots, q_d)$ be a sequence of increasing integers such that $0 < q_1 < q_2 < \cdots < q_d < p$.
A \textit{flag} of \textit{signature} $(p, q_{1:d})$ is a sequence of nested {linear} subspaces $\{0\} \subset \S_1 \subset \S_2 \subset \dots \subset \S_d \subset \R^p$ of respective dimensions $q_1, q_2, \dots, q_d$, noted here $\S_{1:d} \coloneqq (\S_1, \dots, \S_d)$.\footnote{Flags can equivalently be defined as sequences of \textit{mutually orthogonal} subspaces $\V_1 \perp \V_2 \dots \perp \V_d$, of respective dimensions $q_1, q_2-q_1, \dots, q_d - q_{d-1}$, by taking $\V_k$ to be the orthonormal complement of $\S_{k-1}$ onto $\S_k$. This definition is more convenient for computations, but we won't need it in this paper.}
A flag $\S_{1:d}$ can be canonically represented as a sequence of symmetric matrices that are the \textit{orthogonal projection matrices} onto the nested subspaces, i.e. $\S_{1:d} \cong (\Pi_{\S_1}, \dots, \Pi_{\S_d})\in\Sym(p)^d$. We call it the \textit{projection representation} of flags.

The set of flags of signature $(p, q_{1:d})$ is a smooth manifold~\citep{ye_optimization_2022}, denoted here $\Fl(p, q_{1:d})$. 
Flag manifolds generalize {Grassmannians} when $d=1$---since $~{\Fl(p, q) = \Gr(p, q)}$---and therefore share many practical properties that are useful for optimization~\citep{edelman_geometry_1998}. In the following, we will frequently use the following notations: $q_0 \coloneqq 0$, $q \coloneqq q_d$ and $q_{d+1} \coloneqq p$.

For computational and numerical reasons, flags are often represented as orthonormal $q$-frames. Those correspond to points on the Stiefel manifold $\St(p, q) = \{U\in\R^{p\times q}\colon U\T U = I_q\}$).
Let us define sequentially, for $k\in\{1, \dots, d\}$, $U_k\in\St(p, q_k - q_{k-1})$ such that $\bmat{U_1& \cdots & U_k}$ is an orthonormal basis of $\S_k$ (this is possible thanks to the nestedness of the subspaces).
Then, $U_{1:d} \coloneqq \bmat{U_1& \cdots & U_d}\in\St(p, q)$ is a representative of the flag $\S_{1:d}$. We call it the \textit{Stiefel representation} of flags. Such a representation is not unique---contrary to the projection representation defined previously---due to the \textit{rotational-invariance} of orthonormal bases of subspaces. More precisely, if $U_{1:d}$ is a Stiefel representative of the flag $\S_{1:d}$, then for any set of orthogonal matrices $R_k\in\O(q_k - q_{k-1})$, the matrix $U'_{1:d} \coloneqq \bmat{U_1 R_1 & \cdots & U_d R_d}$ spans the same flag of subspaces $\S_{1:d}$.
This provides flag manifolds with a quotient manifold structure~\citep{edelman_geometry_1998, absil_optimization_2008, ye_optimization_2022}: 
\begin{equation}\label{eq:Fl_quotient}
	\Fl(p, q_{1:d}) \cong \St(p, q) \big/ \lrp{\O(q_1) \times \O(q_2 - q_1) \times \dots \times \O(q_d - q_{d-1})}.
\end{equation}

\begin{remark}[Orthogonal Representation]
For computations, one might have to perform the orthogonal completion of some Stiefel representatives.
Let $\Uf \coloneqq \bmat{U_1& \cdots & U_d} \in \St(p, q)$, then one denotes $U_{d+1} \in \St(p, p-q_d)$ to be any orthonormal basis such that $U_{1:d+1} \coloneqq \bmat{U_1 & \cdots & U_d & U_{d+1}} \in \O(p)$. Such an orthogonal matrix $U_{1:d+1}$ will be called an \emph{orthogonal representative} of the flag $\Sf$. In the following, we may abusively switch from one representation to the other since they represent the same flag.
\end{remark}

\subsection{Optimization on Flag Manifolds}
There is a rich literature on optimization on smooth manifolds~\citep{edelman_geometry_1998,absil_optimization_2008,boumal_introduction_2023}, and the particular case of flag manifolds has been notably addressed in~\citet{ye_optimization_2022,zhu_practical_2024}. Since flag manifolds can be represented as quotient spaces of Stiefel manifolds, which themselves are embedded in a Euclidean matrix space, one can develop some optimization algorithms without much difficulty.
In this paper, we will use a \textit{steepest descent} algorithm, which is drawn from several works~\citep{chikuse_statistics_2003,nishimori_riemannian_2006,ye_optimization_2022,zhu_practical_2024}. 
Let $f\colon \Fl(p, \qf) \to \R$ be a smooth function on a flag manifold, expressed in the Stiefel representation (e.g. $f(\Uf) = \sum_{k=1}^d \norm{{U_k} {U_k}\T x} $ for some $x\in\R^p$). Given $\Uf \in \St(p, q)$, let $\operatorname{Grad} f (\Uf) = ({\partial f}/{\partial U_{ij}})_{i, j = 1}^{p, q}$ denote the (Euclidean) gradient of $f$. To ``stay'' on the manifold, one first computes the \textit{Riemannian gradient} of $f$ at $\Uf$, noted $\nabla f (\Uf)$. It can be thought of as a projection of the Euclidean gradient onto the tangent space and computed explicitly~\citep{nishimori_riemannian_2006,ye_optimization_2022}.
Then, one moves in the opposite direction of $\nabla f (\Uf)$ with a so-called \textit{retraction}, which is chosen to be the polar retraction of~\citet[Eq.~(49)]{zhu_practical_2024}, combined with a line-search.
We iterate until convergence.
The final steepest descent algorithm is described in Algorithm~\ref{alg:GD}.
\begin{algorithm}
\caption{Steepest Descent on Flag Manifolds}\label{alg:GD}
\begin{algorithmic}
\Require $f\colon \Fl(p, \qf) \to \R$ a function, $\Uf \in \Fl(p, \qf)$ a flag (Stiefel representation)
\For{$t$ = 1, 2, \dots}
    \State $\bmat{G_1 & \cdots & G_d} \gets \operatorname{Grad} f (\Uf)$ \Comment{Euclidean gradient}
	\State $\nabla \gets \bigl[G_k - \bigl(U_k {U_k}\T G_k + \sum_{l \neq k} U_l {G_l}\T U_k\bigr)\bigr]_{k=1\dots d}$
	\Comment{Riemannian gradient}
	\State $\Uf \gets \operatorname{polar}(U_{1:d} - \alpha \nabla)$ \Comment{polar retraction + line search}
\EndFor
\Ensure $U_{1:d}^* \in \Fl(p, \qf)$ an optimal flag
\end{algorithmic}
\end{algorithm}
\begin{remark}[Initialization]\label{rem:init_gd}
We can initialize Algorithm~\ref{alg:GD} ``randomly'' (under the invariant measure of~\citet{james_normal_1954}): we generate a random $p\times q$ matrix with normal entries and perform its polar decomposition; the orthogonal factor then follows a uniform distribution \citep[Theorem~1.5.5]{chikuse_statistics_2003}. We can also initialize Algorithm~\ref{alg:GD} with a specific flag depending on the objective function, as we will see in Section~\ref{sec:examples}. Such an initial flag can, for instance, be obtained via the eigendecomposition of a symmetric matrix related to the problem, or by refining the solution of the original subspace problem.
\end{remark}
\begin{remark}[Optimization Variants]\label{rk:optim}
Many extensions of Algorithm~\ref{alg:GD} can be considered: conjugate gradient~\citep{ye_optimization_2022}, Riemannian trust region~\citep{absil_optimization_2008}, etc. We can also replace the polar retraction with a geodesic step~\citep{ye_optimization_2022} or other retractions~\citep{zhu_practical_2024}.
\end{remark}
\begin{remark}[Complexity Analysis]\label{rem:complexity}
    The number of operations for the polar retraction is $O(p q^2) + O(q^3)$ (see~\citet{zhu_practical_2024} and comparisons with other retractions therein). The number of operations for the Riemannian gradient (from the Euclidean gradient) is $O(p q^2)$. The number of operations for the Euclidean gradient depends on the objective function $f$ (which itself depends on the dimension $p$, the signature $\qf$ and the number of samples $n$), and the overall complexity of Algorithm~\ref{alg:GD} also depends on the number of iterations (for both the main loop and the line search).
    Therefore, for completeness, we report the empirical running times of our steepest descent algorithm in Appendix~\ref{app:time}.
\end{remark}

 \section{The Flag Trick in Theory}\label{sec:flag_trick}
In this section, we motivate and introduce the flag trick to make subspace learning methods nested.
The key result is Theorem~\ref{thm:flag_trick}, which states that the classical PCA at a fixed dimension can be converted into a nested multilevel method using nested projectors.

\begin{remark}[Centering]
    In the remaining of the paper, we assume that the data has been already \textit{centered} around a point of interest (e.g. its mean or geometric median), so that we are only interested in fitting \textit{linear} subspaces and not \textit{affine} ones. One could directly include the center in the optimization variables---in which case the geometry would be the one of affine Grassmannians~\citep{lim_numerical_2019} or affine flags~\citep{pennec_barycentric_2018}---but we don't do it in this work for conciseness.
\end{remark}

\subsection{From Subspaces to Flags of Subspaces: the Seminal Example of PCA}\label{subsec:nested_pca}
PCA is known as the eigendecomposition of the sample covariance matrix. Originally, it can be formulated as the search for a low dimensional subspace that minimizes the unexplained variance (or maximizes the explained variance).
Let $X\coloneqq\bmat{x_1 & \cdots & x_n} \in \R^{p\times n}$ be a data matrix with $n$ samples, let $\S \in \Gr(p, q)$ be a $q$-dimensional subspace, and let $\Pi_{\S} \in \R^{p\times p}$ be the orthogonal projection matrix onto $\S$.
Then PCA consists in the following optimization problem on Grassmannians:
\begin{equation}\label{eq:PCA_subspace}
\S_q^* = \argmin_{\S \in \Gr(p, q)} \norm{X - \Pi_{\S} X}_F^2,
\end{equation}
where $\norm{M}_F^2 \coloneqq \tr{M\T M}$ denotes the \textit{Frobenius norm}.
The solution to the optimization problem is the $q$-dimensional subspace spanned by the leading eigenvectors of the sample covariance matrix $S \coloneqq \frac 1 n X X\T$, that we note $\S_q^* = \operatorname{Span}(v_1, \dots, v_q)$. 
It is unique when the sample eigenvalues $q$ and $q+1$ are distinct, which is almost sure when $q \leq \operatorname{rank}(S)$. We will assume to be in such a setting in the following for simplicity but it can be easily handled otherwise by ``grouping'' the repeated eigenvalues (cf.~\citet[Theorem~B.1]{szwagier_curse_2025}).
In such a case, the principal subspaces are \textit{nested} for increasing $q$, i.e., if $\S_q^*$ is the $q$-dimensional principal subspace, then for any $r > q$, one has $\S_q^* = \operatorname{Span}(v_1, \dots, v_q) \subset \operatorname{Span}(v_1, \dots, v_r) = \S_r^*$.

Another way of performing PCA is in a sequential manner (cf. Remark~\ref{rem:sequential}). We first estimate the 1D subspace $\mathcal{V}_1^*$ that minimizes the unexplained variance, then estimate the 1D subspace $\mathcal{V}_2^*$ that minimizes the unexplained variance while being orthogonal to the previous one, and so on and so forth. This gives the following \textit{constrained} optimization problem on 1D Grassmannians:
\begin{equation}\label{eq:PCA_sequential}
\mathcal{V}_q^* = \argmin_{\substack{\mathcal{V} \in \Gr(p, 1)\\ \mathcal{V} \perp \mathcal{V}_{q-1} \perp \dots \perp \mathcal{V}_1}} \norm{X - \Pi_{\mathcal{V}} X}_F^2.
\end{equation}
This construction naturally yields a sequence of nested subspaces of increasing dimension---i.e. a flag of subspaces---better and better approximating the dataset:
\begin{equation}\label{eq:PCA_seq_cum}
	\{0\} \subset \S_1^* \subset \S_2^* \subset \dots \subset \S_{p-1}^* \subset \R^p, \text{with } \S_k^* = \bigoplus_{l=1}^k \mathcal{V}_l^*.
\end{equation}
Those subspaces happen to be exactly the same as the ones obtained by solving the subspace learning optimization problem~\eqref{eq:PCA_subspace}, although the way they are obtained (in a greedy manner) is different. This is generally not the case for other dimension reduction problems (for instance in robust subspace recovery, as it is raised in the final open questions of~\citet{lerman_overview_2018}).

Hence, the \textit{subspace} learning formulation of PCA~\eqref{eq:PCA_subspace} is equivalent to the \textit{sequential} formulation of PCA~\eqref{eq:PCA_sequential}---in terms of cumulative spans~\eqref{eq:PCA_seq_cum}---and both yield a flag of subspaces best and best approximating the data. One can wonder if this result could be directly obtained by formulating an optimization problem on flag manifolds. The answer is \textit{yes}, as first proven in~\citet[Theorem~9]{pennec_barycentric_2018} with an \textit{accumulated unexplained variance} (AUV) technique, but there is not a unique way to do it. Motivated by the recent principal subspace analysis~\citep{szwagier_curse_2025}, we propose in the following theorem a generic principle to formulate PCA as an optimization on flag manifolds.
\begin{theorem}[Nested PCA with Flag Manifolds]\label{thm:flag_trick}
	Let $X \coloneqq \bmat{x_1 & \cdots & x_n}\in\R^{p\times n}$ be a centered $p$-dimensional ($p \geq 2$) dataset with $n$ samples. Let $q_{1:d} \coloneqq (q_1, q_2 ,\dots, q_d)$ be a sequence of increasing dimensions such that $0 < q_1 < q_2 < \dots < q_d < p$.
	Let $S \coloneqq \frac 1 n X X\T$ be the sample covariance matrix. Assume that it eigendecomposes as $S \coloneqq \sum_{j=1}^p \ell_j v_j {v_j}\T$ where $\ell_1 \geq \dots \geq \ell_p$ are the eigenvalues and $v_1 \perp \dots \perp v_p$ are the associated eigenvectors.
	Then PCA can be reformulated as the following optimization problem on flag manifolds:
	\begin{equation}
		{\S_{1:d}^*} = \argmin_{\S_{1:d} \in \Fl(p, q_{1:d})} \norm{X - \frac 1 d \sum_{k=1}^d \Pi_{\S_k} X}_F^2.
	\end{equation}
	More precisely, one has ${\S_{1:d}^*} = \lrp{\operatorname{Span}(v_1, \dots, v_{q_1}), \operatorname{Span}(v_1, \dots, v_{q_2}), \dots, \operatorname{Span}(v_1, \dots, v_{q_d})}$.
	The solution is unique if and only if $\ell_{q_k} \neq \ell_{q_{k}+1}, \forall k\in\{1, \dots, d\}$.
\end{theorem}
\begin{proof} 
One has:
\begin{align}
	\norm{X - \frac 1 d \sum_{k=1}^d \Pi_{\S_k} X}_F^2
	&= \tr{X\T \lrp{I_p - \frac1d\sum_{k=1}^d \Pi_{\S_k}}^2 X},\\
	&= \frac1{d^2} \tr{X\T \lrp{\sum_{k=1}^d \lrp{I_p - \Pi_{\S_k}}}^2 X},\\
	&= \frac n {d^2} \tr{W^2 S},
\end{align}
with $W = \sum_{k=1}^d (I_p - \Pi_{\S_k})$ and $S = \frac 1 n X X\T$.
	Let $U_{1:d+1} \coloneqq \bmat{U_1 & \cdots & U_d & U_{d+1}}\in\O(p)$ be an orthogonal representative (cf. Section~\ref{sec:flags}) of the optimization variable $\S_{1:d} \in \Fl(p, q_{1:d})$. Then one has $~{\Pi_{\S_k} = U_{1:d+1} \diag{I_{q_k}, 0_{p - q_k}} {U_{1:d+1}}\T}$. Therefore, one has $W = U_{1:d+1} \Lambda {U_{1:d+1}}\T$, with $~{\Lambda = \diag{0 \, I_{q_1}, 1 \, I_{q_2 - q_1}, \dots, d \, I_{q_{d+1} - q_d}}}$.
	Hence, one has 
	\begin{equation}
		\argmin_{\S_{1:d} \in \Fl(p, q_{1:d})} \norm{X - \frac 1 d \sum_{k=1}^d \Pi_{\S_k} X}_F^2 \Longleftrightarrow \argmin_{U \in \O(p)} \frac n {d^2} \tr{U \Lambda^2 U\T S}.
	\end{equation}
	The latter problem is exactly the same as in~\citet[Equation~(20)]{szwagier_curse_2025}, which solves maximum likelihood estimation for principal subspace analysis.
	Hence, one can conclude the proof on existence and uniqueness using~\citet[Theorem~B.1]{szwagier_curse_2025}.
\end{proof}
The key element of the proof of Theorem~\ref{thm:flag_trick} is that averaging the nested projectors yields a \textit{hierarchical reweighting} of the (mutually-orthogonal) principal subspaces. More precisely, the $k$-th principal subspace has weight $(k-1)^2$, and this monotonic weighting enables to get a hierarchy of eigenspaces~\citep{absil_optimization_2008,cunningham_linear_2015,pennec_barycentric_2018,oftadeh_eliminating_2020}.
In the following, we note $~{\Pi_{\S_{1:d}} \coloneqq \frac 1 d \sum_{k=1}^d \Pi_{\S_k}}$ and call this symmetric matrix the \textit{average multilevel projector}, which will be central in the extension of subspace methods into multilevel subspace methods.

\begin{remark}[Weighted Projectors]
    One can show that any convex combination of the nested projectors yields the same result. Indeed, let $\alpha_1, \dots, \alpha_{d}$ be positive scalars. Let $\Sf \in \Fl(p, \qf)$ and $U_{1:d+1} \coloneqq \bmat{U_1 & U_2 & \cdots & U_{d+1}} \in \O(p)$ be an orthogonal representative of $\Sf$ (cf. Section~\ref{sec:flags}). Then, one has $\sum_{k=1}^{d} \alpha_k \Pi_{\S_k} = \sum_{k=1}^{d} \alpha_k \sum_{l=1}^{k} U_l{U_l}\T = \sum_{l=1}^{d} \lrp{\sum_{k=l}^{d} \alpha_k} U_l{U_l}\T$. Therefore, the weights in front of each (mutually-orthogonal) subspace $\operatorname{Span}(U_l)$ are strictly decreasing, and the solutions of Theorem~\ref{thm:flag_trick} are exactly the same. For criteria beyond PCA, the choice of weights might have an influence on the solution. In the following, we will by default consider uniform weights $\lrp{\frac{1}{d}, \dots, \frac{1}{d}}$.
\end{remark}

\subsection{The Flag Trick}
As we will see in the following (Section~\ref{sec:examples}), many important machine learning problems can be formulated as the optimization of a certain function $f$ on Grassmannians.
Theorem~\ref{thm:flag_trick} shows that replacing the subspace projection matrix $\Pi_{\S}$ appearing in the objective function by the average multilevel projector $\Pi_{\Sf}$ yields a sequence of subspaces that meet the original objective of principal component analysis, while being nested.
This leads us to introduce the \textit{flag trick} for general subspace learning problems.
\begin{definition}[Flag Trick]\label{def:flag_trick}
Let $p \geq 2$, $0 < q < p$ and $q_{1:d} \coloneqq (q_1, q_2 ,\dots, q_d)$ be a sequence of increasing dimensions such that $0 < q_1 < q_2 < \dots < q_d < p$.
The flag trick consists in replacing a subspace learning problem of the form:
\begin{equation}\label{eq:subspace_problem}
    \argmin_{\S \in \Gr(p, q)} f(\Pi_\S)
\end{equation}
with the following optimization problem:
\begin{equation}\label{eq:flag_problem}
    \argmin_{\S_{1:d} \in \Fl(p, q_{1:d})} \, f\lrp{\Pi_{\S_{1:d}}},
\end{equation}
where $~{\Pi_{\S_{1:d}} \coloneqq \frac 1 {d} \sum_{k=1}^{d} \Pi_{\S_k}}$ is the average multilevel projector.
\end{definition}
Let us note that for $d=1$---i.e., with signatures of the form $q_{1:d} = (q)$---the flag-tricked problem~\eqref{eq:flag_problem} is equivalent to the original subspace problem~\eqref{eq:subspace_problem}. For $d>1$, the flag trick yields new objective criteria that we cannot a priori relate to the original subspace learning problem. Except for the very particular case of PCA where $f_X(\Pi) = \norm{X - \Pi X}_F^2$ (Theorem~\ref{thm:flag_trick}), we cannot expect to have an analytic solution to the flag problem~\eqref{eq:flag_problem}; indeed, in general, subspace problems do not even have a closed-form solution as we shall see in Section~\ref{sec:examples}. This justifies the introduction of optimization algorithms on flag manifolds like Algorithm~\ref{alg:GD}. 

\begin{remark}[Flag Trick vs. AUV]
The original idea of accumulated unexplained variance~\citep{pennec_barycentric_2018} (and its subsequent application to several variants of PCA under the name of ``flagification''~\citep{mankovich_fun_2024}) consists in summing the subspace criteria at different dimensions, while the flag trick directly averages the orthogonal projection matrices that appear inside the objective function.
While both ideas are equally worth experimenting with, we believe that the flag trick has a much wider practical reach. Indeed, from a technical viewpoint, the flag trick appears at the covariance level and directly yields a hierarchical reweighting of the principal subspaces (see end of Section~\ref{subsec:nested_pca}). This reweighting is only indirect with the AUV---due to the linearity of the trace operator---and is not expected to occur beyond PCA (flagification). Notably, as we shall see in Section~\ref{sec:examples}, the flag trick enables to easily develop extensions of well-known methods involving PCA, like IRLS~\citep{lerman_fast_2018} or Newton-Lanczos methods for trace ratio optimization~\citep{ngo_trace_2012}, it less likely reaches singularities of the objective function (see, notably, end of~Section~\ref{subsubsec:FT_RSR}) and it is closer in spirit to the statistical formulations of PCA~\citep{szwagier_curse_2025}.
\end{remark}

\subsection{Multilevel Machine Learning}
Subspace learning is often used as a preprocessing task before running a machine learning algorithm, notably to overcome the curse of dimensionality. One usually projects the data onto the optimal subspace $\S^* \in \Gr(p, q)$ and use the resulting lower-dimensional dataset as an input to a machine learning task like clustering, classification or regression~\citep{bouveyron_model-based_2019}. Since the flag trick problem~\eqref{eq:flag_problem} does not output one subspace but a hierarchical sequence of nested subspaces, it is legitimate to wonder what to do with such a multilevel representation.
In this subsection, we propose a general \textit{ensemble learning} method to aggregate the hierarchical information coming from the flag of subspaces.

Let us consider a dataset $X \coloneqq \bmat{x_1 & \cdots & x_n}\in\R^{p \times n}$ (possibly with some associated labels $Y \coloneqq \bmat{y_1& \cdots & y_n}\in\R^{m \times n}$). In machine learning, one often fits a model to the dataset by optimizing an objective function of the form $R_{X, Y}(g) = \frac{1}{n} \sum_{i=1}^n L(g(x_i), y_i)$.
With the flag trick, we get a filtration of projected data points $Z_k = \Pi_{\S_k^*} X, k\in\{1, \dots, d\}$ that can be given as an input to different machine learning algorithms. This yields optimal predictors $g_k^* = \operatorname{argmin}_g \, R_{Z_k, Y}(g)$ which can be aggregated via \href{https://scikit-learn.org/stable/modules/ensemble.html}{ensembling methods}.
For instance, \textit{voting} methods choose the model with the highest performance on holdout data; this corresponds to selecting the optimal dimension $q^* \in q_{1:d}$ \textit{a posteriori}, based on the machine learning objective. A more nuanced idea is the one of \textit{soft voting}, which makes a weighted averaging of the predictions. The weights can be uniform, proportional to the performances of the individual models, or learned to maximize the performance of the weighted prediction~\citep{perrone_when_1992}. Soft voting gives different weights to the nested subspaces depending on their contribution to the ensembled prediction and therefore provides a soft measure of the relative importance of the different dimensions. In that sense, it goes beyond the classical manifold assumption stating that data has one intrinsic dimension, and instead proposes a soft blend between dimensions that is adapted to the learning objective. This sheds light on the celebrated paper of Minka for the automatic choice of dimensionality in PCA~\citep[Section~5]{minka_automatic_2000}.
Many other ensembling methods are possible like gradient boosting, Bayesian model averaging and stacking.
The whole methodology is summarized in Algorithm~\ref{alg:flag_trick} and illustrated in Figure~\ref{fig:illustration}.
\begin{algorithm}
\caption{Flag Trick combined with Ensemble Learning}\label{alg:flag_trick}
\begin{algorithmic}
\Require $X \coloneqq \bmat{x_1 & \cdots & x_n}\in\R^{p \times n}$ a data matrix; $q_{1:d} \coloneqq (q_1, \dots, q_d)$ a flag signature; $f$ a subspace learning objective; (opt.) $Y \coloneqq \bmat{y_1 & \cdots & y_n}\in\R^{m \times n}$ a label matrix
\State ${\S}_{1:d}^* \gets \operatorname{argmin}_{\S_{1:d} \in \Fl(p, q_{1:d})} \, f(\Pi_{\S_{1:d}})$  \Comment{flag trick \eqref{eq:flag_problem} + optimization (Alg.~\ref{alg:GD})}
\For{k = 1 \dots d}
	\State $g_k^* \gets \operatorname{fit}(\Pi_{\S_k^*} X, Y)$  \Comment{learning on $q_k$-dimensional projected data}
	\State $Y_k^* \gets g_k^*(\Pi_{\S_k^*} X)$ \Comment{prediction on $q_k$-dimensional projected data}
\EndFor
\State $Y^* \gets \operatorname{ensembling}({Y}_1^*, \dots, {Y}_d^*)$  \Comment{weighted predictions}
\Ensure $Y^*$ the ensembled predictions
\end{algorithmic}
\end{algorithm}
\begin{figure}
	\centering
    \includegraphics[width=\linewidth]{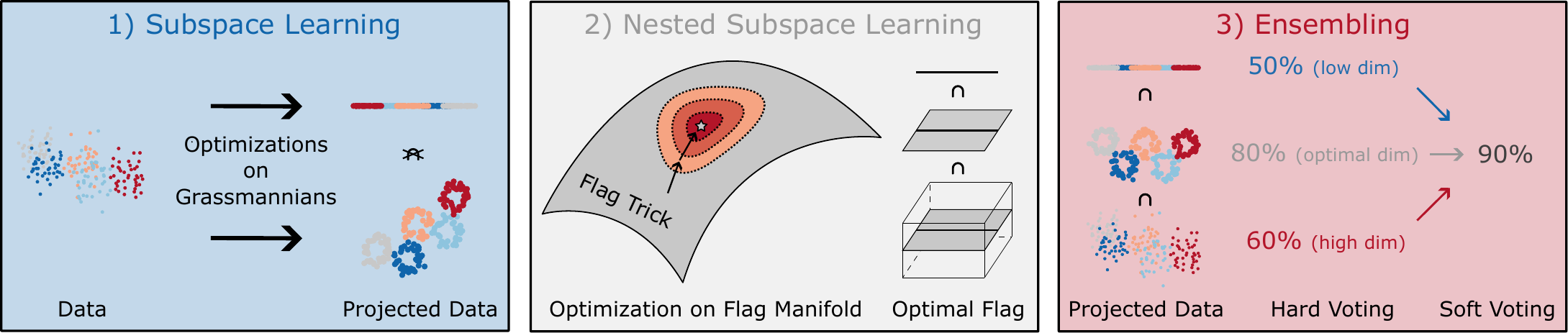}
    \caption{Illustration of the flag trick methodology. 1) We start with a subspace learning problem: $\argmin_{\S \in \Gr(p, q)} f(\Pi_\S)$. Trying different $q$'s yields in general non-nested subspaces, which raises an issue of consistency between data representations. 2) We convert the subspace learning problem into a nested subspace learning problem via the flag trick: $\argmin_{\S_{1:d} \in \Fl(p, q_{1:d})} f(\frac 1 d \sum_{k=1}^d \Pi_{\S_k})$. We run a steepest descent on flag manifolds (Algorithm~\ref{alg:GD}) and get a flag of nested subspaces. 3) We fit a machine learning algorithm (regression, classification, etc.) to the projected data at each dimension $q_k \in q_{1:d}$. We aggregate the estimators via ensembling methods (hard voting, soft voting, etc.) and get improved predictions.}
	\label{fig:illustration}
\end{figure}

\noindent Algorithm~\ref{alg:flag_trick} is a general proposition of multilevel machine learning with flags, but many other uses of the optimal flag $\S_{1:d}^*$ are possible, depending on the application. For instance, one may directly use the reweighted data matrix $\Pi_{\S_{1:d}^*} X$ as an input to the machine learning algorithm. This enables to fit only one model instead of $d$.
One can also simply analyze the projected data \textit{qualitatively} via scatter plots or reconstruction plots as evoked in Section~\ref{sec:intro}. The nestedness will automatically bring consistency contrarily to non-nested subspace methods, and therefore improve interpretability.
Finally, many other ideas can be borrowed from the literature on subspace clustering and flag manifolds~\citep{draper_flag_2014,launay_mechanical_2021, ma_flag_2021, mankovich_flag_2022,mankovich_chordal_2023,mankovich_fun_2024,mankovich_flag_2025}, for instance the computation of distances between flags coming from different datasets as a multilevel measure of similarity between datasets.

\begin{remark}[Choice of Signature]\label{rk:choice_signature}
    In the following experiments, the flag signatures are chosen heuristically. Although important, we believe that such a choice must be thoroughly addressed for each subspace learning problem, which leads us to leave that for future research. Alternative methodologies for the automatic choice of dimensions are discussed in Section~\ref{sec:discussion}. One could notably investigate penalties on the spectral gaps in the same spirit as~\citet{szwagier_eigengap_2025}---where the penalty factor would be tuned according to the downstream machine learning task. Beyond spectral gaps, the signature could be selected by starting from a flag of signature $(p, (1, 2, \dots, p-1))$ and removing the nested subspaces that do not significantly increase the accuracy on the downstream machine learning task. For the specific example of subspace averaging, a hyperparameter-free signature selection method extending the order fitting rule of~\citet{santamaria_order_2016} could be investigated via the flag trick.
\end{remark}

\section{The Flag Trick in Action}\label{sec:examples}
In this section, we provide some applications of the flag trick to several learning problems. We choose to focus on subspace recovery, trace ratio and spectral clustering problems. Other ones, like domain adaptation, matrix completion and subspace tracking are developed or mentioned in the last subsection but not experimented for conciseness.

\subsection{Outline and Experimental Setting}
For each application, we first present the learning problem as an optimization on Grassmannians. Second, we formulate the associated flag learning problem by applying the flag trick (Definition~\ref{def:flag_trick}). Third, we optimize the problem on flag manifolds with the steepest descent method (Algorithm~\ref{alg:GD})---more advanced algorithms are also derived in the appendix. 
Finally, we perform various nestedness and ensemble learning experiments via Algorithm~\ref{alg:flag_trick} on both synthetic and real datasets.

The general methodology to compare Grassmann-based methods to flag-based methods is the following one. For each experiment, we first choose a flag signature $~{q_{1:d} \coloneqq (q_1, \dots, q_d)}$, then we run independent optimization algorithms on $\Gr(p, q_1), \dots, \Gr(p, q_d)$~\eqref{eq:subspace_problem} and finally we compare the optimal subspaces $\S_k^* \in \Gr(p, q_k)$ to the optimal flag of subspaces $\Sf^* \in \Fl(p, q_{1:d})$ obtained via the flag trick~\eqref{eq:flag_problem}. 
To show the nestedness issue in Grassmann-based methods, we compute the subspace distances $\Theta(\S_k^*, \S_{k+1}^*)_{k=1\dots d-1}$, where $\Theta$ is the generalized Grassmann distance of~\citet[Eq.~(14)]{ye_schubert_2016}. It consists in the $\ell^2$-norm of the principal angles, which can be obtained from the singular value decomposition (SVD) of the inner-product matrices ${U_k}\T {U_{k+1}}$---$U_k \in \St(p, q_k)$ being an orthonormal basis of $\S_k^*$.

Regarding the implementation of the steepest descent algorithm on flag manifolds (Algorithm~\ref{alg:GD}), we develop a new class of manifolds in \href{https://pymanopt.org/}{PyManOpt}~\citep{boumal_manopt_2014,townsend_pymanopt_2016}, and run their \href{https://github.com/pymanopt/pymanopt/blob/master/src/pymanopt/optimizers/steepest_descent.py}{SteepestDescent} algorithm. Our implementation of the \texttt{Flag} class is based on the Stiefel representation of flag manifolds, detailed in Section~\ref{sec:flags}, with the retraction being the polar retraction. For the computation of the gradient, we use automatic differentiation with the \texttt{\href{https://github.com/HIPS/autograd}{autograd}} package. We could derive the gradients by hand from the expressions we get, but we use automatic differentiation as strongly suggested in PyManOpt's \href{https://pymanopt.org/docs/stable/quickstart.html}{documentation}.
Finally, the real datasets and the machine learning methods used in the experiments can be found in \href{https://scikit-learn.org/stable/}{scikit-learn}~\citep{pedregosa_scikit-learn_2011}.

\subsection{The Flag Trick for Robust Subspace Recovery Problems}
As shown in introduction and particularly in Figure~\ref{fig:motivations}, robust subspace recovery (RSR) methods based on Grassmannian optimization might suffer from the nestedness issue, which notably yields inconsistent data representations across different dimensions (as seen on the left part of the figure, where outliers somewhat swap position with inliers from 1D to 2D). The RSR overview by~\citet{lerman_overview_2018} particularly emphasizes the need for nested subspace methods, although it leaves this as an open perspective. In this subsection, we address the nestedness issue in robust subspace recovery by applying the flag trick (Definition~\ref{def:flag_trick}) to the problem of absolute deviation minimization~\citep{maunu_well-tempered_2019}.

Robust subspace recovery is an outlier-robust extension of classical dimension reduction methods such as PCA. Let us consider a dataset that is a union of \textit{inliers} and \textit{outliers}---the inliers are assumed to lie near a low-dimensional subspace $\S$ while the outliers live in the ambient space. The aim of RSR is to recover $\S$. Without further specifications, the RSR problem might not be well-posed. Therefore, the works in this domain often have to make some assumptions on the inlier and outlier distributions in order to obtain convergence and recovery guarantees. For instance, in \citet{lerman_overview_2018}, it is assumed that the inliers ``fill'' the lower-dimensional subspace and that the outliers are not much ``aligned''; this is rigorously defined in~\citet{lerman_robust_2015} and~\citet{maunu_well-tempered_2019} through \textit{permeance} and \textit{alignment} statistics. 

A generative model---the \textit{haystack model}---following these assumptions is introduced in~\citet{lerman_robust_2015}. It assumes an isotropic Gaussian distribution on the subspace for the inliers and an isotropic Gaussian distribution on the (full) ambient space for the outliers. A more realistic model---the \textit{generalized haystack model}---is introduced in~\citet{maunu_well-tempered_2019} to circumvent the simplistic nature of the haystack model. This one assumes general (anisotropic) Gaussian distributions for the inliers and outliers.\footnote{The generalized haystack model is more-precisely defined as follows: the $n_\mathrm{in}$ inliers are i.i.d sampled from a sub-Gaussian distribution $\N{0, \Sigma_\mathrm{in}/q}$, with $\Sigma_\mathrm{in} = U \Lambda U\T$, $U\in\St(p, q)$ and $\Lambda \in\diag{\R^q}$, while the $n_\mathrm{out}$ outliers are i.i.d sampled from a sub-Gaussian distribution $\N{0, \Sigma_\mathrm{out} / \operatorname{rk}(\Sigma_\mathrm{out})}$, with $\Sigma_\mathrm{out} \in \Sym(p)$. The (original) haystack model is a particular case of the generalized haystack model with isotropic outliers and (subspace-)isotropic inliers, i.e., $\Sigma_\mathrm{out} = \sigma^2_\mathrm{out} \, I_p / p$ and $\Sigma_\mathrm{in} = \sigma^2_\mathrm{in} \, U U\T / q$.} This makes the learning harder, since the anisotropy may keep the inliers from properly permeating the low-dimensional subspace---as discussed in Section~\ref{sec:intro}. Therefore, one has to make some stronger assumptions on the inlier-outlier ratio and the covariance eigenvalues distributions to derive some convergence and recovery guarantees.
\begin{remark}[Parametrization of RSR Generative Models]
The inlier distribution in the haystack model follows the \textit{isotropic PPCA} model~\citep{bouveyron_hddc_2007, bouveyron_intrinsic_2011}, while it follows the \textit{PPCA} model~\citep{tipping_probabilistic_1999} in the case of the generalized haystack model. Both models are a special case of the principal subspace analysis models~\citep{szwagier_curse_2025}. However, as argued in~\citet{szwagier_curse_2025}, while the haystack model is parameterized with Grassmannians, the generalized haystack model---which has more degrees of freedom accounting for the anisotropy---is parameterized with Stiefel manifolds. Therefore, from a statistical modeling perspective, it only makes sense to conduct subspace learning experiments on the haystack model and not the generalized one.
\end{remark}

\subsubsection{Application of the Flag Trick to RSR}\label{subsubsec:FT_RSR}
Among the large family of methods for robust subspace recovery presented in~\citet{lerman_overview_2018}, we consider the one of \textit{least absolute deviation} (LAD):
\begin{equation}\label{eq:RSR_Gr}
    \argmin_{\S \in \Gr(p, q)} \sum_{i=1}^n \norm{x_i - \Pi_{\S} x_i}_2.
\end{equation}
It is motivated by the sensitivity of squared norms to outliers, which makes PCA~\eqref{eq:PCA_subspace} unsuitable for outlier-contaminated data.
Problem~\eqref{eq:RSR_Gr} has the advantage of being rotationally invariant~\citep{ding_r1-pca_2006} but the drawback of being NP-hard~\citep{mccoy_two_2011, lerman_overview_2018} and obviously non-convex---since Grassmannians are not. A first body of works relaxes the problem, for instance by optimizing on the convex hull of Grassmannians~\citep{mccoy_two_2011, xu_robust_2012, zhang_novel_2014, lerman_robust_2015}.
A second body of works directly optimizes the LAD criterion on Grassmannians, either with an IRLS algorithm~\citep{lerman_fast_2018} or with a geodesic gradient descent~\citep{maunu_well-tempered_2019}, both achieving very good results in terms of recovery and speed.
The following proposition applies the flag trick to the LAD problem~\eqref{eq:RSR_Gr}.
\begin{proposition}[Flag Trick for RSR]\label{prop:RSR}
The flag trick transforms the least absolute deviation problem~\eqref{eq:RSR_Gr} into:
\begin{equation}\label{eq:RSR_Fl}
    \argmin_{\S_{1:d} \in \Fl(p, q_{1:d})} \sum_{i=1}^n \norm{x_i - \Pi_{\S_{1:d}} x_i}_2,
\end{equation} 
which is equivalent to the following optimization problem:
\begin{equation}\label{eq:RSR_Fl_St}
	\argmin_{U_{1:d} \in \St(p, q)} \sum_{i=1}^n \sqrt{\norm{x_i}_2^2 - \norm{{U_{1:d}}\T x_i}_2^2 + \sum_{k=1}^{d} \lrp{\frac {k-1} {d}}^2 \norm{{U_k}\T x_i}_2^2}.
\end{equation}
\end{proposition}
\begin{proof}
The proof is given in Appendix~\ref{app:RSR}.
\end{proof}
\begin{remark}[Generalization of RSR]
We can check that Equation~\eqref{eq:RSR_Fl_St} generalizes the original RSR problem \eqref{eq:RSR_Gr} from~\citet{maunu_well-tempered_2019}. Indeed, if we consider the particular case $d=1$, $U \coloneqq U_{1:d}$, $\S \coloneqq \operatorname{Span}(U_1)$ (Grassmannian), then we get:
\begin{multline}
    \sum_{i=1}^n \sqrt{\norm{x_i}_2^2 - \norm{{U_{1:d}}\T x_i}_2^2 + \sum_{k=1}^{d} \lrp{\frac {k-1} {d}}^2 \norm{{U_k}\T x_i}_2^2} = \sum_{i=1}^n \sqrt{\norm{x_i}_2^2 - \norm{{U_{1:d}}\T x_i}_2^2 + 0} =\\ \sum_{i=1}^n \norm{(I_p - U U\T) x_i}_2 = \sum_{i=1}^n \norm{x_i - \Pi_\S x_i}_2. 
\end{multline}
\end{remark}
Hence, \eqref{eq:RSR_Fl_St} is a multilevel generalization of the RSR problem from~\citet{maunu_well-tempered_2019}. The new term under the square root, $\sum_{k=1}^{d} \lrp{\frac {k-1} {d}}^2 \norm{{U_k}\T x_i}_2^2$, is nonzero when $d>1$ (i.e., for flag manifolds refining Grassmannians). This additional term corresponds to a weighted sum over the blocks of the orthonormal frame $U_{1:d}\in\St(p, q)$, with increasing weights $((k-1)/d)^2$. Therefore, intuitively, the variance in $\mathrm{Im}(U_k^*)$ is expected to decrease with $k$. In other words, the points that are far from the center---i.e., with large norm---should be in the first principal subspaces. Hence, we can interpret our new criterion as a robust generalization of the nested PCA (Theorem~\ref{thm:flag_trick}), with a least influence of outliers due to the square root.

Moreover, we can see that numerical issues are less prone to occur with the flag-tricked problem~\eqref{eq:RSR_Fl} than with the original one~\eqref{eq:RSR_Gr}. Indeed, the numerical issues of the square root arise when the radicand ${\|x_i - \frac{1}{d}\sum_{k=1}^d \Pi_{\S_k} x_i\|}_2^2$ is close to zero. By nestedness, this is the case if and only if $\S_1$ is close to $x_i$. Consequently, while numerical issues arise in classical RSR~\eqref{eq:RSR_Gr} whenever the subspace $\S$ is close to any data point, they arise in flag-tricked RSR~\eqref{eq:RSR_Fl} whenever both the subspace $\S$ \textit{and} the smaller nested subspaces are close to that data point. Hence, whenever $q_1\coloneqq \min \qf$ is smaller than the dimension $q$ that one would have tried for classical RSR, the non-differentiability and exploding-gradient issues caused by the square root are less likely to occur.

Finally, since the flag-tricked problem~\eqref{eq:RSR_Fl} is nothing but a robust version of the nested PCA of Theorem~\ref{thm:flag_trick} ($\sum_{i=1}^n {\|x_i - \Pi_{\Sf} x_i\|}_2^2$), a natural idea can be to initialize the optimization algorithm with the nested PCA solution. This is what is done in~\citet{maunu_well-tempered_2019} for the original problem~\eqref{eq:RSR_Gr}, and it is coming with exact recovery guarantees.

\subsubsection{Nestedness Experiments for RSR}
We first consider a dataset consisting in a mixture of two multivariate Gaussians: the inliers, with zero mean, covariance matrix $\diag{5, 1, .1}$, $n_\mathrm{in} = 450$ and the outliers, with zero mean, covariance matrix $\diag{.1, .1, 5}$, $n_\mathrm{out} = 50$. The dataset is therefore following the generalized haystack model of~\citet{maunu_well-tempered_2019}.
The ambient dimension is $p = 3$ and the intrinsic dimensions that we try are $q_{1:2} = (1, 2)$.
We run Algorithm~\ref{alg:GD} on Grassmann manifolds to solve the LAD problem~\eqref{eq:RSR_Gr}, successively for $q_1 = 1$ and $q_2 = 2$. Then we plot the projections of the data points onto the optimal subspaces. We compare them to the nested projections onto the optimal flag output by running Algorithm~\ref{alg:GD} on $\Fl(3, (1, 2))$ to solve~\eqref{eq:RSR_Fl}. The results are shown in Figure~\ref{fig:RSR_nested}.
\begin{figure}[t]
	\centering
    \includegraphics[width=\linewidth]{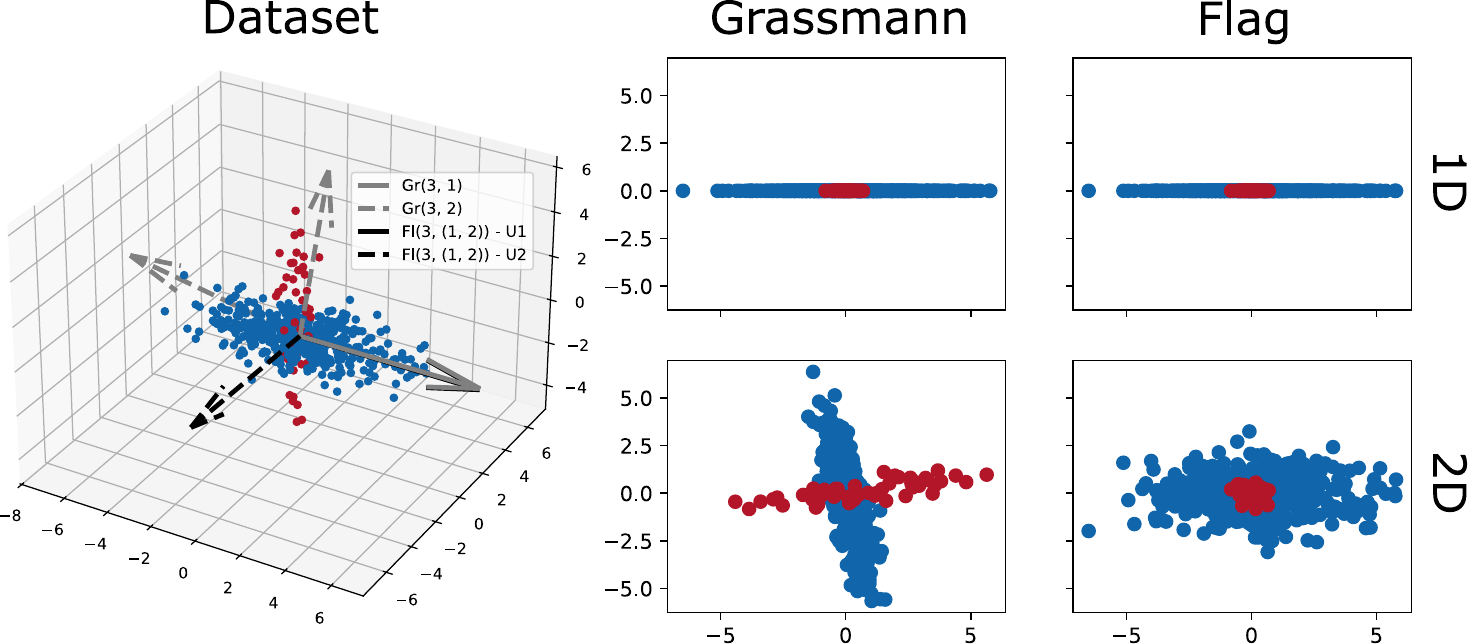}
    \caption{
    Illustration of the nestedness issue in robust subspace recovery. Given a dataset consisting in a mixture of inliers (blue) and outliers (red) we plot its projection onto the optimal 1D subspace and 2D subspace obtained by solving the associated Grassmannian optimization problem~\eqref{eq:RSR_Gr} or flag optimization problem~\eqref{eq:RSR_Fl}. 
    We can see that the Grassmann representations are not nested, while the flag representations are nested and robust to outliers.}
	\label{fig:RSR_nested}
\end{figure}
\begin{figure}[t]
	\centering
    \includegraphics[width=\linewidth]{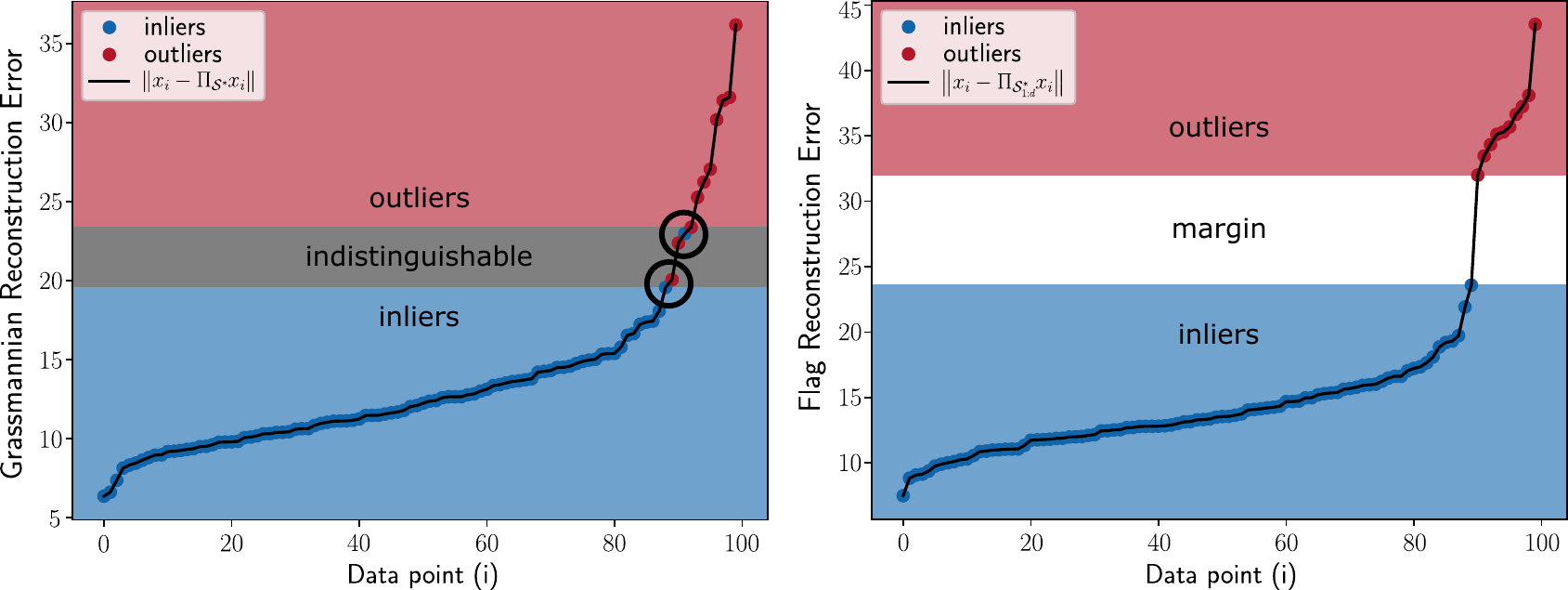}
    \caption{Euclidean reconstruction errors (sorted in ascending order) on the corrupted digits dataset for robust subspace recovery ($\|x_i - \Pi_{\S^*} x_i\|$, left) and {its flag-tricked version} ($\|x_i - \Pi_{\S_{1:d}^*} x_i\|$, right). 
    With the Grassmann-based method, the distributions of reconstruction errors for the inliers and outliers intersect, meaning that we cannot fully distinguish the inliers from the outliers (see the gray zone on the left plot). In contrast, with the flag-based method, the distributions of reconstruction errors for the inliers and outliers are clearly separated, meaning that we can easily distinguish the inliers from the outliers (see the white zone on the right plot). This {phenomenon} can be explained by the multilevel nature of the flag trick (see the main text for more details).}
	\label{fig:RSR_outlier}
\end{figure}
We can see that the Grassmann-based projections are non-nested while their flag counterparts are not only nested but also robust to outliers. This could be explained by the nestedness constraint of flag manifolds which imposes the 2D subspace to contain the 1D subspace.

Second, we perform an outlier detection experiment. A common methodology to detect outliers in a corrupted dataset is to first look for an outlier-robust subspace and then plot the distribution of distances between the data points and their projection onto the recovered subspace. This distribution is expected to show a clear gap between the inliers and outliers~\citep[Fig.~3.7]{vidal_generalized_2016}.
However, in practice, one does not know which subspace dimension $q$ to choose. If $q$ is too large, then the recovered subspace may contain both inliers and outliers, and therefore the distribution of distances might be roughly $0$. In contrast, if $q$ is too small, then some inliers may lie too far from the recovered subspace and be detected as outliers. An idea in the spirit of the flag trick is to perform an average ensembling of the reconstruction errors. More specifically, if $\norm{x_i - \Pi_\S x_i}_2$ is the classical measure of robust reconstruction error, then we compute $\norm{x_i - \Pi_{\S_{1:d}} x_i}_2$. Such a score extracts information from projections at different levels and might result in a richer multilevel analysis.
We consider a dataset where the inliers are images of $8 \times 8$ handwritten $0$'s and outliers correspond to other digits from $1$ to $9$, all extracted from the classical handwritten digits dataset~\citep{alpaydin_optical_1998}.
The ambient dimension is $p = 64$, the number of inliers is $n_\mathrm{in} = 90$ and the number of outliers is $n_\mathrm{out} = 10$. The intrinsic dimensions that we try are $q_{1:3} = (1, 2, 5)$.
We plot the distribution of reconstruction errors for the points of the digits dataset on the optimal flag $\S_{1:3}^* \in \Fl(p, (1, 2, 5))$ in Figure~\ref{fig:RSR_outlier}. We compare it to the distribution of reconstruction errors on $\S^* \in \Gr(p, 5)$.
We can see that the flag trick enables to clearly distinguish the inliers from the outliers compared to the Grassmann-based method, which is a consequence of the multilevel nature of flags.

\subsubsection{Discussion on RSR Optimization and Objective Functions}\label{subsubsec:RSR_discu}
To substantiate the generality of the flag trick, we propose two extensions of the previously proposed multilevel RSR method: an IRLS optimization algorithm and a collection of new multilevel criteria. The experimentation of these extensions is left to future work since the goal is rather to show how the flag trick principle can be easily applied in different contexts.

\paragraph{An IRLS algorithm}
In all the experiments of this paper, we use a steepest descent method on flag manifolds (Algorithm~\ref{alg:GD}) to solve the flag problems.
However, for the specific problem of RSR~\eqref{eq:RSR_Fl}, we believe that more adapted algorithms should be derived, notably due to the non-differentiability and exploding-gradient issues caused by the square root.
To that extent, we derive in Appendix~\ref{app:RSR} an IRLS scheme (Algorithm~\ref{alg:FMF}) for RSR. In short, the RSR problem~\eqref{eq:RSR_Fl} can be reformulated as a weighted least squares problem $\sum_{i=1}^n w_i \norm{x_i - \Pi_{\Sf} x_i}_2^2$ with $w_i ={1}/{\norm{x_i - \Pi_{\Sf} x_i}_2}$ and optimized iteratively, with explicit expressions obtained via our central Theorem~\ref{thm:flag_trick}. We insist on the fact that such an IRLS algorithm could not be developed with the flagification of~\citet{pennec_barycentric_2018} and~\citet{mankovich_fun_2024}, since a sum of square roots does not correspond to a least squares problem.

\paragraph{More RSR problems}
In this work, we explore one specific problem of RSR for conciseness, but we could investigate many other related problems, including Robust PCA. 
Notably, drawing from the Grassmann averages (GA) method~\citep{hauberg_scalable_2016}, one could develop many new multilevel RSR and Robust PCA objective functions.
The idea behind GA is to replace data points with 1D subspaces ($\S_i = \operatorname{Span}(x_i)$) and then perform subspace averaging methods to find a robust prototype for the dataset. GA ends up solving problems of the form $\operatorname{argmin}_{\S \in \Gr(p, 1)} \sum_{i=1}^n w_i \, \operatorname{dist}_{\Gr(p, 1)}^2(\operatorname{Span}(x_i), \S)$, where $w_i$ are some weights and $\operatorname{dist}_{\Gr(p, 1)}$ is a particular subspace distance detailed in~\citet{hauberg_scalable_2016}. Using instead some multidimensional subspace distances, like the principal angles and its variants~\citep{hamm_grassmann_2008, ye_schubert_2016}, we can develop many variants of the Grassmann averages, of the form $\operatorname{argmin}_{\Sf \in\Gr(p, q)} \sum_{i=1}^n w_i \, \rho(\sqrt{{x_i}\T  \Pi_{\S} {x_i}})$, where $\rho\colon\R\to\R$ is a real function, like $\rho(x) = \arccos(x)$ if we want subspace-angle-like distances, $\rho(x) = - x^2$ if we want PCA-like solutions, and many other possible robust variants.
Applying the flag trick to those problems yields the following robust multilevel problem: $\operatorname{argmin}_{\Sf \in\Fl(p, \qf)} \sum_{i=1}^n w_i \rho(\sqrt{{x_i}\T  \Pi_{\Sf} {x_i}})$.

\subsection{The Flag Trick for Trace Ratio Problems}\label{subsec:TR}
As shown in introduction and particularly in Figure~\ref{fig:motivations}, trace ratio methods based on Grassmannian optimization such as linear discriminant analysis might suffer from the nestedness issue. This notably yields inconsistent data representations across different dimensions, as seen on the middle part of the figure, where the clusters get rotated and reflected, leading to cluster identification issues. In this subsection, we address the nestedness issue by applying the flag trick (Definition~\ref{def:flag_trick}) to trace ratio problems~\citep{ngo_trace_2012}.

Trace ratio (TR) problems are ubiquitous in machine learning~\citep{ngo_trace_2012}:
\begin{equation}\label{eq:TR_St}
\argmax_{U \in \St(p, q)} \frac{\tr{U\T A U}}{\tr{U\T B U}},
\end{equation}
where $A, B \in \R^{p\times p}$ are positive semi-definite matrices, with $\operatorname{rank}(B) > p - q$. A famous example of trace ratio problem is Fisher's linear discriminant analysis (LDA)~\citep{fisher_use_1936,belhumeur_eigenfaces_1997}.
It is common in machine learning to project the data onto a low-dimensional subspace before fitting a classifier, in order to circumvent the curse of dimensionality. It is well known that performing an unsupervised dimension reduction method like PCA comes with the risks of mixing up the classes, since the directions of maximal variance are not necessarily the most discriminating ones~\citep{chang_using_1983}. The goal of LDA is to use the knowledge of the data labels to learn a linear subspace that does not mix the classes.
Let $~{X \coloneqq \bmat{x_1 & \cdots & x_n} \in \R^{p\times n}}$ be a dataset with labels $~{Y \coloneqq \bmat{y_1 & \cdots & y_n} \in {\{1, \dots, C\}}^n}$. Let $\mu = \frac{1}{n} \sum_{i=1}^n x_i$ be the dataset mean and $\mu_c = \frac{1}{\#\{i\colon y_i=c\}}\sum_{i\colon y_i=c} x_i$ be the class-wise means. 
The idea of LDA is to search for a subspace $\S \in \Gr(p, q)$ that simultaneously maximizes the projected \textit{between-class variance} $\sum_{c=1}^C {\|\Pi_\S \mu_c - \Pi_\S \mu\|}_2^2$ and minimizes the projected \textit{within-class variance} $\sum_{c=1}^C \sum_{i\colon y_i = c} {\|\Pi_\S x_i - \Pi_\S \mu_c\|}_2^2$. This can be reformulated as a trace ratio problem~\eqref{eq:TR_St}, with $A = \sum_{c=1}^C (\mu_c - \mu) (\mu_c - \mu)\T$ and $B = \sum_{c=1}^C \sum_{i\colon y_i = c} (x_i - \mu_c) (x_i - \mu_c)\T$.

More generally, a large family of dimension reduction methods can be reformulated as a TR problem. The seminal work of~\citet{yan_graph_2007} shows that many dimension reduction and manifold learning objective functions can be written as a trace ratio involving Laplacian matrices of attraction and repulsion graphs. Intuitively, those graphs determine which points should be close in the latent space and which ones should be far apart.
Other methods involving a ratio of traces are \textit{multi-view learning}~\citep{wang_trace_2023}, \textit{partial least squares} (PLS)~\citep{geladi_partial_1986,barker_partial_2003} and \textit{canonical correlation analysis} (CCA)~\citep{hardoon_canonical_2004}, although these methods are originally \textit{sequential} problems (cf. Remark~\ref{rem:sequential}) and not \textit{subspace} problems.

Classical Newton-like algorithms for solving the TR problem~\eqref{eq:TR_St} come from the seminal works of~\citet{guo_generalized_2003, wang_trace_2007, jia_trace_2009}.
The interest of optimizing a trace-ratio instead of a ratio-trace (of the form $\tr{(U\T B U)^{-1}(U\T A U)}$)---which enjoys an explicit solution given by a generalized eigenvalue decomposition---is also tackled in those papers. The \textit{repulsion Laplacians}~\citep{kokiopoulou_enhanced_2009} instead propose to solve a regularized version $\tr{U\T B U} - \rho \tr{U\T A U}$, which enjoys an explicit solution but now has a hyperparameter $\rho$---whereas the latter is directly optimized in the previous approaches.

\subsubsection{Application of the Flag Trick to Trace Ratio Problems}
The trace ratio problem~\eqref{eq:TR_St} can be straightforwardly reformulated as an optimization problem on Grassmannians, due to the orthogonal invariance of the objective function:
\begin{equation}\label{eq:TR_Gr}
\argmax_{\S \in \Gr(p, q)} \frac{\tr{\Pi_\S A}}{\tr{\Pi_\S B}}.
\end{equation}
The following proposition applies the flag trick to the TR problem~\eqref{eq:TR_Gr}.
\begin{proposition}[Flag Trick for TR]\label{prop:TR}
The flag trick transforms the trace ratio problem \eqref{eq:TR_Gr} into: 
\begin{equation}\label{eq:TR_Fl}
	\argmax_{\S_{1:d} \in \Fl(p, q_{1:d})} \frac{\tr{\Pi_{\S_{1:d}} A}}{\tr{\Pi_{\S_{1:d}} B}},
\end{equation}
which is equivalent to the following optimization problem:
\begin{equation}\label{eq:TR_Fl_equiv}
\argmax_{U_{1:d} \in \St(p, q)} \frac{\sum_{k=1}^{d} (d - (k-1)) \tr{{U_k}\T A {U_k}}}{\sum_{l=1}^{d} (d - (l-1)) \tr{{U_{l}}\T B {U_{l}}}}.
\end{equation}
\end{proposition}
\begin{proof}
The proof is given in Appendix~\ref{app:TR}.
\end{proof}
Equation~\eqref{eq:TR_Fl_equiv} tells us several things. First, the subspaces $~{\operatorname{Span}(U_1) \perp \dots \perp \operatorname{Span}(U_d)}$ are weighted decreasingly, which means that they have less and less importance with respect to the TR objective.
Second, we can see that the nested trace ratio problem~\eqref{eq:TR_Fl} somewhat maximizes the numerator $\tr{\Pi_{\S_{1:d}} A}$ while minimizing the denominator $\tr{\Pi_{\S_{1:d}} B}$. Both subproblems have an explicit solution corresponding to our nested PCA Theorem~\ref{thm:flag_trick}. Hence, one can naturally initialize the steepest descent algorithm with the $q$ highest eigenvalues of $A$ or the $q$ lowest eigenvalues of $B$ depending on the application.
For instance, for LDA, initializing Algorithm~\ref{alg:GD} with the highest eigenvalues of $A$ would spread the classes far apart, while initializing it with the lowest eigenvalues of $B$ would concentrate the classes, which seems less desirable since we do not want the classes to concentrate at the same point.

\subsubsection{Nestedness Experiments for Trace Ratio Problems}
We run some nestedness and classification experiments for the specific Trace Ratio problem of LDA. Many other applications (\textit{marginal Fisher analysis}~\citep{yan_graph_2007}, \textit{local discriminant embedding}~\citep{chen_local_2005}, etc.) could be similarly investigated.

First, we consider a synthetic dataset with five clusters.
The ambient dimension is $p = 3$ and the intrinsic dimensions that we try are $q_{1:2} = (1, 2)$.
We adopt a preprocessing strategy similar to~\citet{ngo_trace_2012}: we first center the data, then run a PCA to reduce the dimension to $n - C$ (if $n - C < p$), then construct the LDA scatter matrices $A$ and $B$, then add a diagonal covariance regularization of $10^{-5}$ times their trace and finally normalize them to have unit trace.
We run Algorithm~\ref{alg:GD} on Grassmann manifolds to solve the TR maximization problem~\eqref{eq:TR_Gr}, successively for $q_1 = 1$ and $q_2 = 2$. Then we plot the projections of the data points onto the optimal subspaces. We compare them to the nested projections onto the optimal flag output by running Algorithm~\ref{alg:GD} on $\Fl(3, (1, 2))$ to solve~\eqref{eq:TR_Fl}. The results are shown in Figure~\ref{fig:TR_nested}. 
We can see that the Grassmann representations are non-nested while their flag counterparts perfectly capture the filtration of subspaces that best and best approximates the distribution while discriminating the classes. Even if the colors make us realize that the issue in this experiment for LDA  is not much about the non-nestedness but rather about the rotation of the principal axes within the 2D subspace, we still have an important issue of consistency.
\begin{figure}[t]
	\centering
    \includegraphics[width=\linewidth]{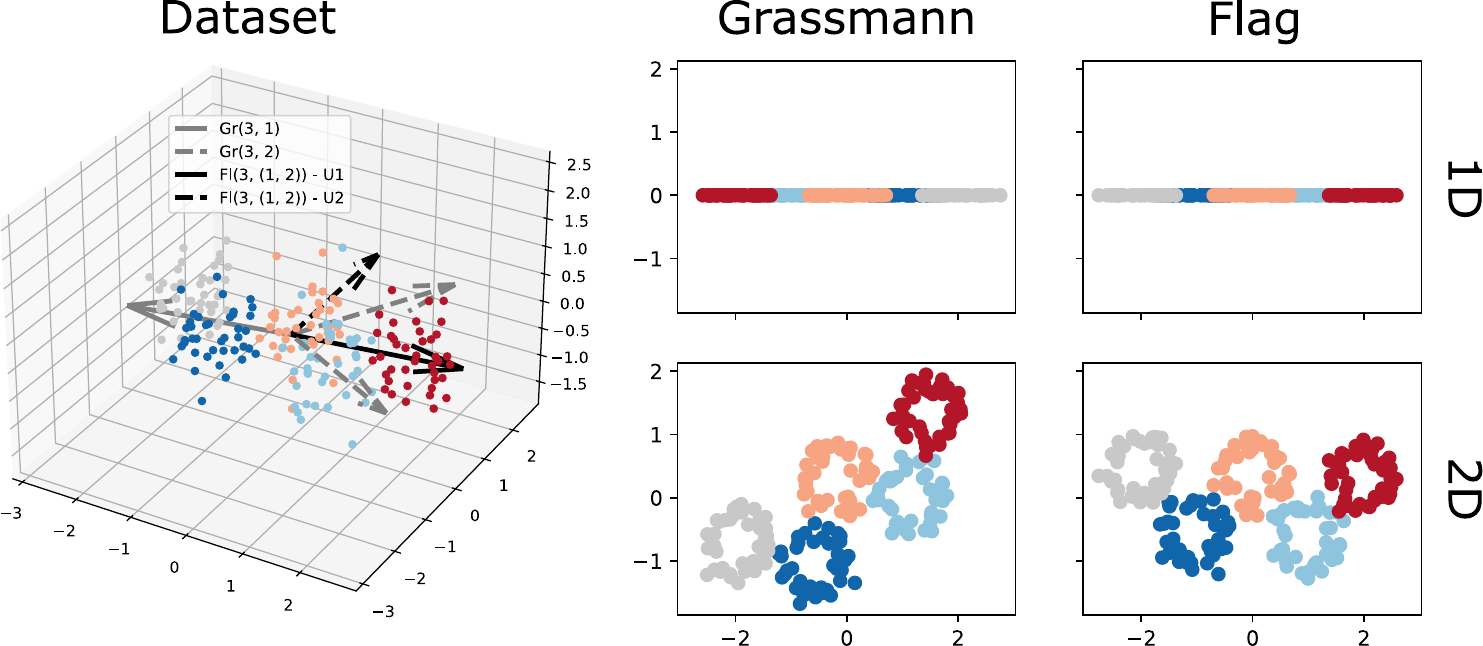}
    \caption{
    Illustration of the nestedness issue in linear discriminant analysis (trace ratio problem). Given a dataset with five clusters, we plot its projection onto the optimal 1D subspace and 2D subspace obtained by solving the associated Grassmannian optimization problem~\eqref{eq:TR_Gr} or flag optimization problem~\eqref{eq:TR_Fl}. 
    We can see that the Grassmann representations are not nested, while the flag representations are nested and well capture the distribution of clusters. In this example, it is less the nestedness than the \textit{rotation} of the optimal axes inside the 2D subspace that is critical to the analysis of the Grassmann-based method.
    }
	\label{fig:TR_nested}
\end{figure}

Second, we consider the (full) handwritten digits dataset~\citep{alpaydin_optical_1998}. It contains $8 \times 8$ pixels images of handwritten digits, from $0$ to $9$, almost uniformly class-balanced. One has $n = 1797$, $p=64$ and $C = 10$.
We run a steepest descent algorithm to solve the trace ratio problem~\eqref{eq:TR_Fl}. We choose the \textit{full signature} $q_{1:d} = (1, 2, \dots, 63)$ {with $d=63$} and compare the output flag to the individual subspaces output by running optimization on $\Gr(p, q_k)$ for $q_k \in q_{1:d}$.
We plot the subspace angles $\Theta(\S_k^*, \S_{k+1}^*)$ and the explained variance ${\operatorname{tr}(\Pi_{\S_k^*} X X\T)} / {\operatorname{tr}(X X\T)}$ as a function of the $k$. The results are illustrated in Figure~\ref{fig:TR_nested_digits}.
We see that the subspace angles are always positive and even very large sometimes with the LDA. Worst, the explained variance is not monotonic. This implies that we sometimes \textit{lose} some information when \textit{increasing} the dimension, which is extremely paradoxical.
\begin{figure}[t]
	\centering
    \includegraphics[width=\linewidth]{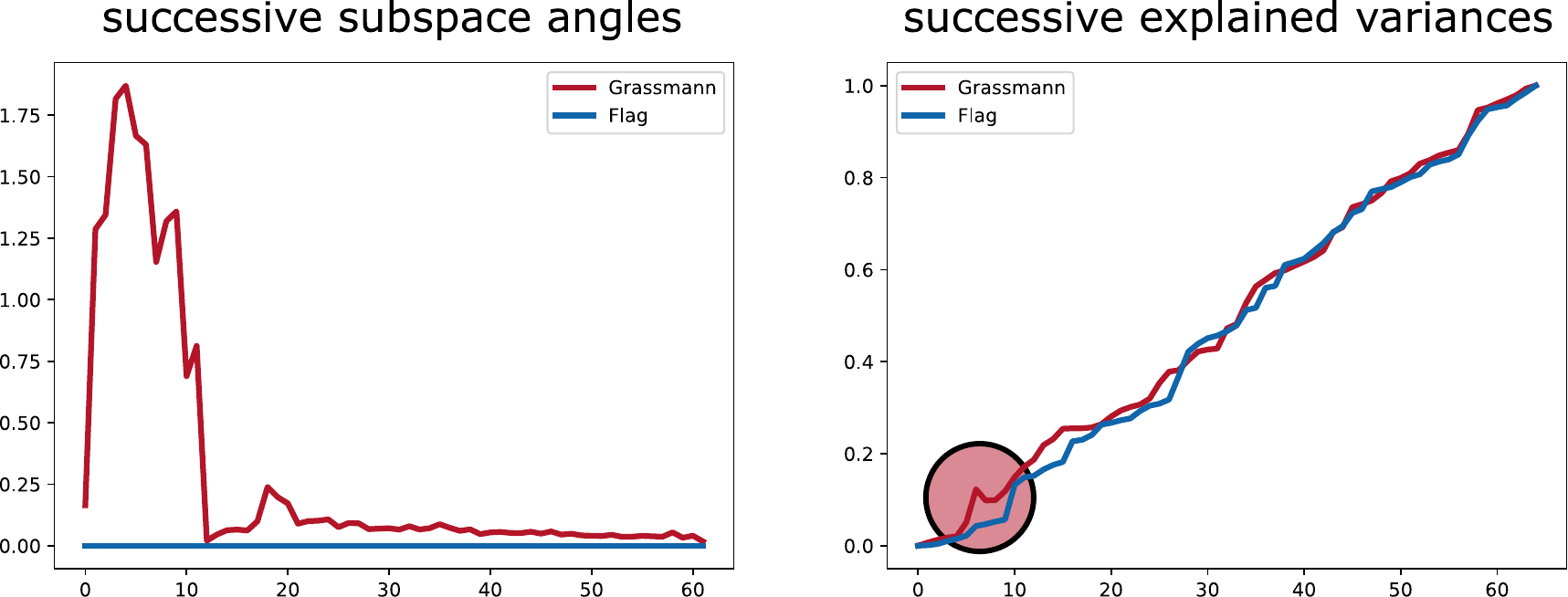}
    \caption{
    Illustration of the nestedness issue in linear discriminant analysis (trace ratio problem) on the digits dataset. For $q_k \in (1, 2, \dots, 63)$, we solve the Grassmannian optimization problem~\eqref{eq:TR_Gr} on $\Gr(64, q_k)$ and plot the subspace angles $\Theta(\S_k^*, \S_{k+1}^*)$ (left) and explained variances ${\operatorname{tr}(\Pi_{\S_k^*} X X\T)} / {\operatorname{tr}(X X\T)}$ (right) as a function of $k$. We compare those quantities to the ones obtained by solving the flag optimization problem~\eqref{eq:TR_Fl}. 
    We can see that the Grassmann-based method is highly non-nested and even yields an extremely paradoxical non-increasing explained variance (cf. red circle on the right).
    }
	\label{fig:TR_nested_digits}
\end{figure}

Third, we perform some classification experiments on the optimal subspaces. For each dataset, we run the optimization problems on $\Fl(p, q_{1:d})$, then project the data onto the different subspaces in $\S_{1:d}^*$ and run a nearest neighbors classifier with $5$ neighbors.
The predictions are then ensembled (cf. Algorithm~\ref{alg:flag_trick}) by weighted averaging, either with uniform weights or with weights minimizing the average cross-entropy:
\begin{equation}\label{eq:soft_voting}
	w_1^*, \dots, w_d^* = \argmin_{\substack{w_k \geq 0 \\ \sum_{k=1}^d w_k = 1}} - \frac 1 {n C} \sum_{i=1}^n \sum_{c=1}^C y_{ic} \ln\lrp{\sum_{k=1}^d w_k y_{kic}^*},
\end{equation}
where $y_{kic}^* \in [0, 1]$ is the predicted probability that $x_i \in \R^p$ belongs to class $c \in \{1, \dots, C\}$, by the classifier $g_k^*$ that is trained on $Z_k \coloneqq {U_k^*}\T X \in \R^{q_k \times n}$. One can show that the latter is a convex problem, which we optimize using the \href{https://www.cvxpy.org/index.html}{cvxpy} Python package~\citep{diamond2016cvxpy}.
We repeat the experiment 10 times in a stratified train-test fashion and report the average test cross-entropy in Table~\ref{tab:TR_classif}.
\begin{table}
  \caption{Results of the TR-based classification experiment. For each method (Gr: Grassmann optimization~\eqref{eq:TR_Gr}, Fl: flag optimization~\eqref{eq:TR_Fl}, Fl-U: flag optimization + uniform soft voting, Fl-W: flag optimization + optimal soft voting~\eqref{eq:soft_voting}), we give the average test cross-entropy between the predictions and the true labels.}
  \label{tab:TR_classif}
  \centering
  \begin{tabular}{ccccccccc}
    \toprule
    dataset & $n$ & $p$ & $q_{1:d}$ & Gr & Fl & Fl-U & Fl-W & weights\\
    \midrule
    digits & $1797$ & $64$ & $(1, 2, 5, 10)$ & $5.1$ & $4.6$ & $3.1$ & $2.9$ & $(0.03, 0.05, 0.30, 0.62)$\\
    wine & $178$ & $13$ & $(1, 2, 5)$ & $0.71$ & $0.69$ & $0.43$ & $0.29$ & $(0, 0.68, 0.32)$\\
    breast & $569$ & $30$ & $(1, 2, 5)$ & $0.534$ & $0.537$ & $0.485$ & $0.475$ & $(0.26, 0.24, 0.50)$\\
    iris & $150$ & $4$ & $(1, 2, 3)$ & $0.275$ & $0.271$ & $0.281$ & $0.265$ & $(0.27, 0.12, 0.62)$\\
    \bottomrule
  \end{tabular}
\end{table}
The wine example is particularly interesting. It first tells us that the optimal 5D subspace obtained by Grassmann optimization discriminates less between the classes than the 5D subspace from the optimal flag. This may show that the flag takes into account some lower dimensional variability that enables to better discriminate the classes. Second, the uniform averaging of the predictors at different dimensions improves the classification. Third, the optimal weights improve even more the classification and tell that the best discrimination occurs by taking a soft blend of classifiers at dimensions $2$ and $5$. Similar kinds of analyses can be made for the other examples.

\subsubsection{Discussion on TR Optimization and Kernelization}
To substantiate the generality of the flag trick, we propose two extensions of the previously proposed multilevel TR method: a Newton optimization algorithm and a kernelization. The experimentation of these extensions is left to future work since the goal is rather to show how the flag trick principle can be easily applied in different contexts.

\paragraph{A Newton algorithm}
In all the experiments of this paper, we use a steepest descent method on flag manifolds (Algorithm~\ref{alg:GD}) to solve the flag problems.
However, for the specific problem of TR~\eqref{eq:TR_Fl}, we believe that more adapted algorithms should be derived to take into account the specific form of the objective function, which is classically solved via a Newton-Lanczos method~\citep{ngo_trace_2012}. 
To that extent, we develop in the appendix (\ref{app:TR}) an extension of the baseline Newton-Lanczos algorithm for the flag-tricked problem~\eqref{eq:TR_Fl}.
In short, the latter can be reformulated as a penalized optimization problem of the form $\operatorname{argmax}_{\Sf\in\Fl(p, \qf)} {\sum_{k=1}^d \tr{\Pi_{\S_k} (A - \rho B)}}$, where $\rho$ is updated iteratively according to a Newton scheme. Once again, our central Theorem~\ref{thm:flag_trick} enables to get an explicit solution to the penalized optimization problem.

\paragraph{A non-linearization via the kernel trick}
The classical trace ratio problems look for \textit{linear} embeddings of the data.
However, in most cases, the data follow a \textit{nonlinear} distribution, which may cause linear dimension reduction methods to fail. The \textit{kernel trick}~\citep{hofmann_kernel_2008} is a well-known method to embed nonlinear data into a linear space and fit linear machine learning methods.
As a consequence, we propose in appendix (\ref{app:TR}) a kernelization of the trace ratio problem~\eqref{eq:TR_Fl} in the same fashion as the one of the seminal graph embedding work~\citep{yan_graph_2007}.
This is expected to yield much better embedding and classification results.

\subsection{The Flag Trick for Spectral Clustering Problems}
As shown in introduction and particularly in Figure~\ref{fig:motivations}, spectral clustering methods based on Grassmannian optimization might suffer from the nestedness issue, which notably yields inconsistent data representations across different dimensions. This can be seen on the right part of the figure, where the clusters are almost separable in 1D but surprisingly intertwine in 2D. In this subsection, we address the nestedness issue by applying the flag trick (Definition~\ref{def:flag_trick}) to Grassmann-based sparse spectral clustering~\citep{wang_grassmannian_2017}.

Spectral clustering~\citep{ng_spectral_2001} is a baseline clustering technique. It can be applied to general cluster distributions compared to the classical k-means (round clusters) and Gaussian mixture models (ellipsoidal clusters).
Given a dataset $X \coloneqq \bmat{x_1 & \cdots & x_n} \in \R^{p\times n}$, spectral clustering roughly consists in the eigenvalue decomposition of a Laplacian matrix $L\in\Sym(n)$ associated to a pairwise similarity matrix, for instance $M_{ij} = \exp(\norm{x_i - x_j}^2 / 2\sigma^2)$. The eigenvectors are then used as new embeddings for the data points, on which standard clustering algorithms like k-means can be performed. This method is closely related to the celebrated Laplacian eigenmaps~\citep{belkin_laplacian_2003} which are used for nonlinear dimension reduction.
The good performances of such a simple method as spectral clustering are theoretically justified by the particular structure of the Laplacian matrix $L$ in an ideal case---block-diagonal with a multiple eigenvalue related to the number of clusters~\citep{ng_spectral_2001}.
The recent \textit{sparse spectral clustering} (SSC) method~\citep{lu_convex_2016} builds on such an ideal case and encourages the block-diagonality by looking for a sparse and low-rank representation:
\begin{equation}\label{eq:SSC_lu}
	\argmin_{U \in \St(n, q)} \langle U U\T, L\rangle_F + \beta \norm{U U\T }_1,
\end{equation}
that they optimize over the convex hull of Grassmann manifolds with an ADMM algorithm.

\subsubsection{Application of the Flag Trick to Spectral Clustering}
The Grassmannian spectral clustering method~\citep{wang_grassmannian_2017} directly optimizes~\eqref{eq:SSC_lu} on Grassmann manifolds:
\begin{equation}\label{eq:SSC_Gr}
	\argmin_{\S \in \Gr(n, q)} \langle \Pi_\S, L\rangle_F + \beta \norm{\Pi_\S }_1.
\end{equation}
The authors use a Riemannian trust region method~\citep{absil_optimization_2008} for optimization and show the interest of directly optimizing over the Grassmann manifold instead of convexifying the optimization space~\citep{lu_convex_2016}.
The following proposition applies the flag trick to SSC.
\begin{proposition}[Flag Trick for SSC]\label{prop:SSC}
The flag trick transforms the sparse spectral clustering problem~\eqref{eq:SSC_Gr} into:
\begin{equation}\label{eq:SSC_Fl}
	\argmin_{\S_{1:d} \in \Fl(n, q_{1:d})} \langle \Pi_{\S_{1:d}}, L\rangle_F + \beta \norm{\Pi_{\S_{1:d}} }_1,
\end{equation}
which is equivalent to the following optimization problem:
\begin{equation}\label{eq:SSC_Fl_equiv}
	\argmin_{U_{1:d} \in \St(n, q)} \sum_{k=1}^d (d - (k-1)) \tr{{U_k}\T L U_k} + \beta \norm{\sum_{k=1}^d (d - (k-1)) U_k {U_k}\T}_1.
\end{equation}
\end{proposition}
\begin{proof}
The proof is given in Appendix~\ref{app:SSC}.
\end{proof}
We can see that the case $\beta=0$ corresponds to classical spectral clustering. Indeed, with a similar reasoning as in the proof of Theorem~\ref{thm:flag_trick}, we can easily show that the solution to problem~\eqref{eq:SSC_Fl} is explicit and corresponds to the flag of nested eigenspaces of $L$ (for increasing eigenvalues). Therefore, initializing the algorithm with the smallest $q$ eigenvectors of $L$ seems like a natural idea.
Moreover, one may intuitively analyze the relative weighting of the mutually-orthogonal subspaces $\operatorname{Span}(U_1)\perp\dots\perp\operatorname{Span}(U_d)$ in~\eqref{eq:SSC_Fl_equiv} as favoring a model with $q_1$ clusters, and then adding successively $q_{k} - q_{k-1}$ clusters to improve the modeling of the Laplacian matrix, with a tradeoff between too much and too few clusters.

\subsubsection{Nestedness Experiments for Spectral Clustering}
We run some nestedness experiments for the sparse spectral clustering problem.
First, we consider a 3D extension of the classical two-moons dataset for clustering with $n=100$.
The ambient dimension is $p = 3$ and the intrinsic dimensions that we try are $q_{1:2} = (1, 2)$.
We adopt a pre-processing strategy similar to~\citet{lu_convex_2016,wang_grassmannian_2017}: we compute the affinity matrix $W\in\R^{n \times n}$ using an exponential kernel with standard deviation being the median of the pairwise Euclidean distances between samples. Then we compute the normalized Laplacian matrix $L = I_n - D^{-\frac12} L D^{-\frac12}$ where $D\in\R^{n \times n}$ is a diagonal matrix with diagonal elements $D_{ii} = \sum_{j=1}^n w_{ij}$.
We run Algorithm~\ref{alg:GD} on Grassmann manifolds to solve the sparse optimization problem~\eqref{eq:SSC_Gr}, successively for $q_1 = 1$ and $q_2 = 2$. Then we plot the projections of the data points onto the optimal subspaces. We compare them to the nested projections onto the optimal flag output by running Algorithm~\ref{alg:GD} on $\Fl(3, (1, 2))$ to solve~\eqref{eq:SSC_Fl}. The results are shown in Figure~\ref{fig:SSC_nested}.
\begin{figure}[t]
	\centering
    \includegraphics[width=\linewidth]{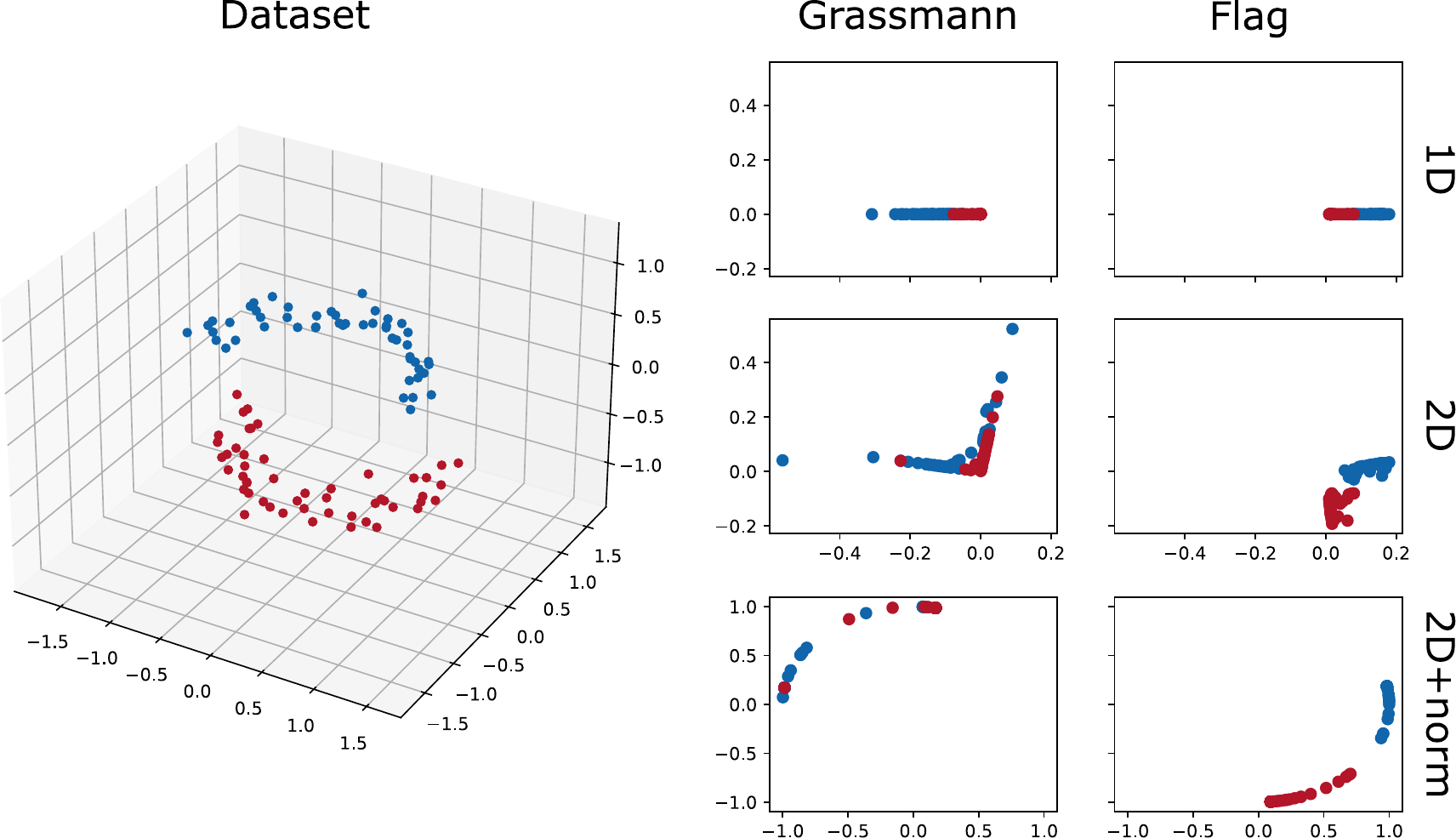}
    \caption{
    Illustration of the nestedness issue in sparse spectral clustering. Given a dataset with two half-moon-shaped clusters, we plot its projection onto the optimal 1D subspace and 2D subspace obtained by solving the associated Grassmannian optimization problem~\eqref{eq:SSC_Gr} or flag optimization problem~\eqref{eq:SSC_Fl}. 
    We can see that the Grassmann representations are not nested, while the flag representations are nested and much better clustered. 
    The last row corresponds to dividing the 2D embeddings by their norms, as commonly done in spectral clustering~\citep{ng_spectral_2001}.}
	\label{fig:SSC_nested}
\end{figure}
\begin{figure}[t]
	\centering
    \includegraphics[width=\linewidth]{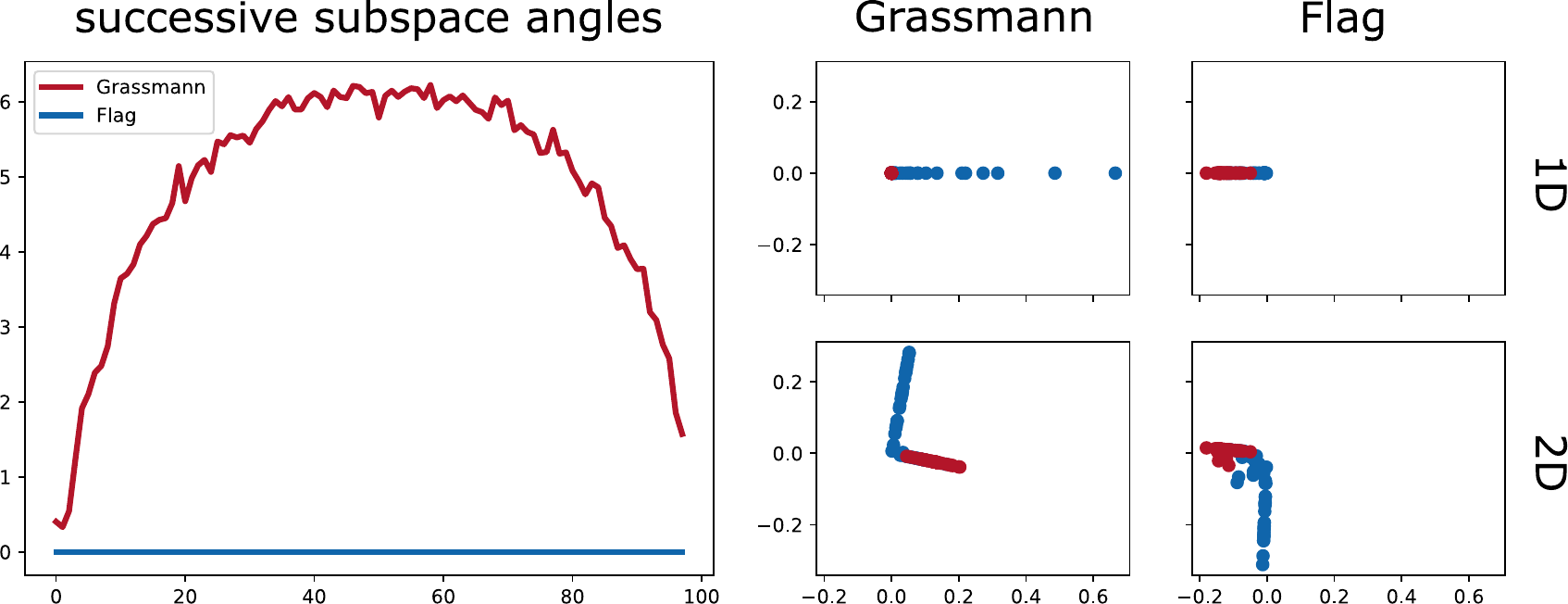}
    \caption{
    Illustration of the nestedness issue in sparse spectral clustering on the breast cancer dataset. On the left, for $q_k \in (1, 2, \dots, 99)$, we solve the Grassmannian optimization problem~\eqref{eq:SSC_Gr} on $\Gr(100, q_k)$ and plot the subspace angles $\Theta(\S_k^*, \S_{k+1}^*)$ as a function of $k$. {On the right, we plot the projection of the dataset onto the output 1D and 2D subspaces.} We compare those quantities to the ones obtained by solving the flag optimization problem~\eqref{eq:SSC_Fl}. We can see that the Grassmann-based method is highly non-nested while the flag-tricked one is not only nested but also seems to be coherent with the cluster structure: the first axis is for the red cluster while the second axis is for the blue cluster.
    }
    \label{fig:SC_nested}
\end{figure}
We can see that the Grassmann representations are not only non-nested, but also mix the two clusters in 2D, while the flag representations are nested and much better discriminate the clusters. This could be explained by the nestedness constraint of flag manifolds which imposes the 2D subspace to contain the 1D subspace.

Second, we consider the breast cancer dataset~\citep{wolberg_breast_1995}.
It contains $569$ samples---from which we extract a subset of $100$ samples for computational time---with $30$ numerical attributes and two classes. One has $n = 100$, $p=30$ and $C = 2$. 
Then we run the steepest descent algorithm~\ref{alg:GD} to solve the SSC problem~\eqref{eq:SSC_Fl}.
We choose the \textit{full signature} $q_{1:99} = (1, 2, \dots, 99)$ and compare the output flag to the individual subspaces output by running optimization on $\Gr(n, q_k)$ for $q_k \in q_{1:d}$.
We perform scatter plots in 1D and 2D and subspace error computations as a function of the $q_k$. The results are illustrated in Figure~\ref{fig:SC_nested}.
The subspace angle plot tells us that the Grassmann spectral clustering yields highly non-nested subspaces, while the flags are by nature nested. The scatter plots show that the Grassmann representations are totally inconsistent, while the flag representations are consistent with the cluster structures.

\subsubsection{Discussion on SSC Optimization}\label{subsubsec:SSC_discu}
Similarly to the previous examples, we can see that the steepest descent algorithm~\ref{alg:GD} might not be the most suited to solve the optimization problem~\eqref{eq:SSC_Fl}, notably due to the {$\ell^1$-penalization}. Contrary to robust subspace recovery, where we want to avoid the non-differentiable points (cf. Section~\ref{subsubsec:RSR_discu}), here we would ideally like to {exactly} attain those points, since they encode the desired sparsity of the Laplacian matrix. Therefore, we believe that sparsity-inducing optimization algorithms~\citep{bach_optimization_2012} on flag manifolds would be more adapted than the steepest descent we use.

\subsection{The Flag Trick for Other Machine Learning Problems}
Subspace learning finds many applications beyond robust subspace recovery, trace ratio and spectral clustering problems, as evoked in~Section~\ref{sec:intro}. The goal of this subsection is to provide a few more examples in brief, without experiments.

\subsubsection{Domain Adaptation}
In machine learning, it is often assumed that the training and test datasets follow the same distribution. However, some \textit{domain shift} issues---where training and test distributions are different---might arise, notably if the test data has been acquired from a different source (for instance a professional camera and a phone camera) or if the training data has been acquired a long time ago. \textit{Domain adaptation} is an area of machine learning that deals with domain shifts, usually by matching the training and test distributions---often referred to as \textit{source} and \textit{target} distributions---before fitting a classical model~\citep{farahani_brief_2021}. 
A large body of works (called ``subspace-based'') learn some intermediary subspaces between the source and target data, and perform the inference for the projected target data on these subspaces. The \textit{sampling geodesic flow}~\citep{gopalan_domain_2011} first performs a geodesic interpolation on Grassmannians between the source and target subspaces, then projects both datasets on (a discrete subset of) the interpolated subspaces, which results in a new representation of the data distributions, that can then be given as an input to a machine learning model. The higher the number of intermediary subspaces, the better the approximation, but the larger the dimension of the representation.
The celebrated \textit{geodesic flow kernel}~\citep{boqing_gong_geodesic_2012} circumvents this issue by integrating the projected data onto the continuum of interpolated subspaces. This yields an inner product between infinite-dimensional embeddings that can be computed explicitly and incorporated in a kernel method for learning. The \textit{domain invariant projection}~\citep{baktashmotlagh_unsupervised_2013} learns a \textit{domain-invariant} subspace that minimizes the maximum mean discrepancy (MMD)~\citep{gretton_kernel_2012} between the projected source $X_s \coloneqq \bmat{x_{s1} & \cdots & x_{s n_s}} \in \R^{p\times n_s}$ and target distributions $X_t \coloneqq \bmat{x_{t1} & \cdots & x_{t n_t}} \in \R^{p\times n_t}$:
$\operatorname{argmin}_{U \in \St(p, q)} \operatorname{MMD}^2(U\T X_{s}, U\T X_{t})$,
where 
$\operatorname{MMD} (X, Y) = \norm{\frac 1 n \sum_{i=1}^n \phi (x_i) - \frac 1 m \sum_{i=1}^m \phi (y_i)}_\mathcal{H}$.
Following~\citet[Eq.~(4)]{baktashmotlagh_unsupervised_2013}, we can rewrite the unsupervised domain adaptation problem as:
\begin{multline}\label{eq:DIP_Gr}
\argmin_{\S \in \Gr(p, q)} \, \left(
	\frac 1 {n_s^2} \sum_{i,j=1}^{n_s} \exp\lrp{-\frac{{z^s_{ij}}\T \Pi_\S z^s_{ij}}{2 \sigma^2}}
	+ \frac 1 {n_t^2} \sum_{i,j=1}^{n_t} \exp\lrp{-\frac{{z^t_{ij}}\T \Pi_\S z^t_{ij}}{2 \sigma^2}}\right.\\
	\left.- \frac 2 {n_s n_t} \sum_{i=1}^{n_s} \sum_{j=1}^{n_t} \exp\lrp{-\frac{{z^{st}_{ij}}\T \Pi_\S z^{st}_{ij}}{2 \sigma^2}}\right),
\end{multline}
where $z^s_{ij} \coloneqq x_{si} - x_{sj}$, $z^t_{ij} \coloneqq x_{ti} - x_{tj}$ and $z^{st}_{ij} \coloneqq x_{si} - x_{tj}$.
One can easily apply the flag trick~\eqref{eq:flag_problem} to this problem to make domain adaptation multilevel:
\begin{multline}\label{eq:DIP_Fl}
\argmin_{\S_{1:d} \in \Fl(p, q_{1:d})} \, \left(
	\frac 1 {n_s^2} \sum_{i,j=1}^{n_s} \exp\lrp{-\frac{{z^s_{ij}}\T \Pi_{\S_{1:d}} z^s_{ij}}{2 \sigma^2}}
	+ \frac 1 {n_t^2} \sum_{i,j=1}^{n_t} \exp\lrp{-\frac{{z^t_{ij}}\T \Pi_{\S_{1:d}} z^t_{ij}}{2 \sigma^2}}\right.\\
	\left.- \frac 2 {n_s n_t} \sum_{i=1}^{n_s} \sum_{j=1}^{n_t} \exp\lrp{-\frac{{z^{st}_{ij}}\T \Pi_{\S_{1:d}} z^{st}_{ij}}{2 \sigma^2}}\right).
\end{multline}
Some experiments similar to the ones of~\citet{baktashmotlagh_unsupervised_2013} can be performed. For instance, one can consider the benchmark visual object recognition dataset of~\citet{saenko_adapting_2010}, learn nested domain invariant projections, fit some support vector machines to the projected source samples at increasing dimensions, and then perform soft-voting ensembling by learning the optimal weights on the target data according to Equation~\eqref{eq:soft_voting}.

\subsubsection{Low-Rank Decomposition}
Many machine learning methods involve finding low-rank representations of a data matrix. 
This is the case of \textit{matrix completion}~\citep{candes_exact_2012} where one looks for a low-rank representation of an incomplete data matrix by minimizing the discrepancy with the observed entries, and which finds many applications including the well-known \href{https://en.wikipedia.org/wiki/Netflix_Prize}{Netflix problem}. Although its most-known formulation is as a convex relaxation, it can also be formulated as an optimization problem on Grassmann manifolds~\citep{keshavan_matrix_2010,boumal_rtrmc_2011} to avoid optimizing the nuclear norm in the full space which can be of high dimension. The intuition is that a low-dimensional point can be described by the subspace it belongs to and its coordinates within this subspace. More precisely, the SVD-based low-rank factorization $M = UW$, with $M \in \R^{p \times n}$, $U \in \St(p, q)$ and $W \in \R^{q \times n}$ is orthogonally-invariant---in the sense that for any $R\in\O(q)$, one has $(UR) (R\T W) = U W$. One could therefore apply the flag trick to such problems, with the intuition that we would try low-rank matrix decompositions at different dimensions. The application of the flag trick would however not be as straightforward as in the previous problems since the subspace-projection matrices $\Pi_\S \coloneqq U U\T$ do not appear explicitly, and since the coefficient matrix $W$ also depends on the dimension $q$.

Many other low-rank problems can be formulated as a Grassmannian optimization. \textit{Robust PCA}~\citep{candes_robust_2011} looks for a low rank + sparse corruption factorization of a data matrix. \textit{Subspace Tracking}~\citep{balzano_online_2010} incrementally updates a subspace from streaming and highly-incomplete observations via small steps on Grassmann manifolds.

\subsubsection{Other subspace learning problems}
Finally, many other general machine learning problems involve optimization on Grassmannians. For instance, \textit{linear dimensionality reduction}~\citep{cunningham_linear_2015} encompasses the already-discussed PCA and LDA, but also many other problems like \textit{multi-dimensional scaling}~\citep{torgerson_multidimensional_1952}, \textit{slow feature analysis}~\citep{wiskott_slow_2002} and \textit{locality preserving projections}~\citep{he_locality_2003}. All these machine learning problems involving optimization on Grassmannians can, likewise, be turned into multilevel problems via the flag trick.

\section{Discussion}\label{sec:discussion}

We introduced a general principle to make subspace learning problems multilevel and nested. The non-nestedness issue of Grassmannian-based methods was clearly demonstrated on a variety of machine learning problems, yielding inconsistent scatter plots and paradoxical learning curves. The flag trick was shown to both solve this nestedness issue and output filtrations of subspaces that meet the original goals of the problems. As a positive side-effect, the flag trick sometimes even output substantially-better subspaces than the original subspace problems, due to its multilevel nature which enables to aggregate information from several dimensions.
When combined with ensemble learning, the flag trick showed an overall improvement over the subspace-based predictors that work at a fixed dimension, and it raised some interesting dimension-blending perspectives that question the manifold hypothesis. 
The code for the experiments can be found at \url{https://github.com/tomszwagier/flag-trick}.

One major limitation of this paper is the lack of in-depth study of the specific interest of the flag trick to each subspace learning method. We indeed focus on general practical results---the nestedness of flag-based methods with respect to classical subspace methods and the interest of ensembling dimensions instead of considering each dimension individually---and not on some application-specific questions---like the robustness of the flag to outliers for robust subspace recovery, the quality of the embedding for trace ratio problems, and the quality of the clustering for spectral clustering problems. 
Therefore, this paper should rather be seen as an invitation for the different communities to explore the multilevel extensions of subspace learning, than as the unique solution to a long-standing issue in machine learning.
To that extent, we end this paper with a list of perspectives.

First, one should perform in-depth experiments to show the interest and the limitations of the flag trick for each specific subspace learning problem among---but not limited to---the list of examples developed in Section~\ref{sec:examples}. Throughout this paper, we focused on linear flags, but nonlinear extensions could be considered to improve expressivity. Beyond the kernel trick considered in~Section~\ref{subsec:kernel} for nonlinear graph embedding, a general methodology based on data mappings to linear spaces---such as the latent space of a deep generative model or the tangent space at the Fréchet mean of a manifold-valued dataset---could be simply implemented. More advanced methodologies relying on nested submanifolds, borrowed from the geometric statistics community (see~Section~\ref{sec:intro}), could be investigated as optimization problems on flag manifolds; typical examples are the principal nested spheres from~\citet{jung_analysis_2012}, which can be seen as intersections between spheres and (linear) hyperplanes.

Second, one should experiment with more complex and more efficient optimization algorithms on flag manifolds such as the ones developed in Appendix (IRLS in Appendix~\ref{app:RSR} and Newton in Appendix~\ref{app:TR}) or the ones proposed in~\citet{ye_optimization_2022}, \citet{nguyen_closed-form_2022} and \citet{zhu_practical_2024} and develop new ones that are specifically adapted to the properties of the objective function. We invite the interested researchers from the manifold optimization community to test their existing algorithms (and develop new ones) on our statistically-grounded optimization problems \eqref{eq:RSR_Fl}, \eqref{eq:TR_Fl}, \eqref{eq:SSC_Fl}, \eqref{eq:DIP_Fl} on flag manifolds.

Third, we derived a very general principle for transitioning from fixed-dimension subspace methods to increasing-dimension flag methods, but this principle could be revisited specifically for each problem. This includes the problems that are not specifically formulated as a Grassmannian optimization, as long as they somewhat involve low-dimensional subspaces, like in domain adaptation~\citep{gopalan_domain_2011, boqing_gong_geodesic_2012} and sparse subspace clustering~\citep{elhamifar_sparse_2013}.

Fourth, we leave the question of the choice of flag signature (i.e. the sequence of dimensions to try) open in this paper (see, notably, Remark~\ref{rk:choice_signature}).
A first step for PCA in the Gaussian case has been achieved with the principal subspace analysis of~\citet{szwagier_curse_2025}, where the dimensions are chosen based on the relative eigengaps of the covariance matrix. Other generative models could be designed specifically for each subspace learning problem, and the {recent} penalty of~\citet{szwagier_eigengap_2025} {on the spectral gaps} could be {added} to automatically select the flag signature.

\acks{The authors would like to thank Pierre-Antoine Absil for his insightful comments on the flag trick and its applications. This work was supported by the ERC grant \#786854 G-Statistics from the European Research Council under the European Union’s Horizon 2020 research and innovation program and by the French government through the 3IA Côte d’Azur Investments ANR-23-IACL-0001 managed by the National Research Agency.}

\newpage

\appendix
\section{Robust Subspace Recovery: Extensions and Proofs}\label{app:RSR}

\subsection{An IRLS Algorithm for Robust Subspace Recovery}
Iteratively reweighted least squares (IRLS) is a ubiquitous method to solve optimization problems involving $\ell^p$-norms. Motivated by the computation of the geometric median~\citep{weiszfeld_sur_1937}, IRLS is highly used to find robust maximum likelihood estimates of non-Gaussian probabilistic models (typically those containing outliers) and finds application in robust regression~\citep{huber_robust_1964}, sparse recovery~\citep{daubechies_iteratively_2010}, etc.

The recent fast median subspace (FMS) algorithm~\citep{lerman_fast_2018}, achieving state-of-the-art results in RSR uses an IRLS scheme to optimize the absolute deviation~\eqref{eq:RSR_Gr}.
The idea is to first rewrite the absolute deviation as 
\begin{equation}
	\sum_{i=1}^n \norm{x_i - \Pi_{\S} x_i}_2 = \sum_{i=1}^n w_i(\S) \norm{x_i - \Pi_{\S} x_i}_2^2,
\end{equation}
with $w_i(\S) = \frac{1}{\norm{x_i - \Pi_{\S} x_i}_2}$, and then successively compute the weights $w_i$ and update the subspace according to the weighted objective.
More precisely, the FMS algorithm creates a sequence of subspaces $\S^1, \dots, \S^m$ such that 
\begin{equation}\label{eq:IRLS_FMF}
    \S^{t+1} = \argmin_{\S \in \Gr(p, q)} \sum_{i=1}^n w_i(\S^t) \norm{x_i - \Pi_\S x_i}_2^2.
\end{equation}
This weighted least-squares problem enjoys a closed-form solution which relates to the eigenvalue decomposition of the weighted covariance matrix $\sum_{i=1}^n w_i(\S^t) x_i {x_i}\T$~\citep[Chapter~3.3]{vidal_generalized_2016}.

We wish to derive an IRLS algorithm for the flag-tricked version of the LAD problem~\eqref{eq:RSR_Fl}.
In order to stay close in mind to the recent work of \citet{peng_convergence_2023} who proved convergence of a general class of IRLS algorithms under some mild assumptions, we first rewrite~\eqref{eq:RSR_Fl} as
\begin{equation}~\label{eq:RSR_Fl_IRLS}
    \argmin_{\S_{1:d} \in\Fl(p, \qf)} \sum_{i=1}^n \rho(r(\S_{1:d}, x_i)),
\end{equation}
where $r(\S_{1:d}, x) = \norm{x - \Pi_{\Sf} x}_2$ is the \textit{residual} and $\rho(r) = |r|$ is the \textit{outlier-robust} loss function.
Following~\citet{peng_convergence_2023}, the IRLS scheme associated with~\eqref{eq:RSR_Fl_IRLS} is:
\begin{equation}
\begin{cases}
w_i^{t+1} = \rho'(r(\S_{1:d}^t, x_i)) /  r(\S_{1:d}^t, x_i) = 1 / \norm{x_i - \Pi_{\Sf} x_i}_2,\\
(\S_{1:d})^{t+1} = \argmin_{\S_{1:d} \in\Fl(p, \qf)} \sum_{i=1}^n w_i^{t+1} \norm{x_i - \Pi_{\Sf} x_i}_2^2.
\end{cases}
\end{equation}
We now show that the second step enjoys a closed-form solution.
\begin{proposition}[{Reweighted Nested PCA}]
The RLS problem
\begin{equation}
    \argmin_{\S_{1:d} \in\Fl(p, \qf)} \sum_{i=1}^n w_i \norm{x_i - \Pi_{\Sf} x_i}_2^2
\end{equation}
has a closed-form solution $\S_{1:d}^* \in\Fl(p, \qf)$, which is given by the eigenvalue decomposition of the weighted sample covariance matrix $S_w = \sum_{i=1}^n w_i x_i {x_i}\T = \sum_{j=1}^p \ell_j v_j {v_j}\T$, i.e.
\begin{equation}
    \S_k^* = \operatorname{Span}(v_1, \dots, v_{q_k}) \quad (k=1\twodots d).
\end{equation}
\end{proposition}
\begin{proof}
One has
\begin{equation}
	\sum_{i=1}^n w_i \norm{x_i - \Pi_{\Sf} x_i}_2^2 = \tr{(I - \Pi_{\Sf})^2 \lrp{\sum_{i=1}^n w_i x_i {x_i}\T}}.
\end{equation}
Therefore, we are exactly in the same case as in Theorem~\ref{thm:flag_trick}, if we replace $X X\T$ with the reweighted covariance matrix $\sum_{i=1}^n w_i x_i {x_i}\T$. This does not change the result, so we conclude with the end of the proof of Theorem~\ref{thm:flag_trick} (itself relying on~\citet{szwagier_curse_2025}).
\end{proof}
Hence, one gets an IRLS scheme for the LAD problem. 
One can modify the robust loss function $\rho(r) = |r|$ by a Huber-like loss function to avoid weight explosion. Indeed, one can show that the weight $w_i \coloneqq 1 / \norm{x_i - \Pi_{\Sf} x_i}_2$ goes to infinity when the first subspace $\S_1$ of the flag gets close to $x_i$ .
Therefore in practice, we take 
\begin{equation}
    \rho(r) = 
        \begin{cases}
            r^2 / (2 p \delta) & \text{if } |r| <= p\delta,\\
            r - p \delta / 2 & \text{if } |r| > p\delta.
        \end{cases}
\end{equation}
This yields
\begin{equation}
    w_i = 1 / \max\lrp{p\delta,  1 / \norm{x_i - \Pi_{\Sf} x_i}_2}.
\end{equation}
The final proposed scheme is given in Algorithm~\ref{alg:FMF}, named \textit{fast median flag} (FMF), in reference to the fast median subspace algorithm of~\citet{lerman_fast_2018}.
\begin{algorithm}
\caption{Fast Median Flag}\label{alg:FMF}
\begin{algorithmic}
\Require $X\in \R^{p\times n}$ (data); $q_{1:d} \coloneqq (q_1, \dots, q_{d})$ (signature); $t_{max}$ (max number of iterations); $\eta$ (convergence threshold); $\varepsilon$ (Huber-like saturation parameter) 
\Ensure
$U \in \St(p, q)$
\State $t \gets 0, \quad \Delta \gets \infty, \quad U^0 \gets \operatorname{SVD}(X, q)$
\While{$\Delta > \eta$ and $t < t_{max}$}
    \State $t \gets t+1$
    \State $r_i \gets \norm{x_i - \Pi_{\Sf} x_i}_2$
    \State $y_i \gets {x_i} / {\max(\sqrt{r_i}, \varepsilon)}$
    \State $U^t \gets \operatorname{SVD(Y, q)}$
    \State $\Delta \gets \sqrt{\sum_{k=1}^{d} \Theta(U^t_{q_k}, U^{t-1}_{q_k})^2}$
\EndWhile
\end{algorithmic}
\end{algorithm}
We can easily check that FMF is a direct generalization of FMS for Grassmannians (i.e. when $d=1$).

\begin{remark}[{Convergence Analysis}]
This is far beyond the scope of the paper, but we believe that the convergence result of~\citet[Theorem~1]{peng_convergence_2023} could be generalized to the FMF algorithm, due to the compactness of flag manifolds and the expression of the residual function $r$.
\end{remark}

\subsection{Proof of Proposition~\ref{prop:RSR}}
Let $\Sf \in \Fl(p, \qf)$ and $U_{1:d+1} \coloneqq \bmat{U_1 & U_2 & \cdots & U_d & U_{d+1}} \in \O(p)$ be an orthogonal representative of $\Sf$ (cf. Section~\ref{sec:flags}). One has:
\begin{align}
	\norm{x_i - \Pi_{\S_{1:d}} x_i}_2 &= \sqrt{{(x_i - \Pi_{\S_{1:d}} x_i)}\T (x_i - \Pi_{\S_{1:d}} x_i)},\\
	 &= \sqrt{{x_i}\T {(I_p - \Pi_{\S_{1:d}})}^2 x_i},\\
	 &= \sqrt{{x_i}\T {\lrp{I_p - \frac1d \sum_{k=1}^d\Pi_{\S_k}}}^2 x_i},\\
 	 &= \sqrt{\frac1{d^2} {x_i}\T {\lrp{\sum_{k=1}^d (I_p - \Pi_{\S_k})}}^2 x_i}.
\end{align}
By definition, one has $\Pi_{\S_k} = \sum_{l=1}^{k} U_l {U_l}\T$ and, in particular, $I_p = \sum_{l=1}^{d+1} U_l {U_l}\T$. Therefore, by re-injection and double summation inversion, one gets:
\begin{multline}
    \sum_{k=1}^{d}(I_p - \Pi_{\mathcal S_k}) = \sum_{k=1}^{d}\left(\sum_{l=1}^{d+1} U_l U_l^\top - \sum_{l=1}^k U_l U_l^\top\right) = \sum_{k=1}^{d}\sum_{l=k+1}^{d+1} U_l U_l^\top = \sum_{l=2}^{d+1}\sum_{k=1}^{l-1} U_l U_l^\top =\\ \sum_{l=2}^{d+1}(l-1) \, U_l U_l^\top = \sum_{l=1}^{d+1}(l-1) \, U_l U_l^\top.
\end{multline}
Re-injecting the previous quantity in the computation of $\norm{x_i - \Pi_{\S_{1:d}} x_i}_2$, one gets:
\begin{align}  
 	\norm{x_i - \Pi_{\S_{1:d}} x_i}_2 &= \sqrt{\frac1{d^2} {x_i}\T \lrp{\sum_{k=1}^{d+1} (k-1) U_k {U_k}\T}^2  x_i},\\
 	&= \sqrt{\frac1{d^2} {x_i}\T \lrp{\sum_{k=1}^{d+1} (k-1)^2 U_k {U_k}\T}  x_i},\\
 	&= \sqrt{\sum_{k=1}^{d+1} \lrp{\frac {k-1} {d}}^2 {x_i}\T U_k {U_k}\T  x_i},\\
  	\norm{x_i - \Pi_{\S_{1:d}} x_i}_2 &= \sqrt{\sum_{k=1}^{d+1} \lrp{\frac {k-1} {d}}^2 \norm{{U_k}\T x_i}_2^2}.
\end{align}
Hence, the flag-tricked robust subspace recovery problem~\eqref{eq:RSR_Fl} is equivalent to the following problem on the orthogonal group:
\begin{equation}\label{eq:RSR_Fl_O}
	\argmin_{U_{1:d+1} \in \O(p)} \sum_{i=1}^n \sqrt{\sum_{k=1}^{d+1} \lrp{\frac {k-1} {d}}^2 \norm{{U_k}\T x_i}_2^2}.
\end{equation}
This problem can be further simplified into an optimization problem on Stiefel manifolds, by noting that:
\begin{align}
    \norm{x_i - \Pi_{\S_{1:d}} x_i}_2 &= \sqrt{\sum_{k=1}^{d} \lrp{\lrp{\frac {k-1} {d}}^2 \norm{{U_k}\T x_i}_2^2} + \lrp{\frac {d+1-1} {d}}^2 \norm{{U_{d+1}}\T x_i}_2^2},\\
    &= \sqrt{\sum_{k=1}^{d} \lrp{\lrp{\frac {k-1} {d}}^2 \norm{{U_k}\T x_i}_2^2} + {x_i}\T {U_{d+1}} {U_{d+1}}\T {x_i}},\\
    &= \sqrt{\sum_{k=1}^{d} \lrp{\lrp{\frac {k-1} {d}}^2 \norm{{U_k}\T x_i}_2^2} + {x_i}\T (I_p - {U_{1:d}} {U_{1:d}}\T) {x_i}},\\
    &= \sqrt{\sum_{k=1}^{d} \lrp{\lrp{\frac {k-1} {d}}^2 \norm{{U_k}\T x_i}_2^2} + \norm{x_i}_2^2 - \norm{{U_{1:d}}\T x_i}_2^2},
\end{align}
which concludes the proof.

\section{Trace Ratio Problems: Extensions and Proofs}\label{app:TR}
In this section, we successively develop a Newton-Lanczos algorithm for the multilevel trace ratio problem~\eqref{eq:TR_Fl}, a kernelization of the same problem to handle nonlinear data, and we provide the proof of Proposition~\ref{prop:TR}.

\subsection{A Newton Method for Multilevel Trace Ratio Problems}
We wish to solve the following optimization problem (cf. \eqref{eq:TR_Fl}):
\begin{equation}\label{eq:TR_Fl_app}
	\argmax_{\S_{1:d} \in \Fl(p, q_{1:d})} \frac{\tr{\Pi_{\S_{1:d}} A}}{\tr{\Pi_{\S_{1:d}} B}}.
\end{equation}
In this subsection, we propose a Newton method---drawn from~\citet{ngo_trace_2012}, which is itself based on \citet{wang_trace_2007} and \citet{jia_trace_2009}---to solve this problem. The following proposition states that the multilevel trace ratio problem~\eqref{eq:TR_Fl_app} can be reformulated as the search for the root (or zero) of a certain function, which can be then classically solved via Newton's method.

\begin{proposition}[Multilevel Trace Ratio as a Newton Method]\label{prop:Newton}
    Let $M$ be a symmetric matrix, $(\ell_j(M), v_j(M))$ its $j$-th largest (eigenvalue, eigenvector)-pair and $\S_k^*(M) = \operatorname{Span}(v_1(M), \dots, v_{q_k}(M))$. 
    Let $f(\rho) \coloneqq \sum_{k=1}^d \sum_{j=1}^{q_k} \ell_j(A-\rho B)$. Then $f$ admits a unique root $\rho^*$, and one has:
    \begin{equation}
        \argmax_{\S_{1:d} \in \Fl(p, q_{1:d})} \frac{\tr{\Pi_{\S_{1:d}} A}}{\tr{\Pi_{\S_{1:d}} B}} = (\S_1^*(A-\rho^* B), \dots, \S_d^*(A-\rho^* B)) 
    \end{equation}
\end{proposition}
\begin{proof}
    Let $\phi\colon \Fl(p, q_{1:d}) \ni \Sf \mapsto {\tr{\sum_{k=1}^d \Pi_{\S_k} A}}/{\tr{\sum_{k=1}^d \Pi_{\S_k} B}}$ be the objective function we want to maximize. We can see that $\phi$ is well defined on $\Fl(p, q_{1:d})$ if and only if $\operatorname{rank}(B) > p - q_d$, which we assume in the following. 
    
    $\phi$ is continuous on $\Fl(p, \qf)$ which is a smooth compact manifold, therefore it admits a maximum $\rho^*$. As a consequence, for any ${\S_{1:d}} \in \Fl(p, \qf)$, one has $\phi(\Sf) \leq \rho^*$, which is equivalent to $~{\psi(\Sf)\coloneqq\sum_{k=1}^d \tr{\Pi_{\S_k} (A - \rho^* B)} \leq 0}$. Similarly to Theorem~\ref{thm:flag_trick}, one can show that the maximum of $\psi$, which is $0$, is attained in $\Sf^* \coloneqq \S_1^*(A-\rho^* B), \dots, \S_d^*(A-\rho^* B)$. Therefore, one has $\psi(\Sf^*) = \sum_{k=1}^d \operatorname{tr}(\Pi_{\S_k^*(A - \rho^* B)} (A - \rho^* B)) = f(\rho^*) = 0$.

    Let us now show that $\rho^*$ is the unique root of $f$. One can show that the derivative of $\sum_{j=1}^{q_k} \ell_j(A - \rho B)$ with respect to $\rho$ is $- \operatorname{tr}(\Pi_{\S_k^*(A - \rho B)} B)$~\citep{ngo_trace_2012}. Therefore, the derivative of $f$ is $f'(\rho) = - \sum_{k=1}^d \operatorname{tr}(\Pi_{\S_k^*(A - \rho B)} B)$, which is everywhere (strictly) negative since $B$ is positive semidefinite and $\operatorname{rank}(B) > p - q_d$. Therefore, f is a (strictly) decreasing function. Hence, $\rho^*$ is the unique root of $f$.
\end{proof}
Newton's method to find the root $\rho^*$ of $f$ is described by the following iteration: 
\begin{equation}
    \rho^{t+1} = \rho^t - \frac{f(\rho^t)}{f'(\rho^t)} = \frac{\sum_{k=1}^d \tr{\Pi_{\S_k(A - \rho^t B)} A}}{\sum_{k=1}^d \tr{\Pi_{\S_k(A - \rho^t B)} B}}. 
\end{equation}
The optimal flag $\Sf^*$ is given by the eigenvalue decomposition of $A - \rho^\infty B$, according to Proposition~\ref{prop:Newton}. The only remaining question is the one of initialization. The following proposition provides a possible answer.
\begin{proposition}[Bounds on Newton's Root]\label{prop:Newton_bounds}
    Let $\ell_j(A, B)$ denote the $j$-th largest eigenvalue of the matrix pencil $(A, B)$. If $B \succ 0$, then $\rho^* \in [\ell_{q_d}(A, B), \ell_{1}(A, B)]$.
\end{proposition}
\begin{proof}
    Let $Z\in \operatorname{GL}_p$ such that $Z\T B Z = I_p$ and $Z\T A Z =\diag{\ell_{1}(A, B), \dots, \ell_{p}(A, B)}$. Then, using Sylvester's law of inertia, one gets that the number of negative eigenvalues of $Z\T (A - \rho B) Z$ and of $\Lambda - \rho I_p$ are the same. Therefore, for $\rho = \ell_{1}(A, B)$, all the eigenvalues of $A - \rho B$ are nonpositive, so $f(\rho) = \sum_{k=1}^d \sum_{j=1}^{q_k} \ell_j(A - \rho B) \leq 0$. Similarly, for $\rho = \ell_{q_d}(A, B)$, the $q_d$ largest eigenvalues of $A - \rho B$ are nonnegative, so in particular for all $k \in \{1, \dots, d\}$, the $q_k$ largest eigenvalues of $A - \rho B$ are nonnegative, so in the end $f(\rho) \geq 0$.
\end{proof}
The bounds given by Proposition~\ref{prop:Newton_bounds} can be used to initialize Newton's method. Alternatively, one can initialize the algorithm with a random flag (see Remark~\ref{rem:init_gd}), similarly to~\citet[Algorithm~4.1]{ngo_trace_2012}, in order to circumvent the potentially costly computation of the generalized eigenvalues.
The final proposed scheme is given in Algorithm~\ref{alg:IATR}, called \textit{flag iterative trace ratio} in reference to the iterative trace ratio of~\citet{wang_trace_2007}.
\begin{algorithm}
\caption{Flag Iterative Trace Ratio}\label{alg:IATR}
\begin{algorithmic}
\Require $A, B \in \Sym_+(p)$; $q_{1:d} \coloneqq (q_1, \dots, q_{d})$ (signature); $\eta$ (convergence threshold)
\Ensure
$U \in \St(p, q)$ 
\State $t \gets 0, \Delta \gets \infty, \rho^0 \in \lrb{\ell_{q}(A, B), \ell_{1}(A, B)}$
\While{$\Delta > \eta$}
	\State $\bmat{u_1^t & \cdots & u_q^t} \gets \argmax_{U\in\St(p, q)} U\T (A - \rho^t B) U$ \Comment{compute leading $q$ eigenvectors}
    \State $\rho^t \gets \frac{\sum_{k=1}^d \tr{\bmat{u_1^t & \cdots & u_{q_k}^t}\T A \bmat{u_1^t & \cdots & u_{q_k}^t}}}{\sum_{k=1}^d \tr{\bmat{u_1^t & \cdots & u_{q_k}^t}\T B \bmat{u_1^t & \cdots & u_{q_k}^t}}}$ \Comment{update trace ratio}
    \State $\Delta \gets \sqrt{\sum_{k=1}^{d} \Theta(\bmat{u_1^t & \cdots & u_{q_k}^t}, \bmat{u_1^{t-1} & \cdots & u_{q_k}^{t-1}})^2}$ \Comment{evaluate convergence}
\EndWhile
\end{algorithmic}
\end{algorithm}

\subsection{Kernel Trick for Graph Embedding}\label{subsec:kernel}

The seminal paper of~\citet{yan_graph_2007} shows that many dimensionality reduction methods can be seen as a trace ratio problem known as \textit{graph embedding}. Let $X\coloneqq \bmat{x_1 & \dots & x_n}\in\R^{p\times n}$ be a $p$-dimensional dataset. We look for a $q$-dimensional linear embedding of $X$ that satisfies pairwise attraction and repulsion relationships. For instance, if we know the classes like in \textit{linear discriminant analysis}~\citep{fisher_use_1936}, we might want same-class points to be close and different-class points to be far in the lower-dimensional subspace. Writing, respectively, $S_{ii'}\in\R$ and $S^p_{ii'}\in\R$ the desired similarity and penalty between points $i$ and $i'$, this yields to the following optimization problem~\citep{yan_graph_2007,wang_trace_2007}:
\begin{equation}
    \argmin_{U\in\St(p, q)} \frac{\sum_{i \neq i'} \norm{U\T x_i - U\T x_{i'}}_2^2 S_{i{i'}}}{\sum_{i \neq {i'}} \norm{U\T x_i - U\T x_{i'}}_2^2 S^p_{i{i'}}}.
\end{equation}
Writing $L\in\Sym(n)$ and $L^p\in\Sym(n)$ the graph Laplacians associated with the adjacency matrices $S\in\Sym(n)$ and $S^p\in\Sym(n)$---i.e., $L = D - S$, with $D$ a diagonal matrix such that $D_{ii}=\sum_{{i'}\neq i} S_{i{i'}} \forall i$---this can be rewritten as the following trace ratio problem:
\begin{equation}
    \argmax_{U\in\St(p, q)} \frac{\tr{U\T X L^p X\T U}}{\tr{U\T X L X\T U}}.
\end{equation}
Finding a linear embedding satisfying such attraction and repulsion properties might be complicated in practice, as real data distributions are often nonlinear. A classical trick---known as the \textit{kernel trick}~\citep{hofmann_kernel_2008}---consists in nonlinearly mapping the data to an infinite-dimensional Hilbert space $(\mathcal{H}, \langle \cdot, \cdot\rangle_\mathcal{H})$---known as the reproducing kernel Hilbert space (RKHS)---via a so-called \textit{kernel} $\mathcal{K}\colon \R^p \times \R^p \to \R$, and then performing the linear graph embedding in the RKHS. This kernel must satisfy symmetry and positive-definiteness properties---a classical example is the radial basis function (RBF) kernel $\mathcal{K}(x, y) = e^{- {\|x - y\|}_2^2 / (2 \sigma^2)}$. The nonlinear map $\phi: \R^p \to \mathcal{H}$ is then chosen to be $\phi(x) = \mathcal{K}(x, \cdot)$, and the scalar product is $\langle \phi(x), \phi(y)\rangle_\mathcal{H} = \mathcal{K}(x, y)$.

Let us now write $K \coloneqq \mathcal{K}(x_i, x_j)_{i, j \in \{1, \dots, n\}} \in \Sym(n)$. Let $f = (f_1, \dots, f_q)\in\mathcal{H}^q$ be an orthonormal family, with $f_j = \sum_{i=1}^n v_{ij} \, \mathcal{K}(x_i, \cdot)$. The weights $V\in\R^{n\times q}$ must then satisfy the constraint $V\T K V = I_q$. Moreover, the coefficients of the orthogonal projection of $\phi(x_i)$ onto $\operatorname{Span}(f)$ are $\langle f_j, \phi(x_i)\rangle_\mathcal{H} = \sum_{i'=1}^n K_{ii'} v_{i'j}$.
Consequently, linear graph embedding in the RKHS boils down to:
\begin{equation}
    \argmin_{V\T K V = I_q} \frac{\sum_{i \neq i'} \norm{V\T K_i - V\T K_{i'}}_2^2 S_{ii'}}{\sum_{i \neq i'} \norm{V\T K_i - V\T K_{i'}}_2^2 S^p_{ii'}},
\end{equation}
which is equivalent to 
\begin{equation}
    \argmax_{V\T K V = I_q} \frac{\tr{V\T K L^p K V}}{\tr{V\T K L K V}}.
\end{equation}
Let $K \coloneqq Q \Lambda Q\T$ be an eigendecomposition of $K$, with $Q\in \O(n)$ and $\Lambda\in\operatorname{diag}(\R^n)$. 
Let us make the change of variable $U\coloneqq \Lambda^{\frac12} Q\T V$, and write $A = \Lambda^{1/2} Q\T L^p Q \Lambda^{1/2}$ and
$B = \Lambda^{1/2} Q\T L Q \Lambda^{1/2}$.
Then the trace ratio problem in the RKHS can be rewritten as:
\begin{equation}
    \argmax_{U\in\St(n, q)} \frac{\tr{U\T A U}}{\tr{U\T B U}}.
\end{equation}
One can now simply apply the flag trick to it, and get the following multilevel kernel graph embedding problem:
\begin{equation}
\argmax_{{\S_{1:d}} \in \Fl(n, q_{1:d})} \, \frac{\sum_{k=1}^{d} \tr{\Pi_{\S_k} A}}{\sum_{k=1}^{d} \tr{\Pi_{\S_k} B}}.
\end{equation}
Let $U^* \coloneqq \bmat{U^*_1 & \cdots & U^*_{d}}\in\St(n, q)$ be a Stiefel representative of the optimal flag $\S^*_{1:d} \in \Fl(n, q_{1:d})$. Then, one has for any $x\in \R^p$ the kernel graph embedding map $\Pi^*(x)_k = \sum_{i=1}^n U^*_{ik} K(x_i, x)$. In particular, $\Pi^*(X) = {U^*}\T K \in \R^{q \times n}$.

\subsection{Proof of Proposition~\ref{prop:TR}}
Let $\Sf \in \Fl(p, \qf)$ and $U_{1:d+1} \coloneqq \bmat{U_1 & U_2 & \cdots & U_d & U_{d+1}} \in \O(p)$ be an orthogonal representative of $\Sf$ (cf. Section~\ref{sec:flags}). One has:
\begin{align}
	\frac{\tr{\Pi_{\S_{1:d}} A}}{\tr{\Pi_{\S_{1:d}} B}} &= \frac{\tr{\frac1d \sum_{k=1}^d \Pi_{\S_k} A}}{\tr{\frac1d \sum_{l=1}^d \Pi_{\S_l} B}},\\
	&= \frac{\tr{\sum_{k=1}^{d+1} (d - (k-1)) U_k {U_k}\T A}}{\tr{\sum_{l=1}^{d+1} (d - (l-1)) U_l {U_l}\T B}},\\
	\frac{\tr{\Pi_{\S_{1:d}} A}}{\tr{\Pi_{\S_{1:d}} B}} &= \frac{\sum_{k=1}^{d+1} (d - (k-1)) \tr{U_k {U_k}\T A}}{\sum_{l=1}^{d+1} (d - (l-1)) \tr{U_l {U_l}\T B}},
\end{align}
which concludes the proof.

\section{Spectral Clustering: Extensions and Proofs}\label{app:SSC}

\subsection{Proof of Proposition~\ref{prop:SSC}}
Let $\Sf \in \Fl(n, \qf)$ and $U_{1:d+1} \coloneqq \bmat{U_1 & U_2 & \cdots & U_d & U_{d+1}} \in \O(p)$ be an orthogonal representative of $\Sf$ (cf. Section~\ref{sec:flags}). One has:
\begin{align}
	\langle \Pi_{\S_{1:d}}, L\rangle_F + \beta \norm{\Pi_{\S_{1:d}}}_1 &= \left\langle \frac1d\sum_{k=1}^d \Pi_{\S_k}, L\right\rangle_F + \beta \norm{\frac1d\sum_{k=1}^d \Pi_{\S_k}}_1,\\
	&= \frac1d \lrp{\left\langle \sum_{k=1}^{d+1} (d - (k-1)) U_k {U_k}\T, L\right\rangle_F + \beta \norm{\sum_{k=1}^{d+1} (d - (k-1)) U_k {U_k}\T }_1},\\
	&= \frac1d \lrp{\sum_{k=1}^{d+1} (d - (k-1)) \left\langle U_k {U_k}\T, L\right\rangle_F + \beta \norm{\sum_{k=1}^{d+1} (d - (k-1)) U_k {U_k}\T }_1},\\
	\langle \Pi_{\S_{1:d}}, L\rangle_F + \beta \norm{\Pi_{\S_{1:d}}}_1 &= \frac1d \lrp{\sum_{k=1}^{d+1} (d - (k-1)) \tr{{U_k}\T L U_k} + \beta \norm{\sum_{k=1}^{d+1} (d - (k-1)) U_k {U_k}\T }_1},
\end{align}
which concludes the proof.

\section{Empirical Running Times}\label{app:time}
In this section, we describe the empirical running times of the flag trick (and the subsequent optimization on flag manifolds via Algorithm~\ref{alg:GD}) for robust subspace recovery, trace ratio and sparse spectral clustering problems. We consider the three synthetic datasets described in Section~\ref{sec:examples}, with varying $n$ and $p$, and report the average running times over 10 independent experiments in Figure~\ref{fig:time}.
\begin{figure}
	\centering
    \includegraphics[width=\linewidth]{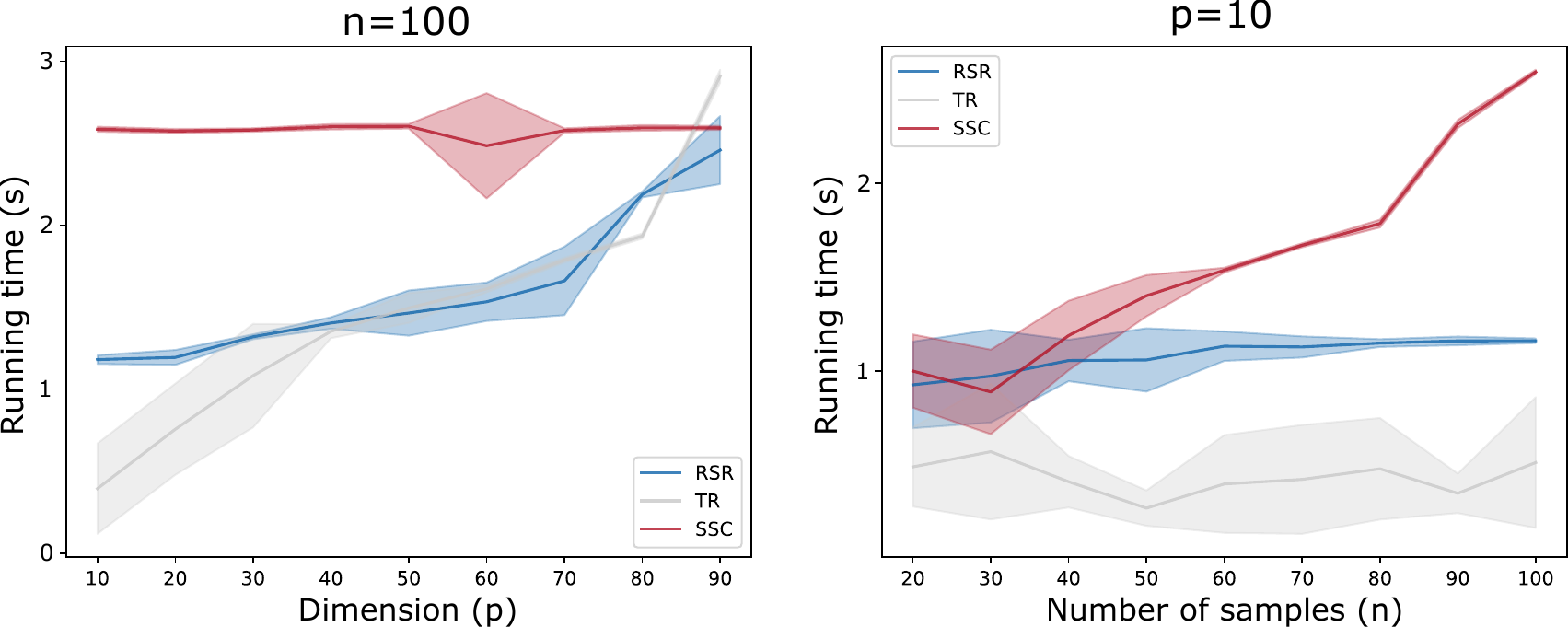}
    \caption{Empirical running times of the flag trick and the subsequent steepest descent on flag manifolds (Algorithm~\ref{alg:GD}) for three machine learning problems: robust subspace recovery (RSR), trace ratio (TR) and sparse spectral clustering (SSC). On the left plot, the number of samples is fixed to $n=100$ and the dimension varies from $p=10$ to $p=90$. On the right plot, the dimension is fixed to $p=10$ and the number of samples varies from $n=20$ to $n=100$. We set $\qf \coloneqq (1, 2, 5)$.}
	\label{fig:time}
\end{figure}

We see that the running times are in the order of a second for all methods. They increase approximately linearly with the dimension $p$ (as expected from the complexity analysis of the Riemannian gradient and the polar retraction in Remark~\ref{rem:complexity}). They do not increase much with the number of samples $n$, except for the sparse spectral clustering problem which relies on the $n\times n$ adjacency matrix.

\vskip 0.2in
\bibliography{sample}

@article{barker_partial_2003,
	title = {Partial least squares for discrimination},
	volume = {17},
	noissn = {0886-9383, 1099-128X},
	nodoi = {10.1002/cem.785},
	language = {en},
	number = {3},
	nourldate = {2023-09-13},
	journal = {Journal of Chemometrics},
	author = {Barker, M. and Rayens, W.},
	nomonth = mar,
	year = {2003},
	pages = {166--173},
}

@article{hardoon_canonical_2004,
	title = {Canonical {Correlation} {Analysis}: {An} {Overview} with {Application} to {Learning} {Methods}},
	volume = {16},
	noissn = {0899-7667},
	shorttitle = {Canonical {Correlation} {Analysis}},
	nodoi = {10.1162/0899766042321814},
	number = {12},
	nourldate = {2025-01-03},
	journal = {Neural Computation},
	author = {Hardoon, David R. and Szedmak, Sandor and Shawe-Taylor, John},
	nomonth = dec,
	year = {2004},
	pages = {2639--2664},
}

@article{nguyen_closed-form_2022,
	title = {Closed-form {Geodesics} and {Optimization} for {Riemannian} {Logarithms} of {Stiefel} and {Flag} {Manifolds}},
	volume = {194},
	noissn = {0022-3239},
	nodoi = {10.1007/s10957-022-02012-3},
	number = {1},
	nourldate = {2023-01-11},
	journal = {Journal of Optimization Theory and Applications},
	author = {Nguyen, Du},
	year = {2022},
	pages = {142--166},
}

@article{hofmann_kernel_2008,
	title = {Kernel methods in machine learning},
	volume = {36},
	noissn = {0090-5364, 2168-8966},
	nodoi = {10.1214/009053607000000677},
	number = {3},
	nourldate = {2025-02-07},
	journal = {The Annals of Statistics},
	author = {Hofmann, Thomas and Schölkopf, Bernhard and Smola, Alexander J.},
	nomonth = jun,
	year = {2008},
	pages = {1171--1220},
}

@article{diamond2016cvxpy,
  author  = {Steven Diamond and Stephen Boyd},
  title   = {{CVXPY}: {A} {P}ython-embedded modeling language for convex optimization},
  journal = {Journal of Machine Learning Research},
  year    = {2016},
  volume  = {17},
  number  = {83},
  pages   = {1--5},
}

@article{pedregosa_scikit-learn_2011,
	title = {Scikit-learn: {Machine} {Learning} in {Python}},
	volume = {12},
	noissn = {1533-7928},
	shorttitle = {Scikit-learn},
	number = {85},
	nourldate = {2024-10-15},
	journal = {Journal of Machine Learning Research},
	author = {Pedregosa, Fabian and Varoquaux, Gaël and Gramfort, Alexandre and Michel, Vincent and Thirion, Bertrand and Grisel, Olivier and Blondel, Mathieu and Prettenhofer, Peter and Weiss, Ron and Dubourg, Vincent and Vanderplas, Jake and Passos, Alexandre and Cournapeau, David and Brucher, Matthieu and Perrot, Matthieu and Duchesnay, Edouard},
	year = {2011},
	pages = {2825--2830},
}

@article{ye_schubert_2016,
	title = {Schubert {Varieties} and {Distances} between {Subspaces} of {Different} {Dimensions}},
	volume = {37},
	noissn = {0895-4798, 1095-7162},
	nourl = {http://epubs.siam.org/nodoi/10.1137/15M1054201},
	nodoi = {10.1137/15M1054201},
	language = {en},
	number = {3},
	nourldate = {2023-01-23},
	journal = {SIAM Journal on Matrix Analysis and Applications},
	author = {Ye, Ke and Lim, Lek-Heng},
	nomonth = jan,
	year = {2016},
	pages = {1176--1197},
}

@article{huckemann_intrinsic_2010,
	title = {Intrinsic {Shape} {Analysis}: {Geodesic} {Pca} for {Riemannian} {Manifolds} {Modulo} {Isometric} {Lie} {Group} {Actions}},
	volume = {20},
	noissn = {1017-0405},
	shorttitle = {Intrinsic {Shape} {Analysis}},
	nourl = {https://www.jstor.org/stable/24308976},
	number = {1},
	nourldate = {2025-01-17},
	journal = {Statistica Sinica},
	author = {Huckemann, Stephan and Hotz, Thomas and Munk, Axel},
	year = {2010},
}

@book{absil_optimization_2008,
	title = {Optimization {Algorithms} on {Matrix} {Manifolds}},
	noisbn = {978-0-691-13298-3},
	nourl = {https://www.degruyter.com/document/nodoi/10.1515/9781400830244/html},
	language = {en},
	nourldate = {2022-05-03},
	publisher = {Princeton University Press},
	author = {Absil, P.-A. and Mahony, R. and Sepulchre, R.},
	nomonth = apr,
	year = {2008},
}

@article{damon_backwards_2014,
	title = {Backwards {Principal} {Component} {Analysis} and {Principal} {Nested} {Relations}},
	volume = {50},
	nodoi = {10.1007/s10851-013-0463-2},
	journal = {Journal of Mathematical Imaging and Vision},
	author = {Damon, James and Marron, James Stephen},
	nomonth = sep,
	year = {2014},
}

@article{sommer_optimization_2014,
	title = {Optimization over {Geodesics} for {Exact} {Principal} {Geodesic} {Analysis}},
	volume = {40},
	nodoi = {10.1007/s10444-013-9308-1},
	journal = {Advances in Computational Mathematics},
	author = {Sommer, Stefan and Lauze, Francois and Nielsen, Mads},
	nomonth = aug,
	year = {2014},
}

@article{draper_flag_2014,
	title = {A flag representation for finite collections of subspaces of mixed dimensions},
	volume = {451},
	noissn = {0024-3795},
	nourl = {https://www.sciencedirect.com/science/article/pii/S0024379514001669},
	nodoi = {10.1016/j.laa.2014.03.022},
	language = {en},
	nourldate = {2022-06-13},
	journal = {Linear Algebra and its Applications},
	author = {Draper, Bruce and Kirby, Michael and Marks, Justin and Marrinan, Tim and Peterson, Chris},
	nomonth = jun,
	year = {2014},
	pages = {15--32},
}

@article{jung_analysis_2012,
	title = {Analysis of principal nested spheres},
	volume = {99},
	noissn = {0006-3444},
	nourl = {https://nodoi.org/10.1093/biomet/ass022},
	nodoi = {10.1093/biomet/ass022},
	number = {3},
	nourldate = {2022-06-29},
	journal = {Biometrika},
	author = {Jung, Sungkyu and Dryden, Ian L. and Marron, James Stephen},
	nomonth = sep,
	year = {2012},
	pages = {551--568},
}

@article{dryden_principal_2019,
	title = {Principal nested shape space analysis of molecular dynamics data},
	volume = {13},
	noissn = {1932-6157, 1941-7330},
	nourl = {https://projecteuclid.org/journals/annals-of-applied-statistics/volume-13/issue-4/Principal-nested-shape-space-analysis-of-molecular-dynamics-data/10.1214/19-AOAS1277.full},
	nodoi = {10.1214/19-AOAS1277},
	number = {4},
	nourldate = {2022-06-29},
	journal = {The Annals of Applied Statistics},
	author = {Dryden, Ian L. and Kim, Kwang-Rae and Laughton, C. A. and Le, Huiling},
	nomonth = dec,
	year = {2019},
	pages = {2213--2234},
}

@article{tipping_probabilistic_1999,
	title = {Probabilistic {Principal} {Component} {Analysis}},
	volume = {61},
	nourl = {https://www.jstor.org/stable/2680726},
	number = {3},
	nourldate = {2022-09-29},
	journal = {Journal of the Royal Statistical Society Series B: Statistical Methodology},
	author = {Tipping, Michael E. and Bishop, Christopher M.},
	year = {1999},
	pages = {611--622},
}

@article{fletcher_principal_2004,
	title = {Principal {Geodesic} {Analysis} for the {Study} of {Nonlinear} {Statistics} of {Shape}},
	volume = {23},
	noissn = {0278-0062},
	nourl = {http://ieeexplore.ieee.org/document/1318725/},
	nodoi = {10.1109/TMI.2004.831793},
	language = {en},
	number = {8},
	nourldate = {2022-10-10},
	journal = {IEEE Transactions on Medical Imaging},
	author = {Fletcher, P.T. and Lu, C. and Pizer, S.M. and Joshi, S.},
	nomonth = aug,
	year = {2004},
	pages = {995--1005},
}

@article{vidal_geometric_2024,
	title = {Geometric {Optimization} of {Restricted}-{Open} and {Complete} {Active} {Space} {Self}-{Consistent} {Field} {Wave} {Functions}},
	noissn = {1089-5639},
	nourl = {https://nodoi.org/10.1021/acs.jpca.4c03213},
	nodoi = {10.1021/acs.jpca.4c03213},
	nourldate = {2024-08-01},
	journal = {The Journal of Physical Chemistry A},
	author = {Vidal, Laurent and Nottoli, Tommaso and Lipparini, Filippo and Cancès, Eric},
	nomonth = jul,
	year = {2024},
}

@book{vidal_generalized_2016,
	series = {Interdisciplinary {Applied} {Mathematics}},
	title = {Generalized {Principal} {Component} {Analysis}},
	volume = {40},
	noisbn = {978-0-387-87810-2 978-0-387-87811-9},
	nourl = {http://link.springer.com/10.1007/978-0-387-87811-9},
	nourldate = {2022-11-04},
	publisher = {Springer},
	author = {Vidal, René and Ma, Yi and Sastry, S.S.},
	year = {2016},
	nodoi = {10.1007/978-0-387-87811-9},
}

@article{hyvarinen_emergence_2000,
	title = {Emergence of {Phase}- and {Shift}-{Invariant} {Features} by {Decomposition} of {Natural} {Images} into {Independent} {Feature} {Subspaces}},
	volume = {12},
	nourl = {https://nodoi.org/10.1162/089976600300015312},
	language = {en},
	number = {7},
	nourldate = {2022-12-21},
	journal = {Neural Computation},
	author = {Hyvärinen, Aapo and Hoyer, Patrik},
	nomonth = jul,
	year = {2000},
	pages = {1705--1720},
}

@inproceedings{mankovich_flag_2022,
	title = {The {Flag} {Median} and {FlagIRLS}},
	noisbn = {978-1-66546-946-3},
	nourl = {https://ieeexplore.ieee.org/document/9879219/},
	nodoi = {10.1109/CVPR52688.2022.01009},
	language = {en},
	nourldate = {2023-02-03},
	booktitle = {2022 {IEEE}/{CVF} {Conference} on {Computer} {Vision} and {Pattern} {Recognition}},
	author = {Mankovich, Nathan and King, Emily J. and Peterson, Chris and Kirby, Michael},
	nomonth = jun,
	year = {2022},
	pages = {10329--10337},
}

@inproceedings{gopalan_domain_2011,
	title = {Domain adaptation for object recognition: {An} unsupervised approach},
	noisbn = {978-1-4577-1102-2 978-1-4577-1101-5 978-1-4577-1100-8},
	shorttitle = {Domain adaptation for object recognition},
	nourl = {http://ieeexplore.ieee.org/document/6126344/},
	nodoi = {10.1109/ICCV.2011.6126344},
	language = {en},
	nourldate = {2023-02-03},
	booktitle = {2011 {International} {Conference} on {Computer} {Vision}},
	publisher = {IEEE},
	author = {Gopalan, Raghuraman and {Ruonan Li} and Chellappa, Rama},
	nomonth = nov,
	year = {2011},
	pages = {999--1006},
}

@article{ammar_geometry_1986,
	title = {The geometry of matrix eigenvalue methods},
	volume = {5},
	noissn = {1572-9036},
	nourl = {https://nodoi.org/10.1007/BF00047344},
	nodoi = {10.1007/BF00047344},
	language = {en},
	number = {3},
	nourldate = {2023-05-03},
	journal = {Acta Applicandae Mathematica},
	author = {Ammar, G. and Martin, C.},
	nomonth = mar,
	year = {1986},
	pages = {239--278},
}

@book{bellman_dynamic_1984,
	title = {Dynamic programming},
	noisbn = {978-0-691-07951-6},
	language = {en},
	publisher = {Princeton University Press},
	author = {Bellman, R.},
	year = {1957},
}

@article{fisher_use_1936,
	title = {The {Use} of {Multiple} {Measurements} in {Taxonomic} {Problems}},
	volume = {7},
	copyright = {1936 Blackwell Publishing Ltd/University College London},
	noissn = {2050-1439},
	nourl = {https://onlinelibrary.wiley.com/nodoi/abs/10.1111/j.1469-1809.1936.tb02137.x},
	nodoi = {10.1111/j.1469-1809.1936.tb02137.x},
	language = {en},
	number = {2},
	nourldate = {2024-11-08},
	journal = {Annals of Eugenics},
	author = {Fisher, Ronald Aylmer},
	year = {1936},
	pages = {179--188},
}

@article{belkin_laplacian_2003,
	title = {Laplacian {Eigenmaps} for {Dimensionality} {Reduction} and {Data} {Representation}},
	volume = {15},
	noissn = {0899-7667},
	nourl = {https://nodoi.org/10.1162/089976603321780317},
	nodoi = {10.1162/089976603321780317},
	number = {6},
	nourldate = {2023-12-22},
	journal = {Neural Computation},
	author = {Belkin, M. and Niyogi, P.},
	nomonth = jun,
	year = {2003},
	pages = {1373--1396},
}

@inproceedings{ng_spectral_2001,
	title = {On {Spectral} {Clustering}: {Analysis} and an algorithm},
	volume = {14},
	shorttitle = {On {Spectral} {Clustering}},
	nourl = {https://papers.nips.cc/paper_files/paper/2001/hash/801272ee79cfde7fa5960571fee36b9b-Abstract.html},
	nourldate = {2023-09-11},
	booktitle = {Advances in {Neural} {Information} {Processing} {Systems}},
	author = {Ng, Andrew and Jordan, Michael and Weiss, Yair},
	year = {2001},
}

@article{bendokat_grassmann_2024,
	title = {A {Grassmann} manifold handbook: basic geometry and computational aspects},
	volume = {50},
	noissn = {1572-9044},
	shorttitle = {A {Grassmann} manifold handbook},
	nourl = {https://nodoi.org/10.1007/s10444-023-10090-8},
	nodoi = {10.1007/s10444-023-10090-8},
	language = {en},
	number = {1},
	nourldate = {2024-07-18},
	journal = {Advances in Computational Mathematics},
	author = {Bendokat, T. and Zimmermann, R. and Absil, P.-A.},
	nomonth = jan,
	year = {2024},
}

@inproceedings{nishimori_riemannian_2006,
	title = {Riemannian {Optimization} {Method} on the {Flag} {Manifold} for {Independent} {Subspace} {Analysis}},
	nourl = {https://nodoi.org/10.1007/11679363_37},
	language = {en},
	booktitle = {Independent {Component} {Analysis} and {Blind} {Signal} {Separation}},
	author = {Nishimori, Yasunori and Akaho, Shotaro and Plumbley, Mark D.},
	year = {2006},
	pages = {295--302},
}

@inproceedings{cardoso_multidimensional_1998,
	title = {Multidimensional independent component analysis},
	volume = {4},
	noisbn = {978-0-7803-4428-0},
	nourl = {http://ieeexplore.ieee.org/document/681443/},
	nodoi = {10.1109/ICASSP.1998.681443},
	language = {en},
	nourldate = {2023-06-09},
	booktitle = {Proceedings of the 1998 {IEEE} {International} {Conference} on {Acoustics}, {Speech} and {Signal} {Processing}},
	author = {Cardoso, J.-F.},
	year = {1998},
	pages = {1941--1944},
}

@article{hyvarinen_independent_2000,
	title = {Independent component analysis: algorithms and applications},
	volume = {13},
	noissn = {0893-6080},
	shorttitle = {Independent component analysis},
	nourl = {https://www.sciencedirect.com/science/article/pii/S0893608000000265},
	nodoi = {10.1016/S0893-6080(00)00026-5},
	language = {en},
	number = {4},
	nourldate = {2023-06-09},
	journal = {Neural Networks},
	author = {Hyvärinen, Aapo and Oja, Erkki},
	nomonth = jun,
	year = {2000},
	pages = {411--430},
}

@article{pennec_barycentric_2018,
	title = {Barycentric {Subspace} {Analysis} on {Manifolds}},
	volume = {46},
	nourl = {https://www.jstor.org/stable/26542880},
	number = {6A},
	nourldate = {2023-07-07},
	journal = {The Annals of Statistics},
	author = {Pennec, X.},
	year = {2018},
	pages = {2711--2746},
}

@article{edelman_geometry_1998,
	title = {The {Geometry} of {Algorithms} with {Orthogonality} {Constraints}},
	volume = {20},
	nourl = {https://nodoi.org/10.1137/S0895479895290954},
	number = {2},
	nourldate = {2023-07-07},
	journal = {SIAM Journal on Matrix Analysis and Applications},
	author = {Edelman, Alan and Arias, Tomás A. and Smith, Steven T.},
	nomonth = jun,
	year = {1998},
	pages = {303--353},
}

@inproceedings{ma_flag_2021,
	title = {The {Flag} {Manifold} as a {Tool} for {Analyzing} and {Comparing} {Sets} of {Data} {Sets}},
	nourl = {https://nodoi.org/10.1109/ICCVW54120.2021.00465},
	language = {English},
	nourldate = {2023-07-07},
	booktitle = {2021 {IEEE}/{CVF} {International} {Conference} on {Computer} {Vision} {Workshops}},
	author = {Ma, Xiaofeng and Kirby, Michael and Peterson, Chris},
	nomonth = oct,
	year = {2021},
	pages = {4168--4177},
}

@article{candes_robust_2011,
	title = {Robust principal component analysis?},
	volume = {58},
	noissn = {0004-5411},
	nourl = {https://dl.acm.org/nodoi/10.1145/1970392.1970395},
	nodoi = {10.1145/1970392.1970395},
	number = {3},
	nourldate = {2023-08-08},
	journal = {Journal of the ACM},
	author = {Candès, E. J. and Li, X. and Ma, Y. and Wright, J.},
	year = {2011},
	pages = {11:1--11:37},
}

@inproceedings{mankovich_chordal_2023,
	title = {Chordal {Averaging} on {Flag} {Manifolds} and {Its} {Applications}},
	nourl = {https://openaccess.thecvf.com/content/ICCV2023/html/Mankovich_Chordal_Averaging_on_Flag_Manifolds_and_Its_Applications_ICCV_2023_paper.html},
	language = {en},
	nourldate = {2024-06-21},
	booktitle = {2023 {International} {Conference} on {Computer} {Vision}},
	author = {Mankovich, Nathan and Birdal, Tolga},
	year = {2023},
	pages = {3881--3890},
}

@article{huckemann_backward_2018,
	title = {Backward nested descriptors asymptotics with inference on stem cell differentiation},
	volume = {46},
	noissn = {0090-5364},
	nourl = {https://projecteuclid.org/journals/annals-of-statistics/volume-46/issue-5/Backward-nested-descriptors-asymptotics-with-inference-on-stem-cell-differentiation/10.1214/17-AOS1609.full},
	nodoi = {10.1214/17-AOS1609},
	language = {en},
	number = {5},
	nourldate = {2023-09-12},
	journal = {The Annals of Statistics},
	author = {Huckemann, Stephan F. and Eltzner, Benjamin},
	nomonth = oct,
	year = {2018},
}

@article{geladi_partial_1986,
	title = {Partial least-squares regression: a tutorial},
	volume = {185},
	noissn = {00032670},
	shorttitle = {Partial least-squares regression},
	nourl = {https://linkinghub.elsevier.com/retrieve/pii/0003267086800289},
	nodoi = {10.1016/0003-2670(86)80028-9},
	language = {en},
	nourldate = {2023-09-12},
	journal = {Analytica Chimica Acta},
	author = {Geladi, Paul and Kowalski, Bruce R.},
	year = {1986},
	pages = {1--17},
}

@article{kokiopoulou_enhanced_2009,
	title = {Enhanced graph-based dimensionality reduction with repulsion {Laplaceans}},
	volume = {42},
	noissn = {0031-3203},
	nourl = {https://www.sciencedirect.com/science/article/pii/S0031320309001460},
	nodoi = {10.1016/j.patcog.2009.04.005},
	number = {11},
	nourldate = {2024-12-19},
	journal = {Pattern Recognition},
	author = {Kokiopoulou, Effrosyni and Saad, Yousef},
	nomonth = nov,
	year = {2009},
	pages = {2392--2402},
}

@article{guo_generalized_2003,
	title = {A generalized {Foley}–{Sammon} transform based on generalized fisher discriminant criterion and its application to face recognition},
	volume = {24},
	noissn = {0167-8655},
	nourl = {https://www.sciencedirect.com/science/article/pii/S0167865502002076},
	nodoi = {10.1016/S0167-8655(02)00207-6},
	number = {1},
	nourldate = {2024-12-19},
	journal = {Pattern Recognition Letters},
	author = {Guo, Yue-Fei and Li, Shi-Jin and Yang, Jing-Yu and Shu, Ting-Ting and Wu, Li-De},
	nomonth = jan,
	year = {2003},
	pages = {147--158},
}

@book{jolliffe_principal_2002,
	series = {Springer {Series} in {Statistics}},
	title = {Principal {Component} {Analysis}},
	nourl = {http://link.springer.com/10.1007/b98835},
	language = {en},
	nourldate = {2023-09-20},
	publisher = {Springer-Verlag},
	author = {Jolliffe, Ian T.},
	year = {2002},
}

@inproceedings{minka_automatic_2000,
	title = {Automatic {Choice} of {Dimensionality} for {PCA}},
	volume = {13},
	nourl = {https://proceedings.neurips.cc/paper_files/paper/2000/file/7503cfacd12053d309b6bed5c89de212-Paper.pdf},
	booktitle = {Advances in {Neural} {Information} {Processing} {Systems}},
	author = {Minka, Thomas},
	year = {2000},
}

@inproceedings{oftadeh_eliminating_2020,
	title = {Eliminating the {Invariance} on the {Loss} {Landscape} of {Linear} {Autoencoders}},
	nourl = {https://proceedings.mlr.press/v119/oftadeh20a.html},
	language = {en},
	nourldate = {2024-12-20},
	booktitle = {Proceedings of the 37th {International} {Conference} on {Machine} {Learning}},
	publisher = {PMLR},
	author = {Oftadeh, Reza and Shen, Jiayi and Wang, Zhangyang and Shell, Dylan},
	nomonth = nov,
	year = {2020},
	pages = {7405--7413},
}

@article{zhu_practical_2024,
	title = {Practical gradient and conjugate gradient methods on flag manifolds},
	volume = {88},
	noissn = {1573-2894},
	nourl = {https://nodoi.org/10.1007/s10589-024-00568-6},
	nodoi = {10.1007/s10589-024-00568-6},
	language = {en},
	number = {2},
	nourldate = {2024-06-21},
	journal = {Computational Optimization and Applications},
	author = {Zhu, Xiaojing and Shen, Chungen},
	nomonth = jun,
	year = {2024},
	pages = {491--524},
}

@article{huber_projection_1985,
	title = {Projection {Pursuit}},
	volume = {13},
	noissn = {0090-5364, 2168-8966},
	nourl = {https://projecteuclid.org/journals/annals-of-statistics/volume-13/issue-2/Projection-Pursuit/10.1214/aos/1176349519.full},
	nodoi = {10.1214/aos/1176349519},
	number = {2},
	nourldate = {2023-10-04},
	journal = {The Annals of Statistics},
	author = {Huber, Peter J.},
	nomonth = jun,
	year = {1985},
	pages = {435--475},
}

@inproceedings{peng_convergence_2023,
	title = {On the {Convergence} of {IRLS} and {Its} {Variants} in {Outlier}-{Robust} {Estimation}},
	noisbn = {9798350301298},
	nourl = {https://ieeexplore.ieee.org/document/10203292/},
	nodoi = {10.1109/CVPR52729.2023.01708},
	language = {en},
	nourldate = {2023-10-11},
	booktitle = {2023 {IEEE}/{CVF} {Conference} on {Computer} {Vision} and {Pattern} {Recognition}},
	author = {Peng, Liangzu and Kümmerle, Christian and Vidal, René},
	nomonth = jun,
	year = {2023},
	pages = {17808--17818},
}

@article{lerman_fast_2018,
	title = {Fast, robust and non-convex subspace recovery},
	volume = {7},
	noissn = {2049-8764},
	nourl = {https://nodoi.org/10.1093/imaiai/iax012},
	nodoi = {10.1093/imaiai/iax012},
	number = {2},
	nourldate = {2023-10-12},
	journal = {Information and Inference: A Journal of the IMA},
	author = {Lerman, Gilad and Maunu, Tyler},
	nomonth = jun,
	year = {2018},
	pages = {277--336},
}

@article{lerman_overview_2018,
	title = {An {Overview} of {Robust} {Subspace} {Recovery}},
	volume = {106},
	noissn = {1558-2256},
	nourl = {https://ieeexplore.ieee.org/abstract/document/8425657},
	nodoi = {10.1109/JPROC.2018.2853141},
	number = {8},
	nourldate = {2023-10-13},
	journal = {Proceedings of the IEEE},
	author = {Lerman, Gilad and Maunu, Tyler},
	nomonth = aug,
	year = {2018},
	pages = {1380--1410},
}

@article{lerman_robust_2015,
	title = {Robust {Computation} of {Linear} {Models} by {Convex} {Relaxation}},
	volume = {15},
	noissn = {1615-3383},
	nourl = {https://nodoi.org/10.1007/s10208-014-9221-0},
	nodoi = {10.1007/s10208-014-9221-0},
	language = {en},
	number = {2},
	nourldate = {2023-10-16},
	journal = {Foundations of Computational Mathematics},
	author = {Lerman, Gilad and McCoy, Michael B. and Tropp, Joel A. and Zhang, Teng},
	nomonth = apr,
	year = {2015},
	pages = {363--410},
}

@article{maunu_well-tempered_2019,
	title = {A {Well}-{Tempered} {Landscape} for {Non}-convex {Robust} {Subspace} {Recovery}},
	volume = {20},
	noissn = {1533-7928},
	nourl = {http://jmlr.org/papers/v20/17-324.html},
	number = {37},
	nourldate = {2023-10-16},
	journal = {Journal of Machine Learning Research},
	author = {Maunu, Tyler and Zhang, Teng and Lerman, Gilad},
	year = {2019},
	pages = {1--59},
}

@article{hauberg_scalable_2016,
	title = {Scalable {Robust} {Principal} {Component} {Analysis} {Using} {Grassmann} {Averages}},
	volume = {38},
	noissn = {1939-3539},
	nourl = {https://ieeexplore.ieee.org/document/7364267},
	nodoi = {10.1109/TPAMI.2015.2511743},
	number = {11},
	nourldate = {2023-10-18},
	journal = {IEEE Transactions on Pattern Analysis and Machine Intelligence},
	author = {Hauberg, Søren and Feragen, Aasa and Enficiaud, Raffi and Black, Michael J.},
	nomonth = nov,
	year = {2016},
	pages = {2298--2311},
}

@article{arnold_modes_1972,
	title = {Modes and quasimodes},
	volume = {6},
	noissn = {1573-8485},
	nourl = {https://nodoi.org/10.1007/BF01077511},
	nodoi = {10.1007/BF01077511},
	language = {en},
	number = {2},
	nourldate = {2023-11-07},
	journal = {Functional Analysis and Its Applications},
	author = {Arnold, V. I.},
	nomonth = apr,
	year = {1972},
	pages = {94--101},
}

@inproceedings{yang_nested_2021,
	title = {Nested {Grassmanns} for {Dimensionality} {Reduction} with {Applications} to {Shape} {Analysis}},
	noisbn = {978-3-030-78191-0},
	nodoi = {10.1007/978-3-030-78191-0_11},
	language = {en},
	booktitle = {Information {Processing} in {Medical} {Imaging}},
	publisher = {Springer International Publishing},
	author = {Yang, Chun-Hao and Vemuri, Baba C.},
	year = {2021},
	pages = {136--149},
}

@article{candes_exact_2012,
	title = {Exact matrix completion via convex optimization},
	volume = {55},
	noissn = {0001-0782},
	nourl = {https://dl.acm.org/nodoi/10.1145/2184319.2184343},
	nodoi = {10.1145/2184319.2184343},
	number = {6},
	nourldate = {2023-11-28},
	journal = {Communications of the ACM},
	author = {Candès, Emmanuel and Recht, Benjamin},
	year = {2012},
	pages = {111--119},
}

@article{huber_robust_1964,
	title = {Robust {Estimation} of a {Location} {Parameter}},
	volume = {35},
	noissn = {0003-4851},
	nourl = {https://www.jstor.org/stable/2238020},
	number = {1},
	nourldate = {2023-11-16},
	journal = {The Annals of Mathematical Statistics},
	author = {Huber, Peter J.},
	year = {1964},
	pages = {73--101},
}

@article{daubechies_iteratively_2010,
	title = {Iteratively reweighted least squares minimization for sparse recovery},
	volume = {63},
	copyright = {Copyright © 2009 Wiley Periodicals, Inc.},
	noissn = {1097-0312},
	nourl = {https://onlinelibrary.wiley.com/nodoi/abs/10.1002/cpa.20303},
	nodoi = {10.1002/cpa.20303},
	language = {en},
	number = {1},
	nourldate = {2025-01-17},
	journal = {Communications on Pure and Applied Mathematics},
	author = {Daubechies, Ingrid and DeVore, Ronald and Fornasier, Massimo and Güntürk, C. Sinan},
	year = {2010},
	pages = {1--38},
}

@article{weiszfeld_sur_1937,
	title = {Sur le point pour lequel la somme des distances de $n$ points donnés est minimum},
	volume = {43},
	journal = {Tohoku Mathematical Journal, First Series},
	author = {Weiszfeld, Endre},
	year = {1937},
	pages = {355--386},
}

@inproceedings{baktashmotlagh_unsupervised_2013,
	title = {Unsupervised {Domain} {Adaptation} by {Domain} {Invariant} {Projection}},
	nourl = {https://www.cv-foundation.org/openaccess/content_iccv_2013/html/Baktashmotlagh_Unsupervised_Domain_Adaptation_2013_ICCV_paper.html},
	nourldate = {2023-12-07},
	booktitle = {Proceedings of the {IEEE} {International} {Conference} on {Computer} {Vision}},
	author = {Baktashmotlagh, M. and Harandi, M. T. and Lovell, B. C. and Salzmann, M.},
	year = {2013},
	pages = {769--776},
}

@inproceedings{fan_nested_2022,
	title = {Nested {Hyperbolic} {Spaces} for {Dimensionality} {Reduction} and {Hyperbolic} {NN} {Design}},
	noisbn = {978-1-66546-946-3},
	nourl = {https://ieeexplore.ieee.org/document/9880339/},
	nodoi = {10.1109/CVPR52688.2022.00045},
	language = {en},
	nourldate = {2023-12-07},
	booktitle = {2022 {IEEE}/{CVF} {Conference} on {Computer} {Vision} and {Pattern} {Recognition}},
	author = {Fan, Xiran and Yang, Chun-Hao and Vemuri, Baba C.},
	nomonth = jun,
	year = {2022},
	pages = {356--365},
}

@article{harandi_dimensionality_2018,
	title = {Dimensionality {Reduction} on {SPD} {Manifolds}: {The} {Emergence} of {Geometry}-{Aware} {Methods}},
	volume = {40},
	noissn = {0162-8828, 2160-9292},
	shorttitle = {Dimensionality {Reduction} on {SPD} {Manifolds}},
	nourl = {http://ieeexplore.ieee.org/document/7822908/},
	nodoi = {10.1109/TPAMI.2017.2655048},
	language = {en},
	number = {1},
	nourldate = {2023-12-07},
	journal = {IEEE Transactions on Pattern Analysis and Machine Intelligence},
	author = {Harandi, Mehrtash and Salzmann, Mathieu and Hartley, Richard},
	nomonth = jan,
	year = {2018},
	pages = {48--62},
}

@inproceedings{boqing_gong_geodesic_2012,
	title = {Geodesic flow kernel for unsupervised domain adaptation},
	noisbn = {978-1-4673-1228-8 978-1-4673-1226-4 978-1-4673-1227-1},
	nourl = {http://ieeexplore.ieee.org/document/6247911/},
	nodoi = {10.1109/CVPR.2012.6247911},
	language = {en},
	nourldate = {2023-12-08},
	booktitle = {2012 {IEEE} {Conference} on {Computer} {Vision} and {Pattern} {Recognition}},
	author = {{Boqing G.} and {Yuan S.} and {Fei S.} and Grauman, K.},
	nomonth = jun,
	year = {2012},
	pages = {2066--2073},
}

@inproceedings{balzano_online_2010,
	title = {Online identification and tracking of subspaces from highly incomplete information},
	nourl = {https://ieeexplore.ieee.org/abstract/document/5706976},
	nodoi = {10.1109/ALLERTON.2010.5706976},
	nourldate = {2023-12-13},
	booktitle = {2010 48th {Annual} {Allerton} {Conference} on {Communication}, {Control}, and {Computing}},
	author = {Balzano, L. and Nowak, R. and Recht, B.},
	nomonth = sep,
	year = {2010},
	pages = {704--711},
}

@inproceedings{boumal_rtrmc_2011,
	title = {{RTRMC}: {A} {Riemannian} trust-region method for low-rank matrix completion},
	volume = {24},
	shorttitle = {{RTRMC}},
	nourl = {https://papers.nips.cc/paper_files/paper/2011/hash/37bc2f75bf1bcfe8450a1a41c200364c-Abstract.html},
	nourldate = {2023-12-13},
	booktitle = {Advances in {Neural} {Information} {Processing} {Systems}},
	author = {Boumal, N. and Absil, P.-A.},
	year = {2011},
}

@article{keshavan_matrix_2010,
	title = {Matrix {Completion} {From} a {Few} {Entries}},
	volume = {56},
	noissn = {1557-9654},
	nourl = {https://ieeexplore.ieee.org/abstract/document/5466511},
	nodoi = {10.1109/TIT.2010.2046205},
	number = {6},
	nourldate = {2023-12-13},
	journal = {IEEE Transactions on Information Theory},
	author = {Keshavan, Raghunandan H. and Montanari, Andrea and Oh, Sewoong},
	nomonth = jun,
	year = {2010},
	pages = {2980--2998},
}

@article{gretton_kernel_2012,
	title = {A {Kernel} {Two}-{Sample} {Test}},
	volume = {13},
	noissn = {1533-7928},
	nourl = {http://jmlr.org/papers/v13/gretton12a.html},
	number = {25},
	nourldate = {2024-01-08},
	journal = {Journal of Machine Learning Research},
	author = {Gretton, Arthur and Borgwardt, Karsten M. and Rasch, Malte J. and Schölkopf, Bernhard and Smola, Alexander},
	year = {2012},
	pages = {723--773},
}

@inproceedings{mankovich_fun_2024,
	title = {Fun with {Flags}: {Robust} {Principal} {Directions} via {Flag} {Manifolds}},
	shorttitle = {Fun with {Flags}},
	nourl = {https://openaccess.thecvf.com/content/CVPR2024/html/Mankovich_Fun_with_Flags_Robust_Principal_Directions_via_Flag_Manifolds_CVPR_2024_paper.html},
	language = {en},
	nourldate = {2024-11-12},
	booktitle = {Proceedings of the {IEEE}/{CVF} {Conference} on {Computer} {Vision} and {Pattern} {Recognition}},
	author = {Mankovich, Nathan and Camps-Valls, Gustau and Birdal, Tolga},
	year = {2024},
	pages = {330--340},
}

@article{hotelling_relations_1936,
	title = {Relations {Between} {Two} {Sets} of {Variates}},
	volume = {28},
	noissn = {0006-3444},
	nourl = {https://www.jstor.org/stable/2333955},
	nodoi = {10.2307/2333955},
	number = {3/4},
	nourldate = {2024-01-16},
	journal = {Biometrika},
	author = {Hotelling, Harold},
	year = {1936},
	pages = {321--377},
}

@book{bouveyron_model-based_2019,
	edition = {1},
	title = {Model-{Based} {Clustering} and {Classification} for {Data} {Science}: {With} {Applications} in {R}},
	copyright = {https://www.cambridge.org/core/terms},
	noisbn = {978-1-108-64418-1 978-1-108-49420-5},
	shorttitle = {Model-{Based} {Clustering} and {Classification} for {Data} {Science}},
	nourl = {https://www.cambridge.org/core/product/identifier/9781108644181/type/book},
	language = {en},
	nourldate = {2024-05-13},
	publisher = {Cambridge University Press},
	author = {Bouveyron, C. and Celeux, G. and Murphy, T. Brendan and Raftery, Adrian E.},
	nomonth = jul,
	year = {2019},
	nodoi = {10.1017/9781108644181},
}

@inproceedings{szwagier_rethinking_2023,
	series = {{LNCS}},
	title = {Rethinking the {Riemannian} {Logarithm} on {Flag} {Manifolds} as an {Orthogonal} {Alignment} {Problem}},
	nourl = {https://nodoi.org/10.1007/978-3-031-38271-0_37},
	nodoi = {10.1007/978-3-031-38271-0_37},
	language = {en},
	booktitle = {Geometric {Science} of {Information}},
	publisher = {Springer},
	author = {Szwagier, Tom and Pennec, X.},
	year = {2023},
	pages = {375--383},
}

@article{ye_optimization_2022,
	title = {Optimization on flag manifolds},
	volume = {194},
	noissn = {1436-4646},
	nourl = {https://nodoi.org/10.1007/s10107-021-01640-3},
	nodoi = {10.1007/s10107-021-01640-3},
	language = {en},
	number = {1},
	nourldate = {2024-01-19},
	journal = {Mathematical Programming},
	author = {Ye, Ke and Wong, Ken Sze-Wai and Lim, Lek-Heng},
	nomonth = jul,
	year = {2022},
	pages = {621--660},
}

@article{ngo_trace_2012,
	title = {The {Trace} {Ratio} {Optimization} {Problem}},
	volume = {54},
	noissn = {0036-1445},
	nourl = {https://epubs.siam.org/nodoi/abs/10.1137/120864799},
	nodoi = {10.1137/120864799},
	number = {3},
	nourldate = {2024-02-08},
	journal = {SIAM Review},
	author = {Ngo, T. T. and Bellalij, M. and Saad, Yousef},
	nomonth = jan,
	year = {2012},
	pages = {545--569},
}

@article{yan_graph_2007,
	title = {Graph {Embedding} and {Extensions}: {A} {General} {Framework} for {Dimensionality} {Reduction}},
	volume = {29},
	noissn = {1939-3539},
	shorttitle = {Graph {Embedding} and {Extensions}},
	nourl = {https://ieeexplore.ieee.org/document/4016549},
	nodoi = {10.1109/TPAMI.2007.250598},
	number = {1},
	nourldate = {2024-02-16},
	journal = {IEEE Transactions on Pattern Analysis and Machine Intelligence},
	author = {Yan, Shuicheng and Xu, Dong and Zhang, Benyu and Zhang, Hong-jiang and Yang, Qiang and Lin, Stephen},
	nomonth = jan,
	year = {2007},
	pages = {40--51},
}

@inproceedings{wang_grassmannian_2017,
	title = {Grassmannian {Manifold} {Optimization} {Assisted} {Sparse} {Spectral} {Clustering}},
	nourl = {https://ieeexplore.ieee.org/document/8099818},
	nodoi = {10.1109/CVPR.2017.335},
	nourldate = {2024-02-21},
	booktitle = {2017 {IEEE} {Conference} on {Computer} {Vision} and {Pattern} {Recognition}},
	author = {Wang, Qiong and Gao, Junbin and Li, Hong},
	nomonth = jul,
	year = {2017},
	pages = {3145--3153},
}

@article{mccoy_two_2011,
	title = {Two proposals for robust {PCA} using semidefinite programming},
	volume = {5},
	noissn = {1935-7524, 1935-7524},
	nourl = {https://projecteuclid.org/journals/electronic-journal-of-statistics/volume-5/issue-none/Two-proposals-for-robust-PCA-using-semidefinite-programming/10.1214/11-EJS636.full},
	nodoi = {10.1214/11-EJS636},
	nourldate = {2023-09-15},
	journal = {Electronic Journal of Statistics},
	author = {McCoy, Michael and Tropp, Joel A.},
	nomonth = jan,
	year = {2011},
	pages = {1123--1160},
}

@article{lim_numerical_2019,
	title = {Numerical {Algorithms} on the {Affine} {Grassmannian}},
	volume = {40},
	noissn = {0895-4798},
	nourl = {https://epubs.siam.org/nodoi/abs/10.1137/18M1169321},
	nodoi = {10.1137/18M1169321},
	number = {2},
	nourldate = {2024-11-20},
	journal = {SIAM Journal on Matrix Analysis and Applications},
	author = {Lim, Lek-Heng and Sze-Wai Wong, Ken and Ye, Ke},
	nomonth = jan,
	year = {2019},
	pages = {371--393},
}

@techreport{perrone_when_1992,
	title = {When networks disagree: {Ensemble} methods for hybrid neural networks},
	institution = {Brown University Providence, Institute for Brain and Neural Systems},
	author = {Perrone, Michael P and Cooper, Leon N},
	year = {1992},
}

@article{bouveyron_intrinsic_2011,
	title = {Intrinsic dimension estimation by maximum likelihood in isotropic probabilistic {PCA}},
	volume = {32},
	nourl = {https://nodoi.org/10.1016/j.patrec.2011.07.017},
	language = {en},
	number = {14},
	nourldate = {2022-09-20},
	journal = {Pattern Recognition Letters},
	author = {Bouveyron, C. and Celeux, G. and Girard, S.},
	nomonth = oct,
	year = {2011},
	pages = {1706--1713},
}

@article{bouveyron_hddc_2007,
	title = {High-dimensional data clustering},
	volume = {52},
	noissn = {0167-9473},
	nourl = {https://www.sciencedirect.com/science/article/pii/S0167947307000692},
	nodoi = {10.1016/j.csda.2007.02.009},
	language = {en},
	number = {1},
	nourldate = {2022-11-14},
	journal = {Computational Statistics \& Data Analysis},
	author = {Bouveyron, C. and Girard, S. and Schmid, C.},
	nomonth = sep,
	year = {2007},
	pages = {502--519},
}

@article{tsakiris_dual_2018,
	title = {Dual principal component pursuit},
	volume = {19},
	noissn = {1532-4435},
	number = {1},
	journal = {The Journal of Machine Learning Research},
	author = {Tsakiris, Manolis C. and Vidal, René},
	nomonth = jan,
	year = {2018},
	pages = {684--732},
}

@inproceedings{ding_r1-pca_2006,
	title = {R1-{PCA}: rotational invariant {L1}-norm principal component analysis for robust subspace factorization},
	noisbn = {978-1-59593-383-6},
	shorttitle = {R1-{PCA}},
	nourl = {https://nodoi.org/10.1145/1143844.1143880},
	nodoi = {10.1145/1143844.1143880},
	nourldate = {2022-06-30},
	booktitle = {Proceedings of the 23rd international conference on {Machine} learning},
	author = {Ding, Chris and Zhou, Ding and He, Xiaofeng and Zha, Hongyuan},
	year = {2006},
	pages = {281--288},
}

@article{xu_robust_2012,
	title = {Robust {PCA} via {Outlier} {Pursuit}},
	volume = {58},
	noissn = {1557-9654},
	nourl = {https://ieeexplore.ieee.org/document/6126034},
	nodoi = {10.1109/TIT.2011.2173156},
	number = {5},
	nourldate = {2023-10-06},
	journal = {IEEE Transactions on Information Theory},
	author = {Xu, Huan and Caramanis, Constantine and Sanghavi, Sujay},
	nomonth = may,
	year = {2012},
	pages = {3047--3064},
}

@article{zhang_novel_2014,
	title = {A novel {M}-estimator for robust {PCA}},
	volume = {15},
	noissn = {1532-4435},
	number = {1},
	journal = {The Journal of Machine Learning Research},
	author = {Zhang, Teng and Lerman, Gilad},
	nomonth = jan,
	year = {2014},
	pages = {749--808},
}

@inproceedings{wang_trace_2007,
	title = {Trace {Ratio} vs. {Ratio} {Trace} for {Dimensionality} {Reduction}},
	nourl = {https://ieeexplore.ieee.org/abstract/document/4270008},
	nodoi = {10.1109/CVPR.2007.382983},
	nourldate = {2024-01-09},
	booktitle = {2007 {IEEE} {Conference} on {Computer} {Vision} and {Pattern} {Recognition}},
	author = {Wang, Huan and Yan, Shuicheng and Xu, Dong and Tang, Xiaoou and Huang, Thomas},
	nomonth = jun,
	year = {2007},
	pages = {1--8},
}

@article{jia_trace_2009,
	title = {Trace {Ratio} {Problem} {Revisited}},
	volume = {20},
	noissn = {1941-0093},
	nourl = {https://ieeexplore.ieee.org/document/4801520},
	nodoi = {10.1109/TNN.2009.2015760},
	number = {4},
	nourldate = {2024-01-09},
	journal = {IEEE Transactions on Neural Networks},
	author = {Jia, Y. and Nie, F. and Zhang, C.},
	nomonth = apr,
	year = {2009},
	pages = {729--735},
}

@article{chang_using_1983,
	title = {On {Using} {Principal} {Components} {Before} {Separating} a {Mixture} of {Two} {Multivariate} {Normal} {Distributions}},
	volume = {32},
	noissn = {0035-9254},
	nourl = {https://www.jstor.org/stable/2347949},
	nodoi = {10.2307/2347949},
	number = {3},
	nourldate = {2024-05-22},
	journal = {Journal of the Royal Statistical Society. Series C (Applied Statistics)},
	author = {Chang, W.-C.},
	year = {1983},
	pages = {267--275},
}

@article{wang_trace_2023,
	title = {Trace ratio optimization with an application to multi-view learning},
	volume = {201},
	noissn = {1436-4646},
	nourl = {https://nodoi.org/10.1007/s10107-022-01900-w},
	nodoi = {10.1007/s10107-022-01900-w},
	language = {en},
	number = {1},
	nourldate = {2024-03-06},
	journal = {Mathematical Programming},
	author = {Wang, Li and Zhang, Lei-Hong and Li, Ren-Cang},
	nomonth = sep,
	year = {2023},
	pages = {97--131},
}

@article{cunningham_linear_2015,
	title = {Linear {Dimensionality} {Reduction}: {Survey}, {Insights}, and {Generalizations}},
	volume = {16},
	noissn = {1533-7928},
	shorttitle = {Linear {Dimensionality} {Reduction}},
	nourl = {http://jmlr.org/papers/v16/cunningham15a.html},
	number = {89},
	nourldate = {2024-11-06},
	journal = {Journal of Machine Learning Research},
	author = {Cunningham, J. P. and Ghahramani, Z.},
	year = {2015},
	pages = {2859--2900},
}

@article{elhamifar_sparse_2013,
	title = {Sparse {Subspace} {Clustering}: {Algorithm}, {Theory}, and {Applications}},
	volume = {35},
	noissn = {1939-3539},
	shorttitle = {Sparse {Subspace} {Clustering}},
	nourl = {https://ieeexplore.ieee.org/abstract/document/6482137},
	nodoi = {10.1109/TPAMI.2013.57},
	number = {11},
	nourldate = {2023-10-13},
	journal = {IEEE Transactions on Pattern Analysis and Machine Intelligence},
	author = {Elhamifar, Ehsan and Vidal, René},
	nomonth = nov,
	year = {2013},
	pages = {2765--2781},
}

@article{lu_convex_2016,
	title = {Convex {Sparse} {Spectral} {Clustering}: {Single}-{View} to {Multi}-{View}},
	volume = {25},
	noissn = {1941-0042},
	shorttitle = {Convex {Sparse} {Spectral} {Clustering}},
	nourl = {https://ieeexplore.ieee.org/abstract/document/7451227?casa_token=2sr9zDNLn9oAAAAA:D7BN2RJzxZNyDmOOytYanRlpx651zpYyBleEyRdFjKVvkdpoi3zd5kTDKAQfvl6k1r_wG3ZxdEk},
	nodoi = {10.1109/TIP.2016.2553459},
	number = {6},
	nourldate = {2024-02-21},
	journal = {IEEE Transactions on Image Processing},
	author = {Lu, Canyi and Yan, Shuicheng and Lin, Zhouchen},
	nomonth = jun,
	year = {2016},
	pages = {2833--2843},
}

@inproceedings{chen_local_2005,
	title = {Local discriminant embedding and its variants},
	volume = {2},
	nourl = {https://ieeexplore.ieee.org/document/1467531},
	nodoi = {10.1109/CVPR.2005.216},
	nourldate = {2024-02-26},
	booktitle = {2005 {IEEE} {Computer} {Society} {Conference} on {Computer} {Vision} and {Pattern} {Recognition}},
	author = {Chen, H.-T. and Chang, H.-W. and Liu, T.-L.},
	nomonth = jun,
	year = {2005},
	pages = {846--853},
}

@inproceedings{farahani_brief_2021,
	title = {A {Brief} {Review} of {Domain} {Adaptation}},
	noisbn = {978-3-030-71704-9},
	nodoi = {10.1007/978-3-030-71704-9_65},
	language = {en},
	booktitle = {Advances in {Data} {Science} and {Information} {Engineering}},
	publisher = {Springer International Publishing},
	author = {Farahani, Abolfazl and Voghoei, Sahar and Rasheed, Khaled and Arabnia, Hamid R.},
	year = {2021},
	pages = {877--894},
}

@inproceedings{saenko_adapting_2010,
	title = {Adapting {Visual} {Category} {Models} to {New} {Domains}},
	noisbn = {978-3-642-15561-1},
	nodoi = {10.1007/978-3-642-15561-1_16},
	language = {en},
	booktitle = {Computer {Vision} – {ECCV} 2010},
	publisher = {Springer},
	author = {Saenko, Kate and Kulis, Brian and Fritz, Mario and Darrell, Trevor},
	year = {2010},
	pages = {213--226},
}

@misc{alpaydin_optical_1998,
	title = {Optical {Recognition} of {Handwritten} {Digits}},
	nourl = {https://archive.ics.uci.edu/dataset/80},
	nodoi = {10.24432/C50P49},
	nourldate = {2024-01-11},
	publisher = {UCI Machine Learning Repository},
	author = {Alpaydin, E. and Kaynak, C.},
	year = {1998},
	note = {doi: 10.24432/C50P49}
}

@article{townsend_pymanopt_2016,
	title = {Pymanopt: {A} {Python} {Toolbox} for {Optimization} on {Manifolds} using {Automatic} {Differentiation}},
	volume = {17},
	noissn = {1533-7928},
	shorttitle = {Pymanopt},
	nourl = {http://jmlr.org/papers/v17/16-177.html},
	number = {137},
	nourldate = {2022-05-04},
	journal = {Journal of Machine Learning Research},
	author = {Townsend, James and Koep, Niklas and Weichwald, Sebastian},
	year = {2016},
	pages = {1--5},
}

@article{boumal_manopt_2014,
	title = {Manopt, a matlab toolbox for optimization on manifolds},
	volume = {15},
	noissn = {1532-4435},
	number = {1},
	journal = {Journal of Machine Learning Research},
	author = {Boumal, N. and Mishra, B. and Absil, P.-A. and Sepulchre, R.},
	nomonth = jan,
	year = {2014},
	pages = {1455--1459},
}

@inproceedings{hamm_grassmann_2008,
	title = {Grassmann discriminant analysis: a unifying view on subspace-based learning},
	noisbn = {978-1-60558-205-4},
	shorttitle = {Grassmann discriminant analysis},
	nourl = {https://dl.acm.org/nodoi/10.1145/1390156.1390204},
	nodoi = {10.1145/1390156.1390204},
	nourldate = {2023-08-17},
	booktitle = {Proceedings of the 25th International Conference on {Machine} Learning},
	author = {Hamm, Jihun and Lee, Daniel D.},
	year = {2008},
	pages = {376--383},
}

@article{launay_mechanical_2021,
	title = {Mechanical assessment of defects in welded joints: morphological classification and data augmentation},
	volume = {11},
	noissn = {2190-5983},
	shorttitle = {Mechanical assessment of defects in welded joints},
	nourl = {https://nodoi.org/10.1186/s13362-021-00114-7},
	nodoi = {10.1186/s13362-021-00114-7},
	number = {1},
	nourldate = {2022-08-26},
	journal = {Journal of Mathematics in Industry},
	author = {Launay, Hugo and Willot, François and Ryckelynck, David and Besson, Jacques},
	nomonth = oct,
	year = {2021},
}

@article{belhumeur_eigenfaces_1997,
	title = {Eigenfaces vs. {Fisherfaces}: recognition using class specific linear projection},
	volume = {19},
	noissn = {1939-3539},
	shorttitle = {Eigenfaces vs. {Fisherfaces}},
	nourl = {https://ieeexplore.ieee.org/document/598228},
	nodoi = {10.1109/34.598228},
	number = {7},
	nourldate = {2024-02-26},
	journal = {IEEE Transactions on Pattern Analysis and Machine Intelligence},
	author = {Belhumeur, P.N. and Hespanha, J.P. and Kriegman, D.J.},
	nomonth = jul,
	year = {1997},
	pages = {711--720},
}

@misc{wolberg_breast_1995,
	title = {Breast {Cancer} {Wisconsin} ({Diagnostic})},
	nourl = {https://nodoi.org/10.24432/C5DW2B},
	nodoi = {10.24432/C5DW2B},
	publisher = {UCI Machine Learning Repository},
	author = {Wolberg, W. H. and Mangasarian, O. L. and Street, W. N.},
	year = {1995},
	note = {doi: 10.24432/C5DW2B}
}

@article{torgerson_multidimensional_1952,
	title = {Multidimensional scaling: {I}. {Theory} and method},
	volume = {17},
	noissn = {1860-0980},
	shorttitle = {Multidimensional scaling},
	nourl = {https://nodoi.org/10.1007/BF02288916},
	nodoi = {10.1007/BF02288916},
	language = {en},
	number = {4},
	nourldate = {2025-01-06},
	journal = {Psychometrika},
	author = {Torgerson, Warren S.},
	nomonth = dec,
	year = {1952},
	pages = {401--419},
}

@article{wiskott_slow_2002,
	title = {Slow {Feature} {Analysis}: {Unsupervised} {Learning} of {Invariances}},
	volume = {14},
	noissn = {0899-7667},
	shorttitle = {Slow {Feature} {Analysis}},
	nourl = {https://nodoi.org/10.1162/089976602317318938},
	nodoi = {10.1162/089976602317318938},
	number = {4},
	nourldate = {2025-01-06},
	journal = {Neural Computation},
	author = {Wiskott, Laurenz and Sejnowski, Terrence J.},
	nomonth = apr,
	year = {2002},
	pages = {715--770},
}

@inproceedings{he_locality_2003,
	title = {Locality {Preserving} {Projections}},
	volume = {16},
	nourl = {https://papers.nips.cc/paper_files/paper/2003/hash/d69116f8b0140cdeb1f99a4d5096ffe4-Abstract.html},
	nourldate = {2025-01-06},
	booktitle = {Advances in {Neural} {Information} {Processing} {Systems}},
	author = {He, Xiaofei and Niyogi, Partha},
	year = {2003},
}

@book{boumal_introduction_2023,
	title = {An {Introduction} to {Optimization} on {Smooth} {Manifolds}},
	noisbn = {978-1-00-916616-4 978-1-00-916617-1 978-1-00-916615-7},
	nourl = {https://www.cambridge.org/core/product/identifier/9781009166164/type/book},
	language = {en},
	nourldate = {2023-06-08},
	publisher = {Cambridge University Press},
	author = {Boumal, N.},
	nomonth = mar,
	year = {2023},
	nodoi = {10.1017/9781009166164},
}

@book{chikuse_statistics_2003,
	series = {{LNS}},
	title = {Statistics on {Special} {Manifolds}},
	volume = {174},
	noisbn = {978-0-387-00160-9 978-0-387-21540-2},
	nourl = {http://link.springer.com/10.1007/978-0-387-21540-2},
	nourldate = {2022-04-19},
	publisher = {Springer},
	author = {Chikuse, Y.},
	year = {2003},
	nodoi = {10.1007/978-0-387-21540-2},
}

@article{pearson_lines_1901,
	title = {On lines and planes of closest fit to systems of points in space},
	volume = {2},
	noissn = {1941-5982},
	nourl = {https://nodoi.org/10.1080/14786440109462720},
	nodoi = {10.1080/14786440109462720},
	number = {11},
	nourldate = {2022-06-02},
	journal = {The London, Edinburgh, and Dublin Philosophical Magazine and Journal of Science},
	author = {Pearson, Karl},
	nomonth = nov,
	year = {1901},
	pages = {559--572},
}

@article{hotelling_analysis_1933,
	title = {Analysis of a complex of statistical variables into principal components},
	volume = {24},
	noissn = {1939-2176},
	nodoi = {10.1037/h0071325},
	number = {6},
	journal = {Journal of Educational Psychology},
	author = {Hotelling, H.},
	year = {1933},
	pages = {417--441},
}

@inproceedings{nishimori_flag_2007,
	title = {Flag {Manifolds} for {Subspace} {ICA} {Problems}},
	volume = {4},
	nourl = {https://ieeexplore.ieee.org/document/4218376},
	nodoi = {10.1109/ICASSP.2007.367345},
	nourldate = {2024-08-21},
	booktitle = {2007 {IEEE} {International} {Conference} on {Acoustics}, {Speech} and {Signal} {Processing}},
	author = {Nishimori, Yasunori and Akaho, Shotaro and Abdallah, Samer and Plumbley, Mark D.},
	nomonth = apr,
	year = {2007},
	pages = {IV--1417--IV--1420},
}

@inproceedings{nishimori_natural_2008,
	title = {Natural {Conjugate} {Gradient} on {Complex} {Flag} {Manifolds} for {Complex} {Independent} {Subspace} {Analysis}},
	noisbn = {978-3-540-87536-9},
	nodoi = {10.1007/978-3-540-87536-9_18},
	language = {en},
	booktitle = {Artificial {Neural} {Networks} - {ICANN} 2008},
	publisher = {Springer},
	author = {Nishimori, Yasunori and Akaho, Shotaro and Plumbley, Mark D.},
	year = {2008},
	pages = {165--174},
}

@article{fiori_extended_2011,
	title = {Extended {Hamiltonian} {Learning} on {Riemannian} {Manifolds}: {Theoretical} {Aspects}},
	volume = {22},
	noissn = {1045-9227, 1941-0093},
	shorttitle = {Extended {Hamiltonian} {Learning} on {Riemannian} {Manifolds}},
	nourl = {http://ieeexplore.ieee.org/document/5735228/},
	nodoi = {10.1109/TNN.2011.2109395},
	language = {en},
	number = {5},
	nourldate = {2023-01-23},
	journal = {IEEE Transactions on Neural Networks},
	author = {Fiori, S},
	nomonth = may,
	year = {2011},
	pages = {687--700},
}

@article{fiori_extended_2012,
	title = {Extended {Hamiltonian} {Learning} on {Riemannian} {Manifolds}: {Numerical} {Aspects}},
	volume = {23},
	noissn = {2162-2388},
	shorttitle = {Extended {Hamiltonian} {Learning} on {Riemannian} {Manifolds}},
	nodoi = {10.1109/TNNLS.2011.2178561},
	number = {1},
	journal = {IEEE Transactions on Neural Networks and Learning Systems},
	author = {Fiori, Simone},
	nomonth = jan,
	year = {2012},
	pages = {7--21},
}

@misc{szwagier_parsimonious_2025,
	title = {Parsimonious {Gaussian} mixture models with piecewise-constant eigenvalue profiles},
	url = {http://arxiv.org/abs/2507.01542},
	nodoi = {10.48550/arXiv.2507.01542},
	nourldate = {2025-07-04},
	publisher = {arXiv},
	author = {Szwagier, T. and Mattei, P.-A. and Bouveyron, C. and Pennec, X.},
	nomonth = jul,
	year = {2025},
	note = {arXiv:2507.01542},
}

@misc{szwagier_curse_2025,
	title = {The curse of isotropy: from principal components to principal subspaces},
	shorttitle = {The curse of isotropy},
	url = {http://arxiv.org/abs/2307.15348},
	nodoi = {10.48550/arXiv.2307.15348},
	nourldate = {2025-09-10},
	publisher = {arXiv},
	author = {Szwagier, Tom and Pennec, X.},
	nomonth = aug,
	year = {2025},
	note = {arXiv:2307.15348},
}

@article{axen_manifoldsjl_2023,
	title = {Manifolds.jl: {An} {Extensible} {Julia} {Framework} for {Data} {Analysis} on {Manifolds}},
	volume = {49},
	noissn = {0098-3500},
	shorttitle = {Manifolds.jl},
	nourl = {https://dl.acm.org/nodoi/10.1145/3618296},
	nodoi = {10.1145/3618296},
	number = {4},
	nourldate = {2024-08-27},
	journal = {ACM Transactions on Mathematical Software},
	author = {Axen, S. D. and Baran, M. and Bergmann, R. and Rzecki, K.},
	year = {2023},
	pages = {33:1--33:23},
}

@inproceedings{szwagier_eigengap_2025,
	title = {Eigengap {Sparsity} for {Covariance} {Parsimony}},
	noisbn = {978-3-032-03921-7},
	nodoi = {10.1007/978-3-032-03921-7_6},
	language = {en},
	booktitle = {Geometric {Science} of {Information}},
	publisher = {Springer Nature Switzerland},
	author = {Szwagier, Tom and Olikier, Guillaume and Pennec, X.},
	nomonth = oct,
	year = {2025},
	pages = {50--59},
}

@inproceedings{mankovich_flag_2025,
	title = {A {Flag} {Decomposition} for {Hierarchical} {Datasets}},
	nourl = {https://openaccess.thecvf.com/content/CVPR2025/html/Mankovich_A_Flag_Decomposition_for_Hierarchical_Datasets_CVPR_2025_paper.html},
	language = {en},
	nourldate = {2025-07-16},
	booktitle = {Proceedings of the {Computer} {Vision} and {Pattern} {Recognition} {Conference}},
	author = {Mankovich, Nathan and Santamaria, Ignacio and Camps-Valls, Gustau and Birdal, Tolga},
	year = {2025},
	pages = {18738--18748},
}

@inproceedings{ciuclea_shape_2023,
	title = {Shape {Spaces} of {Nonlinear} {Flags}},
	noisbn = {978-3-031-38271-0},
	nodoi = {10.1007/978-3-031-38271-0_5},
	language = {en},
	booktitle = {Geometric {Science} of {Information}},
	author = {Ciuclea, I. and Tumpach, A. B. and Vizman, C.},
	year = {2023},
	pages = {41--50},
}

@article{ma_self-organizing_2022,
	title = {Self-organizing mappings on the flag manifold with applications to hyper-spectral image data analysis},
	volume = {34},
	noissn = {1433-3058},
	nourl = {https://nodoi.org/10.1007/s00521-020-05579-y},
	nodoi = {10.1007/s00521-020-05579-y},
	language = {en},
	number = {1},
	nourldate = {2023-01-05},
	journal = {Neural Computing and Applications},
	author = {Ma, Xiaofeng and Kirby, Michael and Peterson, Chris},
	nomonth = jan,
	year = {2022},
	pages = {39--49},
}

@misc{lim_simple_2024,
	title = {Simple matrix models for the flag, {Grassmann}, and {Stiefel} manifolds},
	url = {http://arxiv.org/abs/2407.13482},
	nodoi = {10.48550/arXiv.2407.13482},
	nourldate = {2025-08-26},
	publisher = {arXiv},
	author = {Lim, Lek-Heng and Ye, Ke},
	nomonth = jul,
	year = {2024},
	note = {arXiv:2407.13482},
}

@article{brockett_dynamical_1991,
	title = {Dynamical systems that sort lists, diagonalize matrices, and solve linear programming problems},
	volume = {146},
	noissn = {0024-3795},
	nourl = {https://www.sciencedirect.com/science/article/pii/002437959190021N},
	nodoi = {10.1016/0024-3795(91)90021-N},
	nourldate = {2025-08-26},
	journal = {Linear Algebra and its Applications},
	author = {Brockett, R. W.},
	nomonth = feb,
	year = {1991},
	pages = {79--91},
}

@article{deift_ordinary_1983,
	title = {Ordinary {Differential} {Equations} and the {Symmetric} {Eigenvalue} {Problem}},
	volume = {20},
	noissn = {0036-1429},
	nourl = {https://www.jstor.org/stable/2157167},
	number = {1},
	nourldate = {2023-08-24},
	journal = {SIAM Journal on Numerical Analysis},
	author = {Deift, P. and Nanda, T. and Tomei, C.},
	year = {1983},
	pages = {1--22},
}

@article{kohonen_emergence_1996,
	title = {Emergence of invariant-feature detectors in the adaptive-subspace self-organizing map},
	volume = {75},
	noissn = {1432-0770},
	nourl = {https://nodoi.org/10.1007/s004220050295},
	nodoi = {10.1007/s004220050295},
	language = {en},
	number = {4},
	nourldate = {2024-07-04},
	journal = {Biological Cybernetics},
	author = {Kohonen, Teuvo},
	nomonth = nov,
	year = {1996},
	pages = {281--291},
}

@inproceedings{watanabe_subspace_1973,
	title = {Subspace method of pattern recognition},
	language = {en},
	booktitle = {Proc. 1st. {IJCPR}},
	author = {Watanabe, Satosi and Pakvasa, Nikhil},
	year = {1973},
	pages = {25--32},
}

@article{james_normal_1954,
	title = {Normal {Multivariate} {Analysis} and the {Orthogonal} {Group}},
	volume = {25},
	noissn = {0003-4851, 2168-8990},
	nourl = {https://projecteuclid.org/journals/annals-of-mathematical-statistics/volume-25/issue-1/Normal-Multivariate-Analysis-and-the-Orthogonal-Group/10.1214/aoms/1177728846.full},
	nodoi = {10.1214/aoms/1177728846},
	number = {1},
	nourldate = {2024-10-30},
	journal = {The Annals of Mathematical Statistics},
	author = {James, A. T.},
	nomonth = mar,
	year = {1954},
	pages = {40--75},
}

@article{bach_optimization_2012,
	title = {Optimization with {Sparsity}-{Inducing} {Penalties}},
	volume = {4},
	noissn = {1935-8237, 1935-8245},
	nourl = {http://dx.nodoi.org/10.1561/2200000015},
	nodoi = {10.1561/2200000015},
	language = {English},
	number = {1},
	nourldate = {2025-02-19},
	journal = {Foundations and Trends® in Machine Learning},
	author = {Bach, F. and Jenatton, R. and Mairal, J. and Obozinski, G.},
	year = {2012},
	pages = {1--106},
}

@inproceedings{santamaria_order_2016,
	title = {An order fitting rule for optimal subspace averaging},
	nourl = {https://ieeexplore.ieee.org/document/7551843},
	nodoi = {10.1109/SSP.2016.7551843},
	nourldate = {2025-10-20},
	booktitle = {2016 {IEEE} {Statistical} {Signal} {Processing} {Workshop} ({SSP})},
	author = {Santamaría, I. and Scharf, L. L. and Peterson, C. and Kirby, M. and Francos, J.},
	nomonth = jun,
	year = {2016},
	pages = {1--4},
}

@article{eltzner_torus_2018,
	title = {Torus principal component analysis with applications to {RNA} structure},
	volume = {12},
	noissn = {1932-6157, 1941-7330},
	nourl = {https://projecteuclid.org/journals/annals-of-applied-statistics/volume-12/issue-2/Torus-principal-component-analysis-with-applications-to-RNA-structure/10.1214/17-AOAS1115.full},
	nodoi = {10.1214/17-AOAS1115},
	number = {2},
	nourldate = {2025-02-14},
	journal = {The Annals of Applied Statistics},
	author = {Eltzner, Benjamin and Huckemann, Stephan and Mardia, Kanti V.},
	nomonth = jun,
	year = {2018},
	pages = {1332--1359},
}

@misc{su_principal_2025,
	title = {Principal {Decomposition} with {Nested} {Submanifolds}},
	url = {http://arxiv.org/abs/2502.10010},
	nodoi = {10.48550/arXiv.2502.10010},
	language = {en},
	nourldate = {2025-02-20},
	publisher = {arXiv},
	author = {Su, Jiaji and Yao, Zhigang},
	nomonth = feb,
	year = {2025},
	note = {arXiv:2502.10010},
}

@article{helmke_isospectral_1991,
	title = {Isospectral flows on symmetric matrices and the {Riccati} equation},
	volume = {16},
	noissn = {0167-6911},
	nourl = {https://www.sciencedirect.com/science/article/pii/016769119190044F},
	nodoi = {10.1016/0167-6911(91)90044-F},
	number = {3},
	nourldate = {2023-08-23},
	journal = {Systems \& Control Letters},
	author = {Helmke, U.},
	nomonth = mar,
	year = {1991},
	pages = {159--165},
}

@misc{rabenoro_geometric_2024,
	title = {A geometric framework for asymptotic inference of principal subspaces in {PCA}},
	url = {http://arxiv.org/abs/2209.02025},
	publisher = {arXiv},
	author = {Rabenoro, Dimbihery and Pennec, Xavier},
	year = {2024},
	note = {arXiv:2209.02025},
}

@inproceedings{jaquier_high-dimensional_2020,
	title = {High-{Dimensional} {Bayesian} {Optimization} via {Nested} {Riemannian} {Manifolds}},
	volume = {33},
	booktitle = {Advances in {Neural} {Information} {Processing} {Systems}},
	author = {Jaquier, Noémie and Rozo, Leonel},
	year = {2020},
	pages = {20939--20951},
}

@article{tae-kyun_kim_-line_2010,
	title = {On-line {Learning} of {Mutually} {Orthogonal} {Subspaces} for {Face} {Recognition} by {Image} {Sets}},
	volume = {19},
	copyright = {https://ieeexplore.ieee.org/Xplorehelp/downloads/license-information/IEEE.html},
	noissn = {1057-7149, 1941-0042},
	nourl = {http://ieeexplore.ieee.org/document/5353735/},
	nodoi = {10.1109/TIP.2009.2038621},
	language = {en},
	number = {4},
	nourldate = {2024-07-18},
	journal = {IEEE Transactions on Image Processing},
	author = {{Tae-Kyun Kim} and Kittler, J. and Cipolla, R.},
	nomonth = apr,
	year = {2010},
	pages = {1067--1074},
}

@misc{giguere_manifold_2017,
	title = {A {Manifold} {Approach} to {Learning} {Mutually} {Orthogonal} {Subspaces}},
	url = {http://arxiv.org/abs/1703.02992},
	nodoi = {10.48550/arXiv.1703.02992},
	nourldate = {2025-09-10},
	publisher = {arXiv},
	author = {Giguere, Stephen and Garcia, Francisco and Mahadevan, Sridhar},
	nomonth = mar,
	year = {2017},
	note = {arXiv:1703.02992},
}

@article{duistermaat_functions_1983,
	title = {Functions, flows and oscillatory integrals on flag manifolds and conjugacy classes in real semisimple {Lie} groups},
	volume = {49},
	noissn = {1570-5846},
	nourl = {http://www.numdam.org/item/CM_1983__49_3_309_0/},
	language = {fr},
	number = {3},
	nourldate = {2023-08-18},
	journal = {Compositio Mathematica},
	author = {Duistermaat, J. J. and Kolk, J. a. C. and Varadarajan, V. S.},
	year = {1983},
	pages = {309--398},
}

@article{watkins_chasing_1991,
	title = {Chasing {Algorithms} for the {Eigenvalue} {Problem}},
	volume = {12},
	noissn = {0895-4798},
	nourl = {https://epubs.siam.org/doi/10.1137/0612027},
	nodoi = {10.1137/0612027},
	number = {2},
	nourldate = {2025-11-19},
	journal = {SIAM Journal on Matrix Analysis and Applications},
	author = {Watkins, D. S. and Elsner, L.},
	nomonth = apr,
	year = {1991},
	pages = {374--384},
}

@article{watkins_isospectral_1984,
	title = {Isospectral {Flows}},
	volume = {26},
	noissn = {0036-1445},
	nourl = {https://epubs.siam.org/doi/10.1137/1026075},
	nodoi = {10.1137/1026075},
	number = {3},
	nourldate = {2023-08-23},
	journal = {SIAM Review},
	author = {Watkins, David S.},
	nomonth = jul,
	year = {1984},
	pages = {379--391},
}

@book{helmke_optimization_1994,
	noaddress = {London},
	series = {Communications and {Control} {Engineering}},
	title = {Optimization and {Dynamical} {Systems}},
	noisbn = {978-1-4471-3469-5},
	nourl = {http://link.springer.com/10.1007/978-1-4471-3467-1},
	nourldate = {2023-01-11},
	publisher = {Springer},
	author = {Helmke, Uwe and Moore, John B.},
	noeditor = {Dickinson, B. W. and Fettweis, A. and Massey, J. L. and Modestino, J. W. and Sontag, E. D. and Thoma, M.},
	year = {1994},
	nodoi = {10.1007/978-1-4471-3467-1},
}

\end{document}